\documentclass{article}

\usepackage[dblblindworkshop, final]{neurips_2026}

\usepackage[utf8]{inputenc}
\usepackage[T1]{fontenc}
\usepackage{hyperref}
\usepackage{url}
\usepackage{amsfonts}
\usepackage{nicefrac}
\usepackage{microtype}
\usepackage[dvipsnames]{xcolor}

\usepackage{algorithm}
\usepackage{algorithmic}

\usepackage{booktabs}
\newsavebox{\digitstab}
\newlength{\digitsgap}

\usepackage{amsmath}
\usepackage{amssymb}
\usepackage{mathtools}
\usepackage{amsthm}
\usepackage{wrapfig}
\usepackage{graphicx}
\usepackage{subcaption}
\usepackage{dsfont}
\usepackage{enumitem}
\usepackage{newtxmath}
\usepackage[capitalize,noabbrev]{cleveref}
\usepackage{esvect}
\usepackage{overarrows}
\usepackage{svg}
\usepackage{tabularx}

\usepackage{booktabs}
\usepackage{multirow}
\usepackage{graphicx}
\usepackage{bm}

\usepackage{paracol}
\usepackage{caption}

\crefname{section}{\S\@gobble}{\S\@gobble}
\crefname{subsection}{\S\@gobble}{\S\@gobble}
\crefname{definition}{Def.}{Defs.}
\crefname{Proposition}{Prop.}{Props.}

\AddToHook{env/assumption/begin}{\crefalias{theorem}{assumption}}
\AddToHook{env/definition/begin}{\crefalias{theorem}{definition}}
\AddToHook{env/proposition/begin}{\crefalias{theorem}{proposition}}
\AddToHook{env/corollary/begin}{\crefalias{theorem}{corollary}}

\DeclareMathOperator{\vol}{vol}

\title{Numerical Integration on Bounded Domains \\ via Learned Reflected Diffusions}
\workshoptitle{AI for Stochastic Dynamics}

\author{%
  Kirill Korolev \\
  HSE University\\
  \texttt{kkorolev@hse.ru} \\
  \And
  Artur Goldman \\
  ETH Zürich, ETH AI Center \\
  \texttt{artur.goldman@ai.ethz.ch} \\
  \AND
  Timofei Gritsaev \\
  Constructor University \\
  \texttt{tgritsaev@gmail.com} \\
  \And
  Tigran Ramazyan \\
  HSE University \\
  \texttt{tramazyan@hse.ru} \\
  \And
  Nikita Morozov \\
  HSE University \\
  \texttt{nvmorozov@hse.ru} \\
}

\theoremstyle{plain}
\newtheorem{theorem}{Theorem}[subsection]
\newtheorem{proposition}[theorem]{Proposition}
\newtheorem{lemma}[theorem]{Lemma}

\theoremstyle{definition}

\newtheorem{assumption}[theorem]{Assumption}
\theoremstyle{remark}
\newtheorem{remark}[theorem]{Remark}

\newcommand{\R}{\mathbb{R}}
\newcommand{\N}{\mathbb{N}}

\newcommand{\ind}{\mathds{1}}

\newcommand{\KL}[2]{\mathsf{KL}(#1||#2)}
\newcommand{\dd}{\mathrm{d}}
\newcommand{\rt}{\bar t}

\newcommand{\cC}{\mathcal{C}}
\newcommand{\cD}{\mathcal{D}}

\newcommand{\cF}{\mathcal{F}}

\newcommand{\cL}{\mathcal{L}}

\newcommand{\cN}{\mathcal{N}}

\newcommand{\E}{\mathbb{E}}
\newcommand{\Q}{\mathbb{Q}}
\newcommand{\EE}[1]{\mathbb{E}\left[#1\right]}

\newcommand{\EEs}[2]{\mathbb{E}_{#1}\left[#2\right]}

\renewcommand{\P}{\mathbb{P}}

\newcommand{\PPs}[2]{\mathbb{P}_{#1}\left(#2\right)}

\newcommand{\Var}{\mathrm{Var}}

\DeclareMathOperator{\Law}{Law}
\DeclareMathOperator{\LT}{LT}

\newcommand{\Pf}{\overrightarrow{\P}}
\newcommand{\Qf}{\overrightarrow{\Q}}
\newcommand{\Pb}{\overleftarrow{\P}}
\newcommand{\Qb}{\overleftarrow{\Q}}

\definecolor{mydarkblue}{rgb}{0,0.08,0.45}
\definecolor{dmorange500}{HTML}{FF5F19}
\definecolor{dmblue300}{HTML}{2267EB}
\definecolor{dmred300}{HTML}{FF617B}
\hypersetup{
	colorlinks,
	citecolor=dmblue300,
	linkcolor=dmred300,
	urlcolor=mydarkblue
}

\begin{document}

\maketitle

\begin{abstract}
We study high-dimensional integration on bounded convex domains. Given a non-negative integrand $r$, the goal is to estimate $Z=\int_{\mathcal D}r(x)\,\mathrm dx$ on $\mathcal D\subset\mathbb R^d$. Standard Quasi-Monte Carlo does not adapt its evaluation points to the integrand, while effective importance sampling requires a proposal that captures its concentration. Learned diffusion samplers can supply such a proposal, but existing methods are formulated on $\mathbb R^d$ and do not respect the domain constraint $x\in\mathcal D$. We instead build the proposal from reflected dynamics that stay in $\mathcal D$ by construction. A reflected diffusion progressively transforms the target distribution toward the uniform distribution on $\mathcal D$, and a neural score model is trained to reverse this process. A change-of-measure formula for reflected diffusions yields a training objective without additional boundary terms and an unbiased importance-sampling estimator of $Z$. For non-smooth integrands, where reparameterization gradients may be undefined or uninformative, we derive a REINFORCE estimator of the training gradient using Gaussian proposal densities before reflection, without evaluating reflected transition densities or differentiating $r$. We evaluate on numerous synthetic targets and real-world problems, demonstrating strong performance of the proposed method, as well as favorable scaling with the dimension of the problem and the number of points used for the Monte Carlo estimate.
\end{abstract}

\section{Introduction}
\label{sec:introduction}

We consider the problem of estimating high-dimensional integrals over bounded constrained
domains. Let $\cD\subset\R^d$ be a bounded convex domain, and let $r:\cD \to \R_{\ge 0}$ be
measurable and satisfy $\int_{\cD}r(x)\,\dd x<\infty$. Our goal is to estimate the integral
\begin{equation}
    \label{eq:problem}
    Z = \int_{\cD} r(x)\,\dd x .
\end{equation}
Integrals of this form are a computational bottleneck across the sciences. In statistical
physics and chemistry, free energies and thermodynamic observables are logarithms of such
constants, and learned samplers have been used to estimate them for many-body systems and
lattice field theories \citep{noe2019boltzmann,wirnsberger2020targeted,nicoli2021estimation}.
In computer graphics, every pixel of a rendered image is an integral over light paths,
parameterized so that the integration domain is the unit cube \citep{muller2019neural}. In
computational finance, the price of a path-dependent derivative is an expectation over a
discretized asset path, which the standard inverse-CDF construction turns into an integral
over $[0,1]^d$ with $d$ the number of monitoring dates \citep{glasserman2004monte,dick2013high}.
In Bayesian statistics, the model evidence is an integral of an unnormalized posterior
\citep{kloek1978bayesian, kass1995bayes}, confined to a bounded region whenever the parameters are subject to simplex constraints or under compactly supported priors.

These tasks give rise to integration problems on bounded domains, where the location of integrand mass is unknown in advance. This motivates samplers that respect the domain constraints and adapt to where the mass lies, including when it concentrates near faces or corners, or when the integrand vanishes on parts of the domain.

Existing estimators differ in how much they adapt to $r(x)$. Standard Quasi-Monte Carlo~\citep[QMC]{niederreiter1992random,dick2013high} spreads a deterministic point set over the cube without reference to the integrand. Bayesian quadrature~\citep{o1991bayes,rasmussen2003bayesian} adapts by placing a Gaussian process prior on $r(x)$, but at a cost cubic in the number of evaluations \citep{briol2019probabilistic}. Importance sampling relies on a proposal distribution whose choice determines the estimator’s variance~\citep{hammersley1964monte}. Neural proposals learn that proposal from evaluations of $r(x)$ alone, either with normalizing flows~\citep{rezende2015variational,muller2019neural} or with diffusion samplers~\citep{zhang2022path,vargas2023denoising,berner2024an}. The latter avoids the invertibility constraints that limit flow architectures \citep{kong2020expressive}.  Diffusion samplers, however, are formulated on $\R^d$ with a Gaussian reference process throughout the existing work and do not respect the constraint $x\in\cD$. A more thorough discussion of prior work on numerical integration methods is presented in Appendix~\ref{related_work}.

To design a learnable diffusion-based proposal that respects the domain constraint, we draw inspiration from diffusion models~\citep{song2021scorebased} for constrained generative modeling. Existing approaches fall into two groups. The first changes the geometry in which the process evolves: mirror diffusion models map $\mathcal{D}$ onto $\mathbb{R}^d$ through the gradient of a strictly convex potential and run a standard diffusion in the resulting dual space \citep{liu2023mirror}, while log-barrier diffusions equip the interior of $\mathcal{D}$ with the Riemannian metric induced by the Hessian of a barrier function, under which the process remains in the interior \citep{fishman2023diffusion}. The second group keeps the Euclidean geometry of $\mathcal{D}$ and confines trajectories by reflecting them at the boundary, using reflected Brownian motion as the noising process \citep{fishman2023diffusion, lou2023reflected}. In this work, we follow the second route and adapt reflected diffusions from generative modeling to sampling, where no samples from $\pi$ are available and the
reverse dynamics must be learned from evaluations of $r$ alone.

\textbf{We outline our contributions as follows:}
\begin{itemize}[leftmargin=*,itemsep=1pt,topsep=2pt]
\item We propose \emph{reflected diffusion samplers} (RDS) for integration on bounded domains, extending control-based diffusion samplers~\citep{zhang2022path,vargas2023denoising,berner2025discrete} to constrained $\cD$. Our construction uses reflected Brownian motion with a time-dependent noise schedule. When initialized from the uniform distribution $\nu$ on $\cD$, this process preserves uniform marginals~\citep{loper2020uniform} and defines a tractable reference path measure for the integral \eqref{eq:problem}. Girsanov's theorem~\citep{le2016brownian} then transfers this representation to equivalent controlled reflected diffusions through path-space importance weights.

We show that the unique zero-variance ideal proposal corresponds to the same reflected noising dynamics initialized from the normalized target $\pi(x)=r(x)/Z$, and propose to learn its reversal with a neural network-parameterized drift. The learned reversal by construction also acts as the sampler from $\pi$.
\item We design a training objective for the learned reversal to be the KL divergence to the ideal proposal. For smooth targets, its gradients can be computed via the reparameterization trick. It includes a running control cost with a terminal potential and no boundary term, closely resembling the optimal control objective of~\cite{zhang2022path}. We also derive a REINFORCE gradient estimator with a leave-one-out baseline, which uses only the pre-reflection Gaussian transition densities and does not require gradients of $r(x)$, applying to targets that are non-smooth or vanish to zero on part of $\cD$, where reparameterization gradients are uninformative. Integrands that vanish on part of the domain are handled by adding a positive offset and correcting the resulting integral estimate.
\item We provide rigorous theoretical derivation of all components of the proposed methodology, adapting existing theoretical results to the considered setting of the variance exploding reflected diffusion samplers. In addition, we prove a total-variation accuracy bound for the designed estimator, which is a reflected version of the bound presented in~\cite{guo2026complexity} for diffusion samplers in $\mathbb{R}^d$.
\item We evaluate on Genz integrands~\citep{genz1984testing}, bounded synthetic targets, Feynman parameter integral~\citep{Bonciani_2016} and Asian option pricing~\citep{l2009quasi} against Quasi-Monte Carlo, Bayesian quadrature and normalizing flow importance sampling baselines. The experiments demonstrate strong performance of the proposed method, as well as favorable scaling with the dimension of the problem and the number of points used for the estimate. In addition, we compare against Path Integral Sampler (PIS,~\cite{zhang2022path}) proposal by mapping $\cD$ onto $\mathbb{R}^d$, and demonstrate that the learned RDS proposal achieves lower integration errors against its unconstrained counterpart.
\end{itemize}

Source code: \textcolor{RoyalBlue}{\href{https://github.com/kkorolev1/reflected-diffusion-samplers}{github.com/kkorolev1/reflected-diffusion-samplers}}.

\section{Background on Reflected SDE}
\label{sec:background}

Denote $\overline{\cD}$ and $\partial\cD$ to be the closure and boundary of $\cD$, and $\vol(\cD)$ to be its volume. Consider a reflected SDE on $\cD$ with drift $f:\overline{\cD}\times[0,1]\to\R^d$ and scalar diffusion coefficient $g:[0,1]\to\R$ \citep{pilipenko2014introduction}:
\begin{equation}
    \label{eq:reflected_sde}
    \dd X_t
    =
    f(X_t,t)\,\dd t
    +
    g(t)\,\dd W_t
    +
    \dd L_t.
\end{equation}
Here $W_t$ is a standard $d$-dimensional Brownian motion. Under the
smooth-boundary assumptions used for the theoretical results, a solution of SDE \eqref{eq:reflected_sde} has the form $(X,\ell)$, where $\ell$ is continuous, nondecreasing, starts from zero, and increases only when $X_t\in\partial\cD$. Throughout the manuscript, the associated vector reflection term is the derived process $L_t:=\int_0^t\mathbf n_{\mathrm{in}}(X_s)\,\dd\ell_s$, where $\mathbf n_{\mathrm{in}}$ is the inward unit normal. Thus $\dd L_t$ pushes $X$ inward at boundary times.

Let $\mu_t:=\Law(X_t)$ be the marginal law of $X_t$. Suppose that, for
$t>0$, $\mu_t$ admits a density $p_t$ and that these densities have the
time, space, and boundary-trace regularity needed for the following
identities to hold. Then $p_t$ evolves according to the Fokker--Planck equation with a Neumann boundary condition \citep[Section~2.6.2--2.6.3]{schuss2015brownian}: \vspace{-0.05cm}
\begin{equation}
    \label{eq:fp_with_boundary}
    \begin{cases}
        \displaystyle
        \partial_t p_t(x)
        =
        \operatorname{div}\!\left(
            -p_t(x) f(x,t)
            +
            \frac{g^2(t)}{2}\nabla_x p_t(x)
        \right),
        & x\in\cD,\ t\in(0,1],\\
        \displaystyle
        \mathbf n(x)^\top
        \left(
            p_t(x) f(x,t)
            -
            \frac{g^2(t)}{2}\nabla_x p_t(x)
        \right)
        =0,
        & x\in\partial\cD,\ t\in(0,1],
    \end{cases}
\end{equation}
where $\mathbf n(x)$ is the outward unit normal to $\partial\cD$. The boundary condition prevents probability mass from crossing the boundary, so the density evolves entirely within $\cD$. The initial law is imposed separately through $\mu_t\Rightarrow\mu_0$ as $t\downarrow0$, which follows from path
continuity. Even when $\mu_0$ has a density $p_0$, this does not make~\eqref{eq:fp_with_boundary} valid at $t=0$ without additional initial regularity and boundary compatibility. Without the stated density and regularity, the forward equation is instead interpreted in its weak measure-valued form.

\section{Reflected Diffusion Samplers}
\label{sec:method}

In this section, we present the main theoretical results and the derivation of the proposed method. Each statement points to a fully quantified statement and proof in the Appendix. The two distributions used in this section are the uniform distribution and the normalized target:
\begin{equation}
    \nu(x):=\vol(\cD)^{-1},
    \qquad
    \pi(x):=\frac{r(x)}{Z},
    \qquad x\in\overline{\cD}.
\end{equation}
Path-measure and training theoretical results assume measurable $r$ bounded above and
away from zero. The full-interval score representation requires additional regularity. We later discuss zero- and negative-valued targets. See Appendix~\ref{sec:notation} for notation and conventions.

\subsection{Reference process of the reflected diffusion}
Our goal is to construct an importance-sampling estimator for \eqref{eq:problem} using stochastic dynamics that remain inside the constrained domain $\cD$. To this end, we first introduce a simple reference process whose path measure will serve as a base measure with respect to which a proposal is defined. Ideally, this reference process should satisfy two properties. First, it should remain inside $\overline{\cD}$ by construction. Second, its marginal distribution should be known and easy to sample from. We obtain both properties with the reflected
variance-exploding\footnote{Here, “variance-exploding” refers to the noise schedule inherited from the unconstrained SDE, not to the marginal variance of the reflected process, which remains bounded because $\cD$ is bounded. We follow the terminology of \citet{lou2023reflected}.} SDEs~\citep[RVE SDEs]{lou2023reflected}. Whenever the SDE below is well posed, each initial law determines a probability measure on the trajectory space $\cC =C([0,1];\overline{\cD})$, which we call a path measure. For noise
scales $0<\sigma_0\ll\sigma_1$, define the variance-exploding schedule
\begin{equation}
    \label{eq:rve_noise_schedule}
    \sigma(t):=\sigma_0^{1-t}\sigma_1^t,
    \qquad
    g(t):=\left(\frac{\dd}{\dd t}\sigma(t)^2\right)^{1/2}
    =\sigma(t)\left(2\log\frac{\sigma_1}{\sigma_0}\right)^{1/2}.
\end{equation}
For any initial distribution $\mu_0$, the corresponding RVE SDE is
\begin{equation}
    \label{eq:rve_sde}
    \dd X_t=g(t)\,\dd W_t+\dd L_t,
    \qquad X_0\sim\mu_0,\quad 0\leq t\leq1.
\end{equation}
As the reference process, we use the RVE SDE \eqref{eq:rve_sde} initialized from the uniform distribution $\mu_0 = \nu$, and denote the induced path measure by $\Pf$. The following proposition justifies the use of the RVE SDE \eqref{eq:rve_sde} as a noising process by showing that its marginals remain explicitly known when initialized from the uniform distribution, while, for an arbitrary initial distribution, the terminal law is approximately uniform.
\begin{proposition}
\label{proposition:rve_sde_marginal_informal}
For the reflected noising process in \eqref{eq:rve_sde}, under the assumptions of Proposition~\ref{proposition:rve_sde_marginal}, the uniform distribution is invariant:
\[
    X_0\sim\nu
    \quad\Longrightarrow\quad
    X_t\sim\nu,
    \qquad 0\leq t\leq1.
\] 
Moreover, fix any initial probability distribution $\mu_0$ on $\overline{\cD}$. For each $n\in\N$, let $X^{(n)}$ solve \eqref{eq:rve_sde} with initial law $\mu_0$ and noise schedule $g_n(\cdot)$ corresponding to scales $0<\sigma_{0,n}<\sigma_{1,n}$, and denote its terminal law by $\mu_1^{(n)}:=\Law(X_1^{(n)})$. Suppose that $\int_0^1 g_n^2(t)\,\dd t
    =\sigma_{1,n}^2-\sigma_{0,n}^2
    \underset{n\to\infty}{\longrightarrow} \infty$. Then
$$
    \|\mu_1^{(n)}-\nu\|_{\mathrm{TV}}
    \underset{n\to\infty}{\longrightarrow} 0.
$$
\end{proposition}

Full formulation of Proposition~\ref{proposition:rve_sde_marginal_informal} with proofs is available in Proposition~\ref{proposition:rve_sde_marginal}. In practice one should choose $\sigma_0 \ll \sigma_1$ to control the discrepancy between the terminal distribution $\mu_1$ of $X_1$ and the uniform distribution $\nu$. 

Consequently, since the initial marginal of $\Pf$ is uniform,
\eqref{eq:problem} can be written as an expectation under the reference path
measure $Z=\E_{\Pf}\!\left[{r(X_0)}/{\nu(X_0)}\right]$. When used directly, it corresponds to uniform Monte Carlo. This estimator can have large variance when the mass of $r(x)$ is concentrated in a small region of $\cD$, since most samples from the uniform distribution then make only a negligible contribution to the estimate. This motivates replacing the reference path measure with a more suitable proposal. Generally, any path measure $\Qf$, which is equivalent to $\Pf$ ($\Pf\ll\Qf$ and $\Qf\ll\Pf$), gives the importance weight
\begin{equation}
    W_{\Qf}(X)
    :=
    \frac{r(X_0)}{\nu(X_0)}
    \frac{\dd\Pf}{\dd\Qf}(X),
    \qquad
    Z=\E_{\Qf}[W_{\Qf}(X)].
    \label{eq:path_importance_weight}
\end{equation}
Therefore, every admissible proposal $\Qf$ yields an unbiased importance-sampling
estimator, and the remaining question is how to choose $\Qf$ to reduce
its variance, which is justified in the following proposition. Let $\Qf^\star$ denote a path measure of the RVE SDE \eqref{eq:rve_sde} initialized from a target $\pi$:
\begin{equation}
    \label{eq:optimal_forward_path_law}
    \Qf^\star:\quad
    \dd X_t=g(t)\,\dd W_t+\dd L_t,
    \qquad X_0\sim\pi,\quad 0\leq t\leq1.
\end{equation}

\begin{proposition}
\label{proposition:optimal_proposal_informal}
Under the assumptions of Proposition~\ref{proposition:optimal_proposal}, $\Qf^\star$ is the unique zero-variance proposal among path measures
equivalent to $\Pf$, i.e., $ W_{\Qf^\star}(X)=Z$ with probability one for $X\sim\Qf^\star$.
\end{proposition}
The full formal statement of Proposition~\ref{proposition:optimal_proposal_informal} with proofs is available in Proposition~\ref{proposition:optimal_proposal}. The conditions and proof of existence and pathwise uniqueness of $\Pf$ and $\Qf^\star$ can be found in Lemma~\ref{lemma:forward_reflected_well_posedness}. Sampling $X_0\sim\pi$ is intractable in general, so we instead consider the time
reversal of the process~\eqref{eq:optimal_forward_path_law}, which we discuss in the next section.

\subsection{Time-reversal of the optimal proposal}

Let $p_t$ denote the marginal density of the forward process
\eqref{eq:optimal_forward_path_law}, initialized from the target $\pi$. 
We apply the time-reversal map $Y_t:=X_{1-t}$ to both the optimal forward process and the reference process, defining their reversed path laws by
\begin{equation}
    X\sim\Qf^\star \Longrightarrow Y\sim\Qb^\star,
    \qquad
    X\sim\Pf \Longrightarrow Y\sim\Pb.
    \label{eq:reverse_path_law_definitions}
\end{equation}

Thus, under $\Qb^\star$, trajectories start from $p_1$ and end at $\pi$, and $\frac{\dd\Qb^\star}{\dd\Pb}(Y) =  \frac{\pi(Y_1)}{\nu(Y_1)}$. The reversed reference law $\Pb$ is characterized by the reflected SDE
\begin{equation}
    \label{eq:reverse_reference_sde}
    \dd Y_t
    =
    g(1-t)\,\dd\overline W_t
    +\dd\overline L_t,
    \qquad
    Y_0\sim\nu,\quad 0\leq t\leq1,
\end{equation}
where $\overline W_t$ is a Brownian motion with respect to the reverse-time filtration, and $\overline L_{\rt}:=\int_0^{\rt}\mathbf n_{\mathrm{in}}(Y_s)\,
\dd\overline\ell_s$. 
Lemma~\ref{lemma:rve_stationary_reversal} shows that \eqref{eq:reverse_reference_sde}
has a unique solution law, with strong existence and pathwise uniqueness. The optimal reverse law $\Qb^\star$ admits the SDE representation
\begin{equation}
    \label{eq:full_interval_ideal_sde}
    \dd Y_t
    =
    g^2(1-t)\nabla\log p_{1-t}(Y_t)\,\dd t
    +g(1-t)\,\dd\overline W_t
    +\dd\overline L_t,
    \qquad
    Y_0\sim p_1,\quad 0\leq t\leq1.
\end{equation}
Lemma~\ref{lemma:scheduled_reflected_time_reversal} establishes the weak representation \eqref{eq:full_interval_ideal_sde} and, under additional conditions, its uniqueness. In practice, $p_1$ is replaced by $\nu$, while the unknown score
$\nabla\log p_t$ is approximated by a neural network
$s_\theta:\overline{\cD}\times[0,1]\to\R^d$. This defines the learned reverse law
$\Qb_\theta$ through
\begin{equation}
    \label{eq:learnable_reverse_rve_sde}
    \Qb_\theta:\quad
    \dd Y_t
    =
    g^2(1-t)s_\theta(Y_t,1-t)\,\dd t
    +g(1-t)\,\dd\overline W_t
    +\dd\overline L_t,
    \qquad
    Y_0\sim\nu,\quad 0\leq t\leq1.
\end{equation}
Theorem~\ref{theorem:kl_reference} gives existence and uniqueness of
\eqref{eq:learnable_reverse_rve_sde} under explicit continuity and
spatial-Lipschitz assumptions on $s_\theta$.

\subsection{Training objective and estimator accuracy}
\label{sec:training_objective}

We train $\Qb_\theta$ by minimizing its reverse KL divergence to the ideal reversed proposal $\Qb^\star$.
\begin{theorem}
\label{theorem:kl_reference_informal}
Under the assumptions of Theorem~\ref{theorem:kl_reference}, for a realization of \eqref{eq:learnable_reverse_rve_sde} driven by Brownian motion $\overline W$,
\begin{equation}
    \label{eq:learned_reference_density_informal}
    \log\frac{\dd\Qb_\theta}{\dd\Pb}(Y)
    =
    \int_0^1
        g(1-t)s_\theta(Y_t,1-t)^\top
        \dd\overline W_t
    +
    \frac{1}{2}
    \int_0^1
        g^2(1-t)
        \lVert s_\theta(Y_t,1-t)\rVert_2^2
        \,\dd t,
    \quad \Qb_\theta\text{-a.e.}
\end{equation}
Consequently, the path KL admits a representation
\begin{equation}
    \KL{\Qb_\theta}{\Qb^\star}
    =\cL(\theta)+\log Z,
\end{equation} \vspace{-0.05cm}
where \vspace{-0.05cm}
\begin{equation}
    \label{eq:training_loss}
    \cL(\theta)
    :=
    \EEs{\Qb_\theta}{
        \frac{1}{2}
        \int_0^1
            g^2(1-t)
            \lVert s_\theta(Y_t,1-t)\rVert_2^2
            \,\dd t
        +\log\frac{\nu(Y_1)}{r(Y_1)}
    }.
\end{equation}
\end{theorem}
The full formal statement and proof are available in Theorem~\ref{theorem:kl_reference}. This result shows that minimising $\cL$ is equivalent to minimising $\KL{\Qb_\theta}{\Qb^\star}$ up to the constant $\log Z$, which is not needed for training. 
\begin{remark}
\label{rem:irreducible_initial_law}
Since $\Qb_\theta$ starts from $\nu$, whereas $\Qb^\star$ starts from $p_1$, the
chain rule for KL gives
\[
    \KL{\Qb_\theta}{\Qb^\star}
    =\KL{\nu}{p_1}
    +\int_{\cD}
      \KL{\Qb_\theta(\,\cdot\mid Y_0=y)}
      {\Qb^\star(\,\cdot\mid Y_0=y)}
      \nu(y)\,\dd y.
\]
Consequently, even perfect matching of the conditional dynamics leaves the
irreducible term $\KL{\nu}{p_1}$, which is determined by the noising schedule.
\end{remark} 
The objective in \eqref{eq:training_loss} has the same structure as the standard control-based losses used in diffusion samplers \citep{zhang2022path, vargas2023denoising}. It combines a quadratic running cost with a terminal potential that favors states of high
target density. Despite the reflected dynamics, no additional boundary term appears in the loss. After training, trajectories from $\Qb_\theta$ give the importance-sampling estimator
\begin{equation}
    \label{eq:importance_sampling_estimator}
    \widehat Z
    =
    \frac{1}{N}
    \sum_{i=1}^N
    \frac{r(Y_1^{(i)})}{\nu(Y_1^{(i)})}
    \frac{\dd\Pb}{\dd\Qb_\theta}\bigl(Y^{(i)}\bigr),
    \qquad Y^{(i)}\sim\Qb_\theta,
\end{equation}
where $\dd\Pb / \dd\Qb_\theta$ is calculated using \eqref{eq:learned_reference_density_informal}.
If the learned law equals the ideal reverse law and the noised initial
distribution is exactly uniform, the weight is the constant $Z$ and
$Y_1\sim\pi$, so the learned
process can also be viewed as a sampler for the normalized target.

Closeness of the learned and ideal path laws controls the relative error of the normalizing-constant estimate. We quantify this relation for the
continuous estimator~\eqref{eq:importance_sampling_estimator} on $[0,1]$. The full statement and proof appear in
Theorem~\ref{theorem:continuous_normalizing_constant_accuracy}. 

\begin{theorem}
    \label{theorem:continuous_normalizing_constant_accuracy_informal}
Under the assumptions of
Theorem~\ref{theorem:continuous_normalizing_constant_accuracy}, fix $\theta$
and let $\widehat Z$ be the estimator \eqref{eq:importance_sampling_estimator},
computed from $N\geq1$ paths, each with law $\Qb_\theta$. Then
$\mathbb E[\widehat Z]=Z$ and, for every $\varepsilon>0$,
\begin{equation}
    \Pr\!\left(\left|\frac{\widehat Z}{Z}-1\right|\geq\varepsilon\right)
    \leq\frac{2\|\Qb^\star-\Qb_\theta\|_{\mathrm{TV}}}{\varepsilon}.
    \label{eq:continuous_normalizing_constant_tail_informal}
\end{equation}
\end{theorem}

Pinsker's
inequality connects this total-variation bound to either KL direction:
\begin{equation}
    2\|\Qb^\star-\Qb_\theta\|_{\mathrm{TV}}
    \leq\sqrt{2\min\{\KL{\Qb^\star}{\Qb_\theta},\KL{\Qb_\theta}{\Qb^\star}\}}.
    \label{eq:continuous_pinsker_comparison_informal}
\end{equation}
The training optimizes $\KL{\Qb_\theta}{\Qb^\star}$ via
Theorem~\ref{theorem:kl_reference_informal}. Under additional score assumptions,
Theorem~\ref{theorem:continuous_ideal_to_learned_kl} expresses
the other term
$\KL{\Qb^\star}{\Qb_\theta}$ as the initial-law mismatch $\KL{p_1}{\nu}$
plus an integrated score error, 
which brings additional mathematical insight about the nature of this bound. 

One can also note that the bound~\eqref{eq:continuous_normalizing_constant_tail} does not depend on $N$. Remark~\ref{remark:dependence_on_n} shows that an explicit dependence on $N$ can be given through $\chi^2$ divergence.

\subsection{Practical implementation}
\label{sec:practical}

In practice, we simulate \eqref{eq:learnable_reverse_rve_sde} using a
discrete variance-exploding (VE) noise schedule together with the reflection
operator used in reflected diffusion models \citep{lou2023reflected}. Let
$\Delta=\{0=t_0<\dots<t_K=1\}$ be a time discretization, set $|\Delta|:=\max_{0\leq k<K}(t_{k+1}-t_k)$ and
\begin{equation}
    \overline\sigma_k:=\sigma(1-t_k),
    \qquad
    g_k^2:=\overline\sigma_k^2-\overline\sigma_{k+1}^2,
    \qquad k=0,\ldots,K-1,
\end{equation}
where $\sigma(\cdot)$ is defined in \eqref{eq:rve_noise_schedule}. Starting from $Y_0\sim\nu$, at each step $k=0,\ldots,K-1$ we first sample the unconstrained update and then apply a reflection operator
\begin{equation}
\begin{aligned}
\label{eq:euler_scheme}
    \widetilde{Y}_{k+1}
    &=
    Y_k
    +
    g_k^2 s_\theta(Y_k,1-t_k)
    +
    g_k\xi_k,
    \qquad
    \xi_k\stackrel{\mathrm{i.i.d.}}{\sim}\cN(0,I_d),\\
    Y_{k+1} &= \operatorname{refl}_{\cD}(\widetilde{Y}_{k+1}),
\end{aligned}
\end{equation}
where $\operatorname{refl}_{\cD}:\R^d\to\overline{\cD}$ denotes the domain-specific reflection map. Under standard regularity assumptions, the scheme converges weakly to
the reflected SDE as $|\Delta|\to0$ \citep{lou2023reflected}.

\textbf{Reparameterized gradients.} For the given discretization, define the running and terminal costs
\begin{equation}
\label{eq:costs}
    \phi_k(y;\theta)
    :=
    \frac12 g_k^2
    \lVert s_\theta(y,1-t_k)\rVert_2^2,
    \qquad
    \psi(y):=\log\frac{\nu(y)}{r(y)}.
\end{equation}
Then the finite-step version of the objective \eqref{eq:training_loss} admits the following representation:
\begin{equation}
    \label{eq:discrete_objective_reinforce}
    \cL_{\text{RP}}(\theta)
    =
    \E_{\tau\sim q(\cdot;\theta)}
    \left[
        \sum_{k=0}^{K-1}\phi_k(Y_k;\theta)
        +\psi(Y_K)
    \right],
\end{equation}
where $q(\cdot;\theta)$ is the distribution induced by the model over trajectories in discretized time. We compute gradients using the reparameterization trick, expressing each stochastic transition as a deterministic function of the current state and an independent Gaussian noise variable, and backpropagating through the simulated trajectory. Algorithm~\ref{algorithm:rds_training} summarizes the resulting training procedure.

\begin{algorithm}[t]
\caption{Training reflected diffusion samplers}
\label{algorithm:rds_training}
\begin{algorithmic}[1]
\STATE {\bfseries Input:} Initial parameters $\theta$, target $r$, domain $\cD$,
number of steps $K$, time discretization $\{t_k\}_{k=0}^{K-1}$,
variance increments $\{g_k^2\}_{k=0}^{K-1}$,
batch size $B$, number of iterations $M$
\STATE Define $\phi_k(y;\theta),\psi(y)$ from \eqref{eq:costs} and
$\widetilde q_k(\widetilde y\mid y;\theta)$ from
\eqref{eq:latent_transition_density}

\FOR{$m=0,\ldots,M-1$}
    \STATE Sample
    $Y_0^{(i)}\stackrel{\mathrm{i.i.d.}}{\sim}\nu,
    \quad i=1,\ldots,B$
    \FOR{$k=0,\ldots,K-1$}
        \STATE Sample
        $\xi_k^{(i)}\stackrel{\mathrm{i.i.d.}}{\sim}\cN(0,I_d),
        \quad i=1,\ldots,B$
        \STATE $\widetilde Y_{k+1}^{(i)}
        \gets Y_k^{(i)}
        +g_k^2s_\theta(Y_k^{(i)},1-t_k)
        +g_k\xi_k^{(i)}, \quad Y_{k+1}^{(i)}
        \gets\operatorname{refl}_{\cD}
        (\widetilde Y_{k+1}^{(i)})$
    \ENDFOR

    \IF{$\text{algorithm is RDS-RP}$}
    \STATE $\displaystyle
    \cL_{\mathrm{RP}}(\theta)
    \gets
    \frac1B\sum_{i=1}^B
    \left[
        \sum_{k=0}^{K-1}\phi_k(Y_k^{(i)};\theta)
        +\psi(Y_K^{(i)})
    \right]$

    \STATE $\displaystyle
    \theta\gets
    \operatorname{OptStep}
    \left(
        \theta,
        \nabla_\theta\cL_{\mathrm{RP}}(\theta)
    \right)$
    \ELSIF{$\text{algorithm is RDS-RF}$}
    \STATE Detach augmented trajectories
    $(Y_0,\widetilde Y_1,\ldots,\widetilde Y_K,Y_K)$

    \STATE Compute $c_k^{(i)}$ from
    \eqref{eq:reinforce_cost_to_go} and set  $b_k^{(i)}\gets\frac{1}{B-1} \sum_{\ell\neq i} c_k^{(\ell)}$

    \STATE $\displaystyle
    \cL_{\mathrm{RF}}(\theta)
    \gets
    \frac1B\sum_{i=1}^B\sum_{k=0}^{K-1}
    \left[
        \phi_k(Y_k^{(i)};\theta)
        +
        (c_k^{(i)}-b_k^{(i)})\,
        \log\widetilde q_k
        (\widetilde Y_{k+1}^{(i)}
        \mid Y_k^{(i)};\theta)
    \right]$

    \STATE $\displaystyle
    \theta\gets
    \operatorname{OptStep}
    \left(
        \theta,
        \nabla_\theta\cL_{\mathrm{RF}}(\theta)
    \right)$

    \ENDIF
\ENDFOR
\end{algorithmic}
\end{algorithm}

\textbf{REINFORCE gradients.} The reparameterization gradient used above relies on differentiating the
terminal cost through the sampled state. For non-smooth targets this derivative may be undefined, while on regions where the target is constant it vanishes and cannot capture changes in probability mass across the region boundaries. We therefore use a score-function gradient estimator, inspired by the REINFORCE technique \citep{williams1992simple}, that
differentiates the trajectory distribution instead.  Before reflection is applied, the unconstrained transitions have Gaussian density
\begin{equation}
    \label{eq:latent_transition_density}
    \widetilde q_k(
        \widetilde Y_{k+1}\mid Y_k;\theta
    )
    =
    \cN\!\left(
        \widetilde Y_{k+1};
        Y_k+g_k^2s_\theta(Y_k,1-t_k),
        g_k^2I_d
    \right).
\end{equation}
To calculate true densities of a reflected transition, one has to compute an infinite sum of Gaussian densities corresponding to pre-images of the reflection map, see~\cite{lou2023reflected}. However, we show that the gradient of~\eqref{eq:discrete_objective_reinforce} can be expressed without evaluating computationally expensive reflected transition densities. We augment the trajectory with the pre-reflection states, $\widetilde\tau:=(Y_0,\widetilde Y_1,Y_1,\ldots,\widetilde Y_K,Y_K)$,
and denote its law by $\widetilde q(\cdot;\theta)$. Then, $\nabla_\theta\cL_{\text{RP}}(\theta)$ admits a simplified form.

\begin{theorem}
\label{theorem:reinforce_gradient_informal}
Under the assumptions of
Theorem~\ref{theorem:reinforce_gradient}, the gradient of
\eqref{eq:discrete_objective_reinforce} satisfies
\begin{equation}
    \nabla_\theta\cL_{\text{RP}}(\theta)
    =
    \EEs{\widetilde{\tau} \sim \widetilde{q}(\cdot;\theta)}{
        \sum_{k=0}^{K-1}
        \nabla_\theta
        \phi_k(Y_k;\theta)
        +
        \sum_{k=0}^{K-1}
        c_k\,
        \nabla_\theta
        \log
        \widetilde q_k(
            \widetilde Y_{k+1}\mid Y_k;\theta
        )
    },
\end{equation} \vspace{-0.1cm}
where  \vspace{-0.1cm}
\begin{equation}
    \label{eq:reinforce_cost_to_go}
    c_k :=
    \sum_{j=k+1}^{K-1}\phi_j(Y_j;\theta)
    +\psi(Y_K).
\end{equation}
\end{theorem}
Here $\nabla_\theta$ denotes differentiation with the state held fixed. Full formal statement of Theorem~\ref{theorem:reinforce_gradient_informal} with proofs is available in Theorem~\ref{theorem:reinforce_gradient}. This result allows the gradient to be
computed using only the tractable pre-reflection Gaussian transitions.
In particular, no reflected transition density or derivative of the reflection
map is required. To reduce the variance of the score-function term, for a batch of $B$ independent trajectories we use the leave-one-out~\citep{kool2019buy} baseline $b_k^{(i)}$ and replace $c_k^{(i)}$ in by $c_k^{(i)}-b_k^{(i)}$. We summarize the training procedure using the REINFORCE gradient estimator in Algorithm~\ref{algorithm:rds_training}.

\textbf{Evaluation.} For a batch of $N$ trajectories generated by the reflected scheme \eqref{eq:euler_scheme}, the final importance-sampling estimator~\eqref{eq:importance_sampling_estimator} under a time discretization corresponds to
\begin{equation}
\label{eq:discrete_Z_estimator}
\widehat Z
=
\frac{1}{N}
\sum_{i=1}^{N}
\frac{r(Y_K^{(i)})}{\nu(Y_K^{(i)})}
\exp\left(
-\sum_{k=0}^{K-1}
\left[
g_k
s_\theta(Y_k^{(i)},1-t_k)^\top
\xi_k^{(i)}
+
\frac12 g_k^2
\|
s_\theta(Y_k^{(i)},1-t_k)
\|_2^2
\right]
\right).
\end{equation}
Under the assumption that a discretized step of~\eqref{eq:reverse_reference_sde} preserves the uniform distribution on $\cD$, which holds for $\cD$ considered in our experiments, the discrete-time estimate~\eqref{eq:discrete_Z_estimator} remains unbiased. The formal proof and discussion are available in Proposition~\ref{prop:discrete_estimator_unbiased} and Remark~\ref{remark:discrete_reference_invariance}.
\section{Experiments}
\label{sec:experiments}
\textbf{Benchmarks}. We evaluate our method on bounded-domain integration problems defined on the unit hypercube
$\overline{\cD}=[0,1]^d$, using the same reflection operator as in~\cite{lou2023reflected}. We consider Genz Continuous and Gaussian Peak to test smooth and non-smooth interior concentration, Genz Product Peak for sharply localized factorized structure, Genz Corner Peak for mass concentrated at a boundary corner, and Genz Discontinuous for discontinuous support with a zero-valued region \citep{genz1984testing, kaarnioja2025explicit}. GMM Cube tests multimodal exploration, Stripe Cube tests concentration near boundary faces, and Funnel Cube tests adaptation to strongly varying conditional scales. Finally, we consider two real world integrals - Elliptic Two-Loop and Asian option. The Elliptic Two-Loop integral is a finite five-dimensional Feynman parameter integral over the unit cube~\citep{Bonciani_2016, Borowka_2018}. And the Asian option benchmark \citep{l2009quasi} provides a high-dimensional problem with strong coordinate dependence, a non-smooth payoff, and a large zero-valued region. We additionally compare against PIS on the benchmark from sampling literature, namely Digits~\citep{blessing2024beyond}, in Section~\ref{sec:digits}. Complete definitions and parameter settings are given in Appendix~\ref{app_sec:benchmarks}.

\begin{figure*}[!t]
    \centering
    \begin{minipage}{0.32\textwidth}
        \centering
        \includegraphics[width=\linewidth]{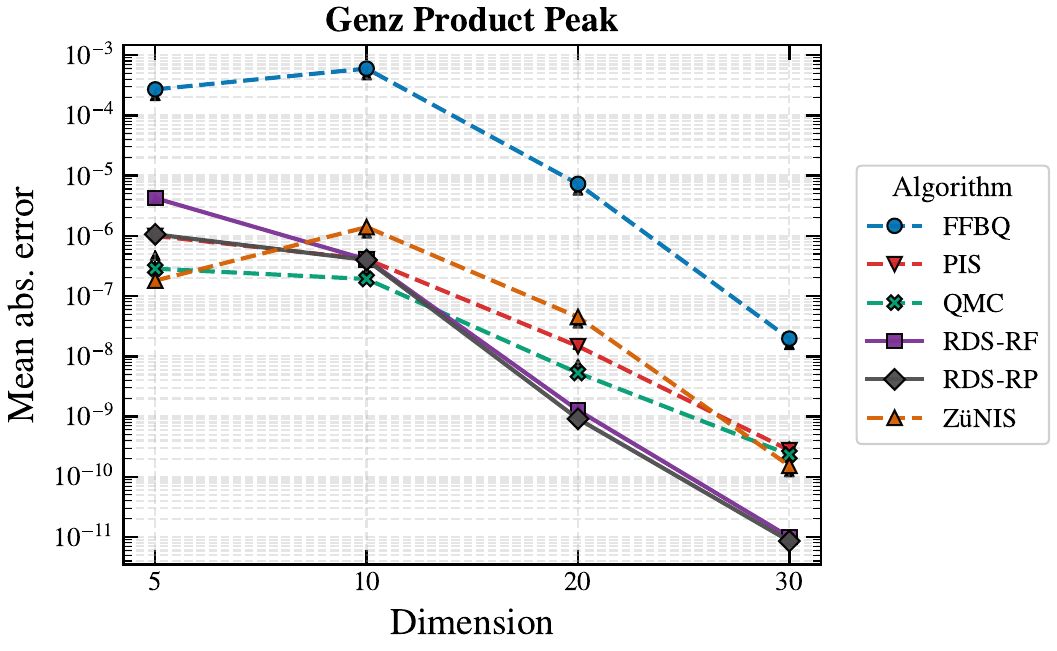}
    \end{minipage}
    \hfill
    \begin{minipage}{0.32\textwidth}
        \centering
        \includegraphics[width=\linewidth]{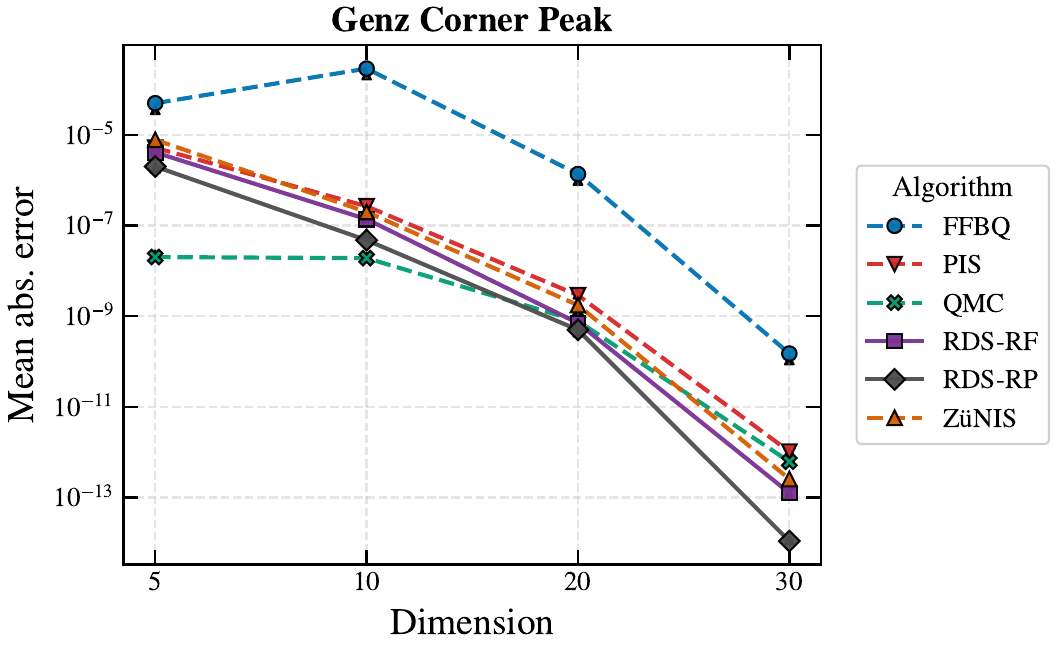}
    \end{minipage}
    \hfill
    \begin{minipage}{0.32\textwidth}
        \centering
        \includegraphics[width=\linewidth]{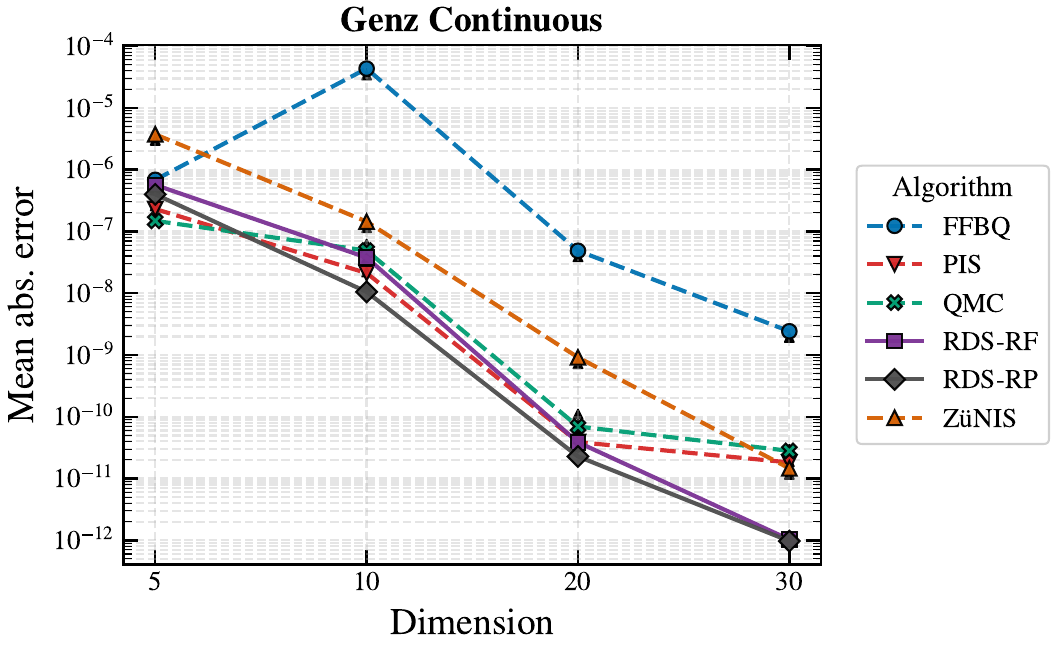}
    \end{minipage}

    \caption{
        Mean absolute error as a function of dimension on Genz test integrals for the number of samples $N=2^{20}$. Error bars show variability across independent
        evaluation runs.
    }
    \label{fig:dimension_scaling}
\end{figure*}

\begin{figure*}[!t]
    \centering
    \begin{minipage}{0.32\textwidth}
        \centering
        \includegraphics[width=\linewidth]{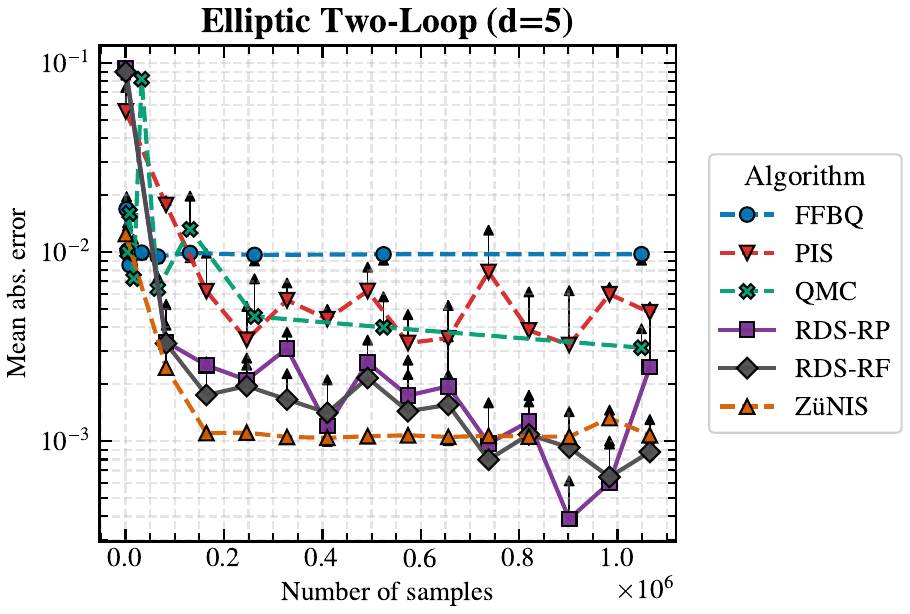}
    \end{minipage}
    \hfill
    \begin{minipage}{0.32\textwidth}
        \centering
        \includegraphics[width=\linewidth]{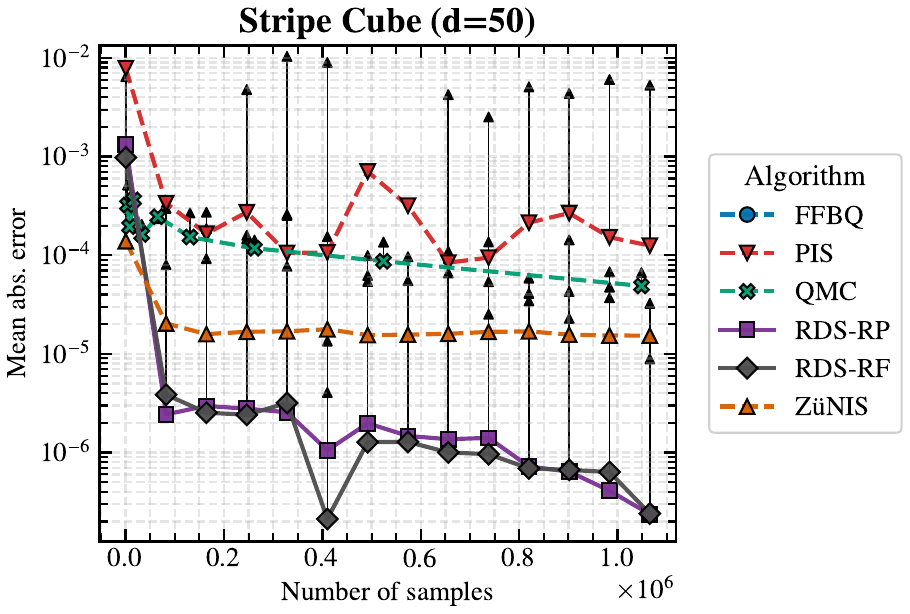}
    \end{minipage}
    \hfill
    \begin{minipage}{0.32\textwidth}
        \centering
        \includegraphics[width=\linewidth]{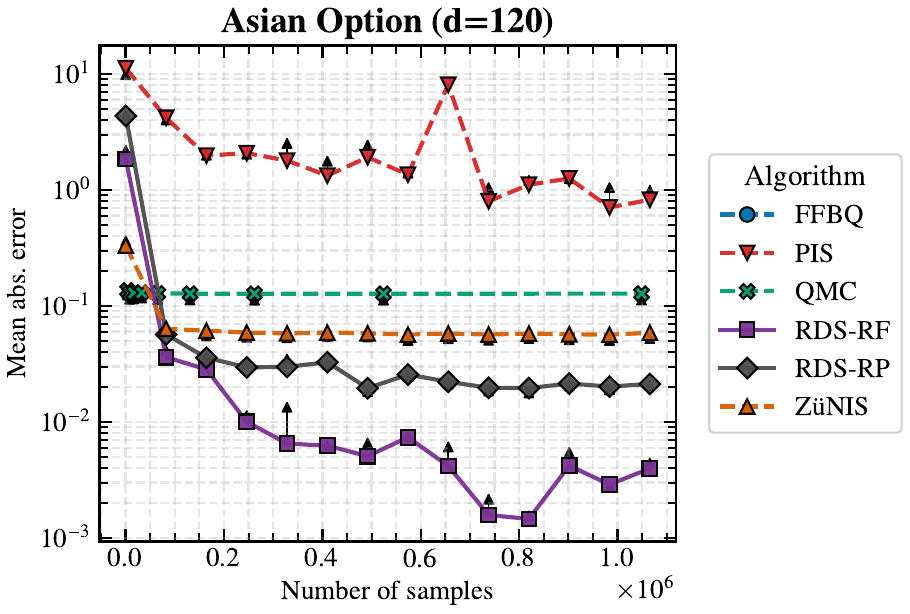}
    \end{minipage}

    \caption{
        Mean absolute error as a function of the number of samples used for estimation
        on representative benchmarks. Results for FFBQ are missing on some plots due to memory limitations of this algorithm in large dimensions.
    }
    \label{fig:nfe_scaling}
\end{figure*}

\textbf{Baselines}. We compare the reparameterization and reinforced variants of our method, denoted RDS-RP and RDS-RF, respectively, against Path Integral Sampler \citep[PIS]{zhang2022path}, Quasi-Monte Carlo~\citep{niederreiter1992random,dick2013high},
Fast Fourier Bayesian Quadrature~\citep[FFBQ]{warren2024fast}, and the normalizing flow importance sampling library Z\"uNIS \citep{deutschmann2024accelerating}. For PIS, we map $z\in\mathbb{R}^d$ to $x\in[0,1]^d$ with the sigmoid transformation $x=\sigma(z)$ and train on the change-of-variables target $\widetilde r(z)=r(\sigma(z))\left|\det J_\sigma(z)\right|$,
which preserves the original integral. We report the absolute normalizing-constant error $\Delta Z = |\widehat Z-Z|$, and show mean and standard deviation across independent evaluation runs. Full experimental details are in Appendix~\ref{app_sec:experimental_details}.

Overall, RDS performs well across a broad range of targets with substantially different geometry. In particular, it remains competitive on smooth, multimodal, boundary-concentrated, and high-dimensional problems, indicating that the learned proposal is not specialized to a particular class of integrands. The strong performance across both the dimension and target-evaluation sweeps also shows that the method can effectively adapt its sampling distribution to the regions that contribute most to the integral. Figure~
\ref{fig:dimension_scaling} illustrates the dependence of the estimation error on dimension, while Figure~\ref{fig:nfe_scaling} compares methods at matched sample budgets. The full results are given in Appendix~\ref{app_sec:experiments_results}. 

\subsection{Problems with smooth target functions}

RDS-RP is closely related to PIS, as both optimize a diffusion proposal using pathwise reparameterization gradients. The main difference is that RDS-RP evolves directly on the constrained domain and handles the boundary by reflection. This is particularly helpful for targets such as Genz Corner Peak and Stripe Cube, where a substantial amount of mass may lie close to the boundary. In PIS, the transformation to an unconstrained space sends such points far into the tails, which makes these regions harder to represent accurately. This is consistent with the stronger performance of RDS-RP on boundary-concentrated benchmarks. Besides normalizing-constant estimation, we inspect samples generated by the
learned reverse process. Figure~\ref{fig:synthetic_samples} shows representative
RDS samples for several cube-supported targets. These visualizations provide
a complementary check that the learned dynamics capture the geometry of the
target rather than only producing accurate estimates of $Z$.

\subsection{Problems with zero-valued target regions}
\label{sec:zero-reward-reg}
The path-measure and REINFORCE results require measurable $r(x)$ bounded above
and away from zero, but several benchmarks vanish on subsets of $\cD$. For
these targets we use $r_\varepsilon(x):=r(x)+\varepsilon$ with a small $\varepsilon>0$. Since $Z_\varepsilon=Z+\varepsilon\vol(\cD)$,
we estimate $Z_\varepsilon$ and report the corrected estimate
$\widehat Z=\widehat Z_\varepsilon-\varepsilon\vol(\cD)$, where $\vol(\cD)=1$ for the unit cube.

\begin{wrapfigure}{r}{0.34\textwidth}
    \centering
    \vspace{-1.5em}
    \includegraphics[width=0.85\linewidth]{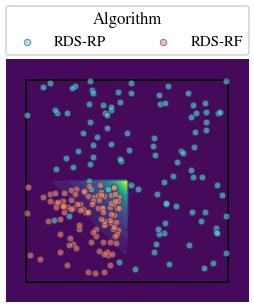}
    \caption{Visual sample comparison between RDS-RP and RDS-RF on Genz Discontinuous.}
    \label{fig:genz_discont}
    \vspace{-0.5cm}
\end{wrapfigure}

Zero-valued regions are particularly challenging for optimal-control methods trained with reparameterization: it is not applicable to discontinuous functions. When the target is locally constant, $\nabla_x \log r(x)$ provides no information about how probability mass should be moved, while at the boundary of such regions the derivative may be undefined. This issue is present in both Genz Discontinuous and the Asian option benchmark, whose targets vanish over subsets of the domain. On these problems, RDS-RF consistently outperforms the reparameterized RDS-RP variant, demonstrating the advantage of differentiating the trajectory distribution rather than the terminal target itself. The samples for Genz Discontinuous in Figure~\ref{fig:genz_discont} illustrate this behavior particularly clearly. RDS-RF successfully concentrates probability mass inside the non-zero-valued region and recovers the sharp support geometry despite the absence of informative local gradients outside it. This improved adaptation also translates into more accurate estimates of $Z$ on both benchmarks.

\subsection{High-Dimensional Sampling}
\label{sec:digits}

\begin{figure*}[!t]
    \centering
    \def\imgw{0.215\textwidth}
    \def\ms#1#2{\begin{tabular}[c]{@{}c@{}}#1\\[-2pt]{\scriptsize$\pm$#2}\end{tabular}}
    \def\digitsrows#1{%
        \setlength{\tabcolsep}{3pt}%
        \begin{tabular}[b]{@{}lcc@{}}
            \toprule
            Method & $\mathcal{W}_2 \downarrow$ & ELBO $\uparrow$ \\
            \midrule \addlinespace[#1]
            RDS-RP & \ms{\textbf{145.40}}{0.76} & \ms{\textbf{$-$154.50}}{0.08} \\ \addlinespace[#1]
            RDS-RF & \ms{161.87}{0.63}          & \ms{$-$163.75}{0.42} \\ \addlinespace[#1]
            PIS    & \ms{185.56}{0.33}          & \ms{$-$172.16}{0.67} \\ \addlinespace[#1]
            \bottomrule
        \end{tabular}}
    \sbox{\digitstab}{\footnotesize\digitsrows{0pt}}
    \setlength{\digitsgap}{\dimexpr(\imgw-\ht\digitstab-\dp\digitstab)/4\relax}

    \setlength{\tabcolsep}{4pt}
    \begin{tabular}{@{}cccc@{}}
        \includegraphics[width=\imgw]{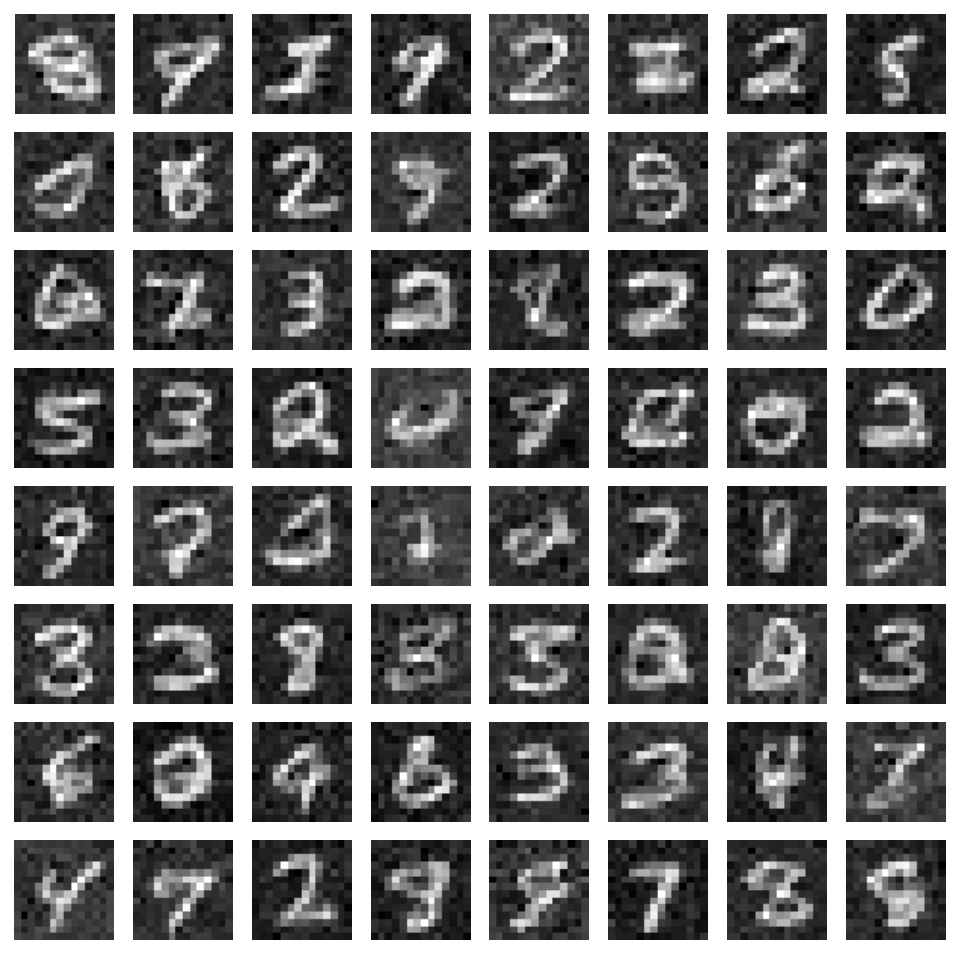} &
        \includegraphics[width=\imgw]{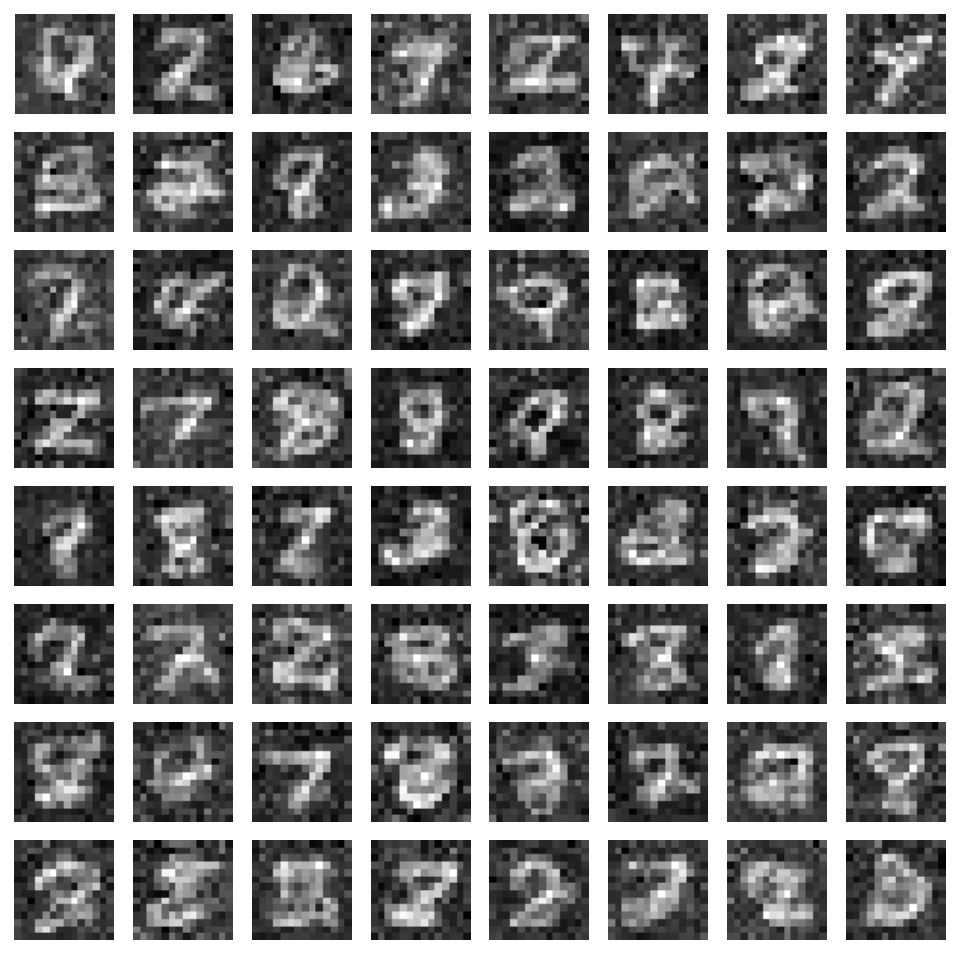} &
        \includegraphics[width=\imgw]{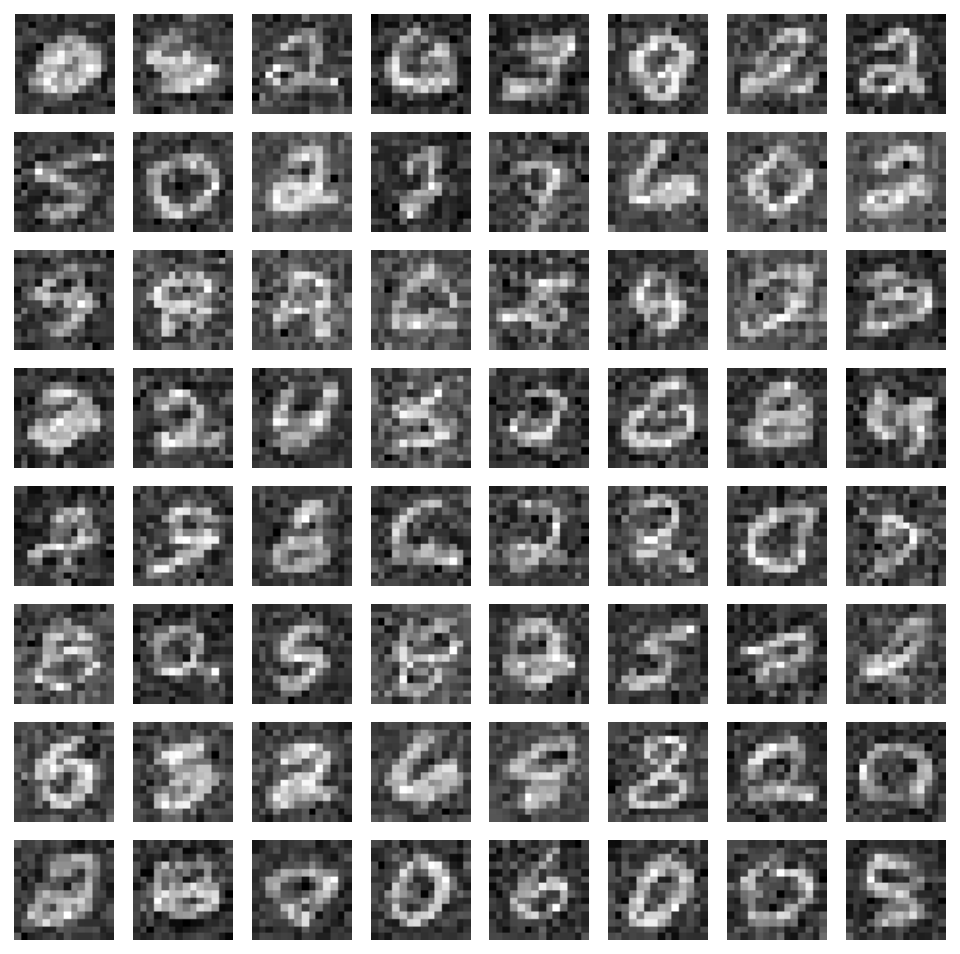} &
        {\footnotesize\digitsrows{\digitsgap}} \\[2pt]
        \small RDS-RP & \small RDS-RF & \small PIS & \small Metrics \\
    \end{tabular}

    \caption{
        Sample visualizations and metrics on the Digits benchmark~\citep{blessing2024beyond}.
    }
    \label{fig:digits}
\end{figure*}

In addition, to demonstrate the benefits of learning a reflected diffusion proposal (RDS) on a bounded domain in comparison to an unconstrained diffusion proposal (PIS) from another angle, we consider the high-dimensional Digits benchmark from the diffusion sampling literature~\citep{blessing2024beyond}. The target $\widetilde r(z)$ is a NICE~\citep{dinh2014nice} normalizing flow model trained on downsampled $14 \times 14$ images of MNIST~\citep{lecun1998gradient}. While the images naturally lie in a unit hypercube, the benchmark was designed for $\mathbb{R}^d$ sampling, so the normalizing flow was trained on images preprocessed with a shifted logit transform. Mirroring the previously described setup for PIS, to define an equivalent target with the same value of $Z$ on $\overline{\cD}=[0,1]^d$, we train RDS on $r(x)=\widetilde r(\sigma^{-1}(x))\left|\det J_{\sigma^{-1}}(x)\right|$. As discussed in Section~\ref{sec:training_objective}, our learned model also acts as a sampler for the normalized target $\pi=r/Z$. And, since the ground truth value of $Z$ is unknown, we compare RDS against PIS using sampling metrics, namely ELBO, which lower-bounds the value of $\log Z$, and $\mathcal{W}_2$ to the set of ground truth samples (see~\cite{blessing2024beyond} and Appendix~\ref{app_sec:experiments} for more details).

RDS-RP and PIS in essence have an identical KL training target, and the gradients are identically computed through reparameterization. We parameterize both models with neural networks of the same architecture and use the same number of discretization steps, so the comparison targets the only crucial difference that lies in the underlying diffusion dynamics: Brownian motion on $\mathbb{R}^d$ and reflected Brownian motion on $[0,1]^d$. Figure~\ref{fig:digits} presents the comparison results, demonstrating that RDS-RP and RDS-RF outperform PIS across both metrics.

\section{Conclusion}
\vspace{-0.2cm}

We presented reflected diffusion samplers for integration over bounded domains. The
construction replaces the Gaussian reference process of existing diffusion samplers with a
reflected variance-exploding process that leaves the uniform distribution on $\cD$ invariant,
where trajectories never leave the domain, and a Girsanov change of measure for reflected
diffusions supplies both the training objective and the importance weights used to estimate
$Z$. Across a number of benchmarks, in our experiments the developed approach supplies accurate estimates of the integral. While the proposed approach is currently limited to non-negative targets, one can consider a decomposition $\int_{\cD} r(x)\,\dd x = \int_{\cD} r^{+}(x)\,\dd x - \int_{\cD} r^{-}(x)\,\dd x$, then estimate both parts as discussed in Section~\ref{sec:zero-reward-reg}. We leave the analysis and evaluation of this adaptation for future work.

\begin{ack}
This research was supported in part through computational resources of HPC facilities at HSE University~\citep{kostenetskiy2021hpc}. The work of Kirill Korolev, Tigran Ramazyan and Nikita Morozov was supported by the grant for research centers in the field of AI provided by the Ministry of Economic Development of the Russian Federation in accordance with the agreement 000000C313925P4E0002 and the agreement with HSE University №139-15-2025-009. Artur Goldman is supported by the ETH AI Center through an ETH AI Center doctoral fellowship.
\end{ack}

\bibliography{bibliography}

\begin{thebibliography}{59}
\providecommand{\natexlab}[1]{#1}
\providecommand{\url}[1]{\texttt{#1}}
\expandafter\ifx\csname urlstyle\endcsname\relax
  \providecommand{\doi}[1]{doi: #1}\else
  \providecommand{\doi}{doi: \begingroup \urlstyle{rm}\Url}\fi

\bibitem[Akhound-Sadegh et~al.(2024)Akhound-Sadegh, Rector-Brooks, Bose, Mittal, Lemos, Liu, Sendera, Ravanbakhsh, Gidel, Bengio, Malkin, and Tong]{akhound-sadegh2024iterated}
Akhound-Sadegh, T., Rector-Brooks, J., Bose, J., Mittal, S., Lemos, P., Liu, C.-H., Sendera, M., Ravanbakhsh, S., Gidel, G., Bengio, Y., Malkin, N., and Tong, A.
\newblock Iterated denoising energy matching for sampling from boltzmann densities.
\newblock In \emph{Forty-first International Conference on Machine Learning}, 2024.
\newblock URL \url{https://openreview.net/forum?id=gVjMwLDFoQ}.

\bibitem[Alfaro et~al.(2008)Alfaro, Hilhorst, and Matano]{alfaro2008singular}
Alfaro, M., Hilhorst, D., and Matano, H.
\newblock The singular limit of the allen--cahn equation and the fitzhugh--nagumo system.
\newblock \emph{Journal of Differential Equations}, 245\penalty0 (2):\penalty0 505--565, 2008.

\bibitem[Berner et~al.(2024)Berner, Richter, and Ullrich]{berner2024an}
Berner, J., Richter, L., and Ullrich, K.
\newblock An optimal control perspective on diffusion-based generative modeling.
\newblock \emph{Transactions on Machine Learning Research}, 2024.
\newblock ISSN 2835-8856.
\newblock URL \url{https://openreview.net/forum?id=oYIjw37pTP}.

\bibitem[Berner et~al.(2026)Berner, Richter, Sendera, Rector-Brooks, and Malkin]{berner2025discrete}
Berner, J., Richter, L., Sendera, M., Rector-Brooks, J., and Malkin, N.
\newblock From discrete-time policies to continuous-time diffusion samplers: Asymptotic equivalences and faster training.
\newblock \emph{Transactions on Machine Learning Research}, 2026.
\newblock ISSN 2835-8856.
\newblock URL \url{https://openreview.net/forum?id=xLE3xJUuDO}.
\newblock Expert Certification.

\bibitem[Blessing et~al.(2024)Blessing, Jia, Esslinger, Vargas, and Neumann]{blessing2024beyond}
Blessing, D., Jia, X., Esslinger, J., Vargas, F., and Neumann, G.
\newblock Beyond {ELBOs}: A large-scale evaluation of variational methods for sampling.
\newblock \emph{International Conference on Machine Learning (ICML)}, 2024.

\bibitem[Bonciani et~al.(2016)Bonciani, Del~Duca, Frellesvig, Henn, Moriello, and Smirnov]{Bonciani_2016}
Bonciani, R., Del~Duca, V., Frellesvig, H., Henn, J.~M., Moriello, F., and Smirnov, V.~A.
\newblock Two-loop planar master integrals for higgs → 3 partons with full heavy-quark mass dependence.
\newblock \emph{Journal of High Energy Physics}, 2016\penalty0 (12), December 2016.
\newblock ISSN 1029-8479.
\newblock \doi{10.1007/jhep12(2016)096}.
\newblock URL \url{http://dx.doi.org/10.1007/JHEP12(2016)096}.

\bibitem[Borowka et~al.(2018)Borowka, Heinrich, Jahn, Jones, Kerner, Schlenk, and Zirke]{Borowka_2018}
Borowka, S., Heinrich, G., Jahn, S., Jones, S., Kerner, M., Schlenk, J., and Zirke, T.
\newblock pysecdec: A toolbox for the numerical evaluation of multi-scale integrals.
\newblock \emph{Computer Physics Communications}, 222:\penalty0 313–326, January 2018.
\newblock ISSN 0010-4655.
\newblock \doi{10.1016/j.cpc.2017.09.015}.
\newblock URL \url{http://dx.doi.org/10.1016/j.cpc.2017.09.015}.

\bibitem[Briol et~al.(2019)Briol, Oates, Girolami, Osborne, and Sejdinovic]{briol2019probabilistic}
Briol, F.-X., Oates, C.~J., Girolami, M., Osborne, M.~A., and Sejdinovic, D.
\newblock Probabilistic integration.
\newblock \emph{Statistical Science}, 34\penalty0 (1):\penalty0 1--22, 2019.

\bibitem[Burdzy et~al.(2004)Burdzy, Chen, and Sylvester]{burdzy2004heat}
Burdzy, K., Chen, Z.-Q., and Sylvester, J.
\newblock The heat equation and reflected brownian motion in time-dependent domains.
\newblock \emph{The Annals of Probability}, 32\penalty0 (1B):\penalty0 775--804, 2004.

\bibitem[Cattiaux(1988)]{cattiaux1988time}
Cattiaux, P.
\newblock Time reversal of diffusion processes with a boundary condition.
\newblock \emph{Stochastic processes and their Applications}, 28\penalty0 (2):\penalty0 275--292, 1988.

\bibitem[Chen et~al.(2016)Chen, Georgiou, and Pavon]{chen2016relation}
Chen, Y., Georgiou, T.~T., and Pavon, M.
\newblock On the relation between optimal transport and schr{\"o}dinger bridges: A stochastic control viewpoint.
\newblock \emph{Journal of Optimization Theory and Applications}, 169\penalty0 (2):\penalty0 671--691, 2016.

\bibitem[Chen et~al.(2021)Chen, Georgiou, and Pavon]{chen2021stochastic}
Chen, Y., Georgiou, T.~T., and Pavon, M.
\newblock Stochastic control liaisons: Richard sinkhorn meets gaspard monge on a schrodinger bridge.
\newblock \emph{Siam Review}, 63\penalty0 (2):\penalty0 249--313, 2021.

\bibitem[Deutschmann \& G{\"o}tz(2024)Deutschmann and G{\"o}tz]{deutschmann2024accelerating}
Deutschmann, N. and G{\"o}tz, N.
\newblock Accelerating hep simulations with neural importance sampling.
\newblock \emph{Journal of High Energy Physics}, 2024\penalty0 (3):\penalty0 83, 2024.

\bibitem[Dick et~al.(2013)Dick, Kuo, and Sloan]{dick2013high}
Dick, J., Kuo, F.~Y., and Sloan, I.~H.
\newblock High-dimensional integration: the quasi-monte carlo way.
\newblock \emph{Acta Numerica}, 22:\penalty0 133--288, 2013.

\bibitem[Dinh et~al.(2014)Dinh, Krueger, and Bengio]{dinh2014nice}
Dinh, L., Krueger, D., and Bengio, Y.
\newblock {NICE}: Non-linear independent components estimation.
\newblock \emph{arXiv preprint arXiv:1410.8516}, 2014.

\bibitem[Fishman et~al.(2023)Fishman, Klarner, Bortoli, Mathieu, and Hutchinson]{fishman2023diffusion}
Fishman, N., Klarner, L., Bortoli, V.~D., Mathieu, E., and Hutchinson, M.~J.
\newblock Diffusion models for constrained domains.
\newblock \emph{Transactions on Machine Learning Research}, 2023.
\newblock ISSN 2835-8856.
\newblock URL \url{https://openreview.net/forum?id=xuWTFQ4VGO}.
\newblock Expert Certification.

\bibitem[Foote(1984)]{foote1984regularity}
Foote, R.~L.
\newblock Regularity of the distance function.
\newblock \emph{Proceedings of the American Mathematical Society}, 92\penalty0 (1):\penalty0 153--155, 1984.

\bibitem[Genz(1984)]{genz1984testing}
Genz, A.
\newblock Testing multidimensional integration routines.
\newblock In \emph{Proc. of international conference on Tools, methods and languages for scientific and engineering computation}, pp.\  81--94, 1984.

\bibitem[Glasserman(2004)]{glasserman2004monte}
Glasserman, P.
\newblock \emph{Monte Carlo methods in financial engineering}, volume~53.
\newblock Springer, 2004.

\bibitem[Gunter et~al.(2014)Gunter, Osborne, Garnett, Hennig, and Roberts]{gunter2014sampling}
Gunter, T., Osborne, M.~A., Garnett, R., Hennig, P., and Roberts, S.~J.
\newblock Sampling for inference in probabilistic models with fast bayesian quadrature.
\newblock \emph{Advances in neural information processing systems}, 27, 2014.

\bibitem[Guo et~al.(2026)Guo, Tao, and Chen]{guo2026complexity}
Guo, W., Tao, M., and Chen, Y.
\newblock Complexity analysis of normalizing constant estimation: from jarzynski equality to annealed importance sampling and beyond.
\newblock In \emph{The Fourteenth International Conference on Learning Representations}, 2026.
\newblock URL \url{https://openreview.net/forum?id=96fJALwotm}.

\bibitem[Hammersley \& Handscomb(1964)Hammersley and Handscomb]{hammersley1964monte}
Hammersley, J.~M. and Handscomb, D.~C.
\newblock \emph{Monte Carlo Methods}.
\newblock Methuen \& Co. London; John Wiley \& Sons, New York, 1964.

\bibitem[Havens et~al.(2025)Havens, Miller, Yan, Domingo-Enrich, Sriram, Levine, Wood, Hu, Amos, Karrer, Fu, Liu, and Chen]{havens2025adjoint}
Havens, A.~J., Miller, B.~K., Yan, B., Domingo-Enrich, C., Sriram, A., Levine, D.~S., Wood, B.~M., Hu, B., Amos, B., Karrer, B., Fu, X., Liu, G.-H., and Chen, R. T.~Q.
\newblock Adjoint sampling: Highly scalable diffusion samplers via adjoint matching.
\newblock In \emph{Forty-second International Conference on Machine Learning}, 2025.
\newblock URL \url{https://openreview.net/forum?id=6Eg1OrHmg2}.

\bibitem[Holk et~al.(2026)Holk, Strauch, and Trottner]{holk2026statistical}
Holk, A., Strauch, C., and Trottner, L.
\newblock Statistical guarantees for denoising reflected diffusion models.
\newblock \emph{Journal of Machine Learning Research}, 27\penalty0 (116):\penalty0 1--50, 2026.

\bibitem[Kaarnioja(2025)]{kaarnioja2025explicit}
Kaarnioja, V.
\newblock Explicit solutions of genz test integrals.
\newblock \emph{Applied Mathematics Letters}, 163:\penalty0 109444, 2025.

\bibitem[Kass \& Raftery(1995)Kass and Raftery]{kass1995bayes}
Kass, R.~E. and Raftery, A.~E.
\newblock Bayes factors.
\newblock \emph{Journal of the american statistical association}, 90\penalty0 (430):\penalty0 773--795, 1995.

\bibitem[Kloek \& Van~Dijk(1978)Kloek and Van~Dijk]{kloek1978bayesian}
Kloek, T. and Van~Dijk, H.~K.
\newblock Bayesian estimates of equation system parameters: an application of integration by monte carlo.
\newblock \emph{Econometrica: Journal of the Econometric Society}, pp.\  1--19, 1978.

\bibitem[Kong \& Chaudhuri(2020)Kong and Chaudhuri]{kong2020expressive}
Kong, Z. and Chaudhuri, K.
\newblock The expressive power of a class of normalizing flow models.
\newblock In \emph{International conference on artificial intelligence and statistics}, pp.\  3599--3609. PMLR, 2020.

\bibitem[Kool et~al.(2019)Kool, van Hoof, and Welling]{kool2019buy}
Kool, W., van Hoof, H., and Welling, M.
\newblock Buy 4 {REINFORCE} samples, get a baseline for free!
\newblock In \emph{ICLR 2019 Workshop: Deep Reinforcement Learning Meets Structured Prediction}, 2019.
\newblock URL \url{https://openreview.net/forum?id=r1lgTGL5DE}.

\bibitem[Kostenetskiy et~al.(2021)Kostenetskiy, Chulkevich, and Kozyrev]{kostenetskiy2021hpc}
Kostenetskiy, P.~S., Chulkevich, R., and Kozyrev, V.
\newblock Hpc resources of the higher school of economics.
\newblock In \emph{Journal of Physics: Conference Series}, volume 1740, pp.\  012050. IOP Publishing, 2021.

\bibitem[Le~Gall(2016)]{le2016brownian}
Le~Gall, J.-F.
\newblock \emph{Brownian motion, martingales, and stochastic calculus}, volume 274.
\newblock Springer, 2016.

\bibitem[LeCun et~al.(1998)LeCun, Bottou, Bengio, and Haffner]{lecun1998gradient}
LeCun, Y., Bottou, L., Bengio, Y., and Haffner, P.
\newblock Gradient-based learning applied to document recognition.
\newblock \emph{Proceedings of the IEEE}, 86\penalty0 (11):\penalty0 2278--2324, 1998.

\bibitem[Liu et~al.(2023)Liu, Chen, Theodorou, and Tao]{liu2023mirror}
Liu, G.-H., Chen, T., Theodorou, E., and Tao, M.
\newblock Mirror diffusion models for constrained and watermarked generation.
\newblock \emph{Advances in Neural Information Processing Systems}, 36:\penalty0 42898--42917, 2023.

\bibitem[Loper(2020)]{loper2020uniform}
Loper, J.
\newblock Uniform ergodicity for brownian motion in a bounded convex set.
\newblock \emph{Journal of Theoretical Probability}, 33\penalty0 (1):\penalty0 22--35, 2020.

\bibitem[Lou \& Ermon(2023)Lou and Ermon]{lou2023reflected}
Lou, A. and Ermon, S.
\newblock Reflected diffusion models.
\newblock In \emph{International Conference on Machine Learning}, pp.\  22675--22701. PMLR, 2023.

\bibitem[L’Ecuyer(2009)]{l2009quasi}
L’Ecuyer, P.
\newblock Quasi-monte carlo methods with applications in finance.
\newblock \emph{Finance and Stochastics}, 13\penalty0 (3):\penalty0 307--349, 2009.

\bibitem[Midgley et~al.(2023)Midgley, Stimper, Simm, Sch{\"o}lkopf, and Hern{\'a}ndez-Lobato]{midgley2023flow}
Midgley, L.~I., Stimper, V., Simm, G. N.~C., Sch{\"o}lkopf, B., and Hern{\'a}ndez-Lobato, J.~M.
\newblock Flow annealed importance sampling bootstrap.
\newblock In \emph{The Eleventh International Conference on Learning Representations}, 2023.
\newblock URL \url{https://openreview.net/forum?id=XCTVFJwS9LJ}.

\bibitem[M{\"u}ller et~al.(2019)M{\"u}ller, McWilliams, Rousselle, Gross, and Nov{\'a}k]{muller2019neural}
M{\"u}ller, T., McWilliams, B., Rousselle, F., Gross, M., and Nov{\'a}k, J.
\newblock Neural importance sampling.
\newblock \emph{ACM Transactions on Graphics (ToG)}, 38\penalty0 (5):\penalty0 1--19, 2019.

\bibitem[Nicoli et~al.(2021)Nicoli, Anders, Funcke, Hartung, Jansen, Kessel, Nakajima, and Stornati]{nicoli2021estimation}
Nicoli, K.~A., Anders, C.~J., Funcke, L., Hartung, T., Jansen, K., Kessel, P., Nakajima, S., and Stornati, P.
\newblock Estimation of thermodynamic observables in lattice field theories with deep generative models.
\newblock \emph{Physical review letters}, 126\penalty0 (3):\penalty0 032001, 2021.

\bibitem[Niederreiter(1992)]{niederreiter1992random}
Niederreiter, H.
\newblock \emph{Random number generation and quasi-Monte Carlo methods}.
\newblock SIAM, 1992.

\bibitem[No{\'e} et~al.(2019)No{\'e}, Olsson, K{\"o}hler, and Wu]{noe2019boltzmann}
No{\'e}, F., Olsson, S., K{\"o}hler, J., and Wu, H.
\newblock Boltzmann generators: Sampling equilibrium states of many-body systems with deep learning.
\newblock \emph{Science}, 365\penalty0 (6457):\penalty0 eaaw1147, 2019.

\bibitem[O'Hagan(1991)]{o1991bayes}
O'Hagan, A.
\newblock Bayes--hermite quadrature.
\newblock \emph{Journal of statistical planning and inference}, 29\penalty0 (3):\penalty0 245--260, 1991.

\bibitem[Ott et~al.(2023)Ott, Tiemann, Hennig, and Briol]{ott2023bayesian}
Ott, K., Tiemann, M., Hennig, P., and Briol, F.-X.
\newblock Bayesian numerical integration with neural networks.
\newblock In \emph{Uncertainty in Artificial Intelligence}, pp.\  1606--1617. PMLR, 2023.

\bibitem[Owen(1997)]{owen1997monte}
Owen, A.~B.
\newblock Monte carlo variance of scrambled net quadrature.
\newblock \emph{SIAM Journal on Numerical Analysis}, 34\penalty0 (5):\penalty0 1884--1910, 1997.

\bibitem[Pilipenko(2014)]{pilipenko2014introduction}
Pilipenko, A.
\newblock \emph{An introduction to stochastic differential equations with reflection}.
\newblock Universit{\"a}tsverlag Potsdam, 2014.

\bibitem[Rasmussen \& Ghahramani(2003)Rasmussen and Ghahramani]{rasmussen2003bayesian}
Rasmussen, C.~E. and Ghahramani, Z.
\newblock Bayesian monte carlo.
\newblock \emph{Advances in neural information processing systems}, pp.\  505--512, 2003.

\bibitem[Rezende \& Mohamed(2015)Rezende and Mohamed]{rezende2015variational}
Rezende, D. and Mohamed, S.
\newblock Variational inference with normalizing flows.
\newblock In \emph{International conference on machine learning}, pp.\  1530--1538. PMLR, 2015.

\bibitem[Schuss(2015)]{schuss2015brownian}
Schuss, Z.
\newblock \emph{Brownian dynamics at boundaries and interfaces}, volume 758.
\newblock Springer, 2015.

\bibitem[Sendera et~al.(2024)Sendera, Kim, Mittal, Lemos, Scimeca, Rector-Brooks, Adam, Bengio, and Malkin]{sendera2024improved}
Sendera, M., Kim, M., Mittal, S., Lemos, P., Scimeca, L., Rector-Brooks, J., Adam, A., Bengio, Y., and Malkin, N.
\newblock Improved off-policy training of diffusion samplers.
\newblock \emph{Advances in Neural Information Processing Systems}, 37:\penalty0 81016--81045, 2024.

\bibitem[Song et~al.(2021)Song, Sohl-Dickstein, Kingma, Kumar, Ermon, and Poole]{song2021scorebased}
Song, Y., Sohl-Dickstein, J., Kingma, D.~P., Kumar, A., Ermon, S., and Poole, B.
\newblock Score-based generative modeling through stochastic differential equations.
\newblock In \emph{International Conference on Learning Representations}, 2021.
\newblock URL \url{https://openreview.net/forum?id=PxTIG12RRHS}.

\bibitem[Taraldsen(2018)]{taraldsen2018optimal}
Taraldsen, G.
\newblock Optimal learning from the doob-dynkin lemma.
\newblock \emph{arXiv preprint arXiv:1801.00974}, 2018.

\bibitem[Vargas et~al.(2023)Vargas, Grathwohl, and Doucet]{vargas2023denoising}
Vargas, F., Grathwohl, W.~S., and Doucet, A.
\newblock Denoising diffusion samplers.
\newblock In \emph{International Conference on Learning Representations}, 2023.
\newblock URL \url{https://openreview.net/forum?id=8pvnfTAbu1f}.

\bibitem[Vargas et~al.(2024)Vargas, Padhy, Blessing, and N{\"u}sken]{vargas2024transport}
Vargas, F., Padhy, S., Blessing, D., and N{\"u}sken, N.
\newblock Transport meets variational inference: Controlled monte carlo diffusions.
\newblock In \emph{The Twelfth International Conference on Learning Representations}, 2024.
\newblock URL \url{https://openreview.net/forum?id=PP1rudnxiW}.

\bibitem[Virtanen et~al.(2020)Virtanen, Gommers, Oliphant, Haberland, Reddy, Cournapeau, Burovski, Peterson, Weckesser, Bright, et~al.]{virtanen2020scipy}
Virtanen, P., Gommers, R., Oliphant, T.~E., Haberland, M., Reddy, T., Cournapeau, D., Burovski, E., Peterson, P., Weckesser, W., Bright, J., et~al.
\newblock Scipy 1.0: fundamental algorithms for scientific computing in python.
\newblock \emph{Nature methods}, 17\penalty0 (3):\penalty0 261--272, 2020.

\bibitem[Warren \& Ramos(2024)Warren and Ramos]{warren2024fast}
Warren, H. and Ramos, F.
\newblock Fast fourier bayesian quadrature.
\newblock In \emph{International Conference on Artificial Intelligence and Statistics}, pp.\  4555--4563. PMLR, 2024.

\bibitem[Williams(1987)]{williams1987time}
Williams, R.
\newblock On time-reversal of reflected brownian motions.
\newblock In \emph{Seminar on Stochastic Processes, 1987}, pp.\  265--276. Springer, 1987.

\bibitem[Williams(1992)]{williams1992simple}
Williams, R.~J.
\newblock Simple statistical gradient-following algorithms for connectionist reinforcement learning.
\newblock \emph{Machine learning}, 8\penalty0 (3):\penalty0 229--256, 1992.

\bibitem[Wirnsberger et~al.(2020)Wirnsberger, Ballard, Papamakarios, Abercrombie, Racani{\`e}re, Pritzel, Jimenez~Rezende, and Blundell]{wirnsberger2020targeted}
Wirnsberger, P., Ballard, A.~J., Papamakarios, G., Abercrombie, S., Racani{\`e}re, S., Pritzel, A., Jimenez~Rezende, D., and Blundell, C.
\newblock Targeted free energy estimation via learned mappings.
\newblock \emph{The Journal of Chemical Physics}, 153\penalty0 (14), 2020.

\bibitem[Zhang \& Chen(2022)Zhang and Chen]{zhang2022path}
Zhang, Q. and Chen, Y.
\newblock Path integral sampler: A stochastic control approach for sampling.
\newblock In \emph{International Conference on Learning Representations}, 2022.
\newblock URL \url{https://openreview.net/forum?id=_uCb2ynRu7Y}.

\end{thebibliography}
\bibliographystyle{neurips_2026}
\appendix

\section{Related work}
\label{related_work}

\textbf{Quasi-Monte Carlo} (QMC) approximates
\eqref{eq:problem} using low-discrepancy point sets:
$\widehat Z_N=N^{-1}\sum_{i=1}^N r(x_i)$.
Unlike independent Monte Carlo samples, QMC points are
constructed to cover $[0,1]^d$ more uniformly, typically
using digital nets or lattice rules
\citep{niederreiter1992random,dick2013high}.
For integrands satisfying suitable regularity conditions,
QMC can achieve faster convergence than ordinary
Monte Carlo. Randomized QMC methods, such as scrambled
Sobol sequences, retain low-discrepancy properties while
allowing randomized error estimation \citep{owen1997monte}.
However, conventional QMC point sets are constructed
independently of the integrand and do not adapt to the
location of its mass.

\textbf{Bayesian quadrature} (BQ) models the integrand
using a Gaussian process and conditions on its
evaluations at selected points
\citep{o1991bayes,rasmussen2003bayesian}.
Since integration is a linear functional, the Gaussian
process posterior induces a Gaussian posterior over
the integral, providing both an estimate and a measure
of uncertainty. Evaluation points can also be selected
adaptively to reduce the posterior uncertainty
\citep{gunter2014sampling}.
For fixed kernel hyperparameters, the posterior
variance depends on the evaluation locations rather
than the observed function values. Consequently, its
uncertainty estimates depend on the adequacy of the
chosen prior \citep{briol2019probabilistic}.
Standard GP-based BQ requires $O(n^3)$ computation
for $n$ observations, although scalable variants,
including Fast Fourier Bayesian Quadrature,
reduce this cost \citep{warren2024fast}.

\textbf{Importance sampling} (IS) estimates
\eqref{eq:problem} using a proposal density $q$
on $\mathcal D$:
\begin{equation}
    \widehat{Z}_N
    =
    \frac{1}{N}\sum_{i=1}^{N}
    \frac{r(x_i)}{q(x_i)},
    \qquad x_i\stackrel{\mathrm{i.i.d.}}{\sim}q.
    \label{eq:is-estimate}
\end{equation}
Provided $q$ is positive wherever $r$ is positive
up to sets of measure zero, the estimator is unbiased.
Its variance satisfies
\begin{equation}
    \operatorname{Var}_q(\widehat{Z}_N)
    =
    \frac{Z^2}{N}\chi^2(\pi\Vert q),
    \label{eq:is-variance}
\end{equation}
where $\pi=r/Z$ and $\chi^2$ denotes the chi-square
divergence. The optimal proposal is $q=\pi$, which
gives zero variance \citep{hammersley1964monte}.
Since sampling directly from $\pi$ is generally
intractable, the effectiveness of IS depends on
constructing a proposal that accurately represents
the regions where the target concentrates.

\textbf{Normalizing flows} (NF) provide adaptive
proposal densities through invertible transformations
with tractable Jacobian determinants
\citep{dinh2014nice,rezende2015variational}.
For sampling from unnormalized targets,
\citet{midgley2023flow} combine normalizing flows
with annealed importance sampling to improve
the coverage of poorly represented regions.
For numerical integration, \citet{muller2019neural}
train a flow directly from evaluations of the
unnormalized integrand, without requiring samples
from the normalized target. They derive
importance-weighted gradients for minimizing
$\mathrm{KL}(\pi\Vert q_\theta)$ and also consider
the chi-square divergence, which directly controls
the IS variance. Their method is applied to
light-transport integration and path guiding.
Normalizing flows have also been applied to sampling equilibrium
distributions and estimating free-energy differences in molecular
systems~\citep{noe2019boltzmann}. Building on neural importance sampling,
\citet{deutschmann2024accelerating} introduce
Z\"uNIS, a flow-based importance-sampling library
for high-energy physics simulations.
Normalizing flows require invertible architectures
with tractable Jacobians, which restricts the
available transformations; expressivity limitations
for particular flow architectures have been studied
by \citet{kong2020expressive}.

\textbf{Diffusion samplers}
\citep{zhang2022path,vargas2023denoising,
berner2024an,sendera2024improved,
akhound-sadegh2024iterated,havens2025adjoint}
learn stochastic dynamics for sampling from
unnormalized target distributions without requiring
invertible neural architectures.
Unlike conventional diffusion models
\citep{song2021scorebased}, which learn from
samples of the target distribution, diffusion
samplers are trained using its unnormalized density.
The Path Integral Sampler (PIS)
\citep{zhang2022path} formulates sampling as a
stochastic optimal control problem, where a
reference diffusion is reweighted by a terminal
potential determined by the target. Its training
objective combines a quadratic control cost with
a terminal density term, obtained through
Girsanov's theorem.
Denoising Diffusion Samplers (DDS)
\citep{vargas2023denoising} instead formulate
sampling through the time reversal of a diffusion
that transforms the target towards a tractable
reference distribution. Both approaches lead
to related path-space optimization objectives.
The resulting path measures also permit
normalizing-constant estimation through
importance sampling
\citep{zhang2022path,guo2026complexity}.

Most of these sampling methods are formulated
using diffusions on $\mathbb R^d$ and do not
directly enforce bounded-domain constraints.
Reflected diffusions have previously been used
for constrained generative modeling
\citep{lou2023reflected,fishman2023diffusion},
while mirror diffusion models
\citep{liu2023mirror} handle constraints through
transformations to an unconstrained dual space.
In contrast to these data-driven generative
models, our method learns a reflected diffusion
proposal from an unnormalized target, without
requiring target samples. It operates directly
on the bounded integration domain and handles
zero-valued regions by adding a positive offset
to the integrand and correcting the final
estimate.

\section{Formal notation and assumptions}
\label{sec:formal_notation_and_assumptions}
\subsection{Notation and conventions}
\label{sec:notation}
\textbf{General notation.} The domain $\cD\subset\R^d$ is bounded and convex. Its closure, boundary,
and Lebesgue volume are $\overline{\cD}$, $\partial\cD$, and
$\vol(\cD)$. For topological spaces $E$ and $F$, we denote by $C(E;F)$ the set of
continuous maps from $E$ to $F$ and abbreviate $C(E):=C(E;\R)$. When $E$ is
compact and $F$ is metric, this space carries the uniform topology
induced by metric $d_\infty(f, g)=\sup_{x\in E}d_F(f(x), g(x))$ unless
stated otherwise. For $E\subseteq\R^d$ and $f:E\to\R^m$,
\begin{equation}
    \operatorname{Lip}_{x\in E}(f)
    :=\sup_{x\neq y\in E}
      \frac{\lVert f(x)-f(y)\rVert_2}{\lVert x-y\rVert_2}.
\end{equation}
Forward and reverse time are $t$ and $\rt$. Forward and reverse processes are
denoted by $X$ and $Y$. Symbols used in only one proof are defined locally.

For probability measures on a common measurable space, $\mu_1\ll\mu_0$
denotes absolute continuity, and
$\mu_1\ll\mu_0\ll\mu_1$ denotes equivalence, meaning mutual absolute
continuity. We also use
\begin{align}
    \lVert\mu_1-\mu_0\rVert_{\mathrm{TV}}
    &:=\sup_E|\mu_1(E)-\mu_0(E)|,\\
    \KL{\mu_1}{\mu_0}&:=\int\log\!\left(\frac{\dd\mu_1}{\dd\mu_0}\right)\,\dd\mu_1,\\
    \chi^2(\mu_1\Vert\mu_0)
    &:=\int\left(\frac{\dd\mu_1}{\dd\mu_0}-1\right)^2\,\dd\mu_0.
\end{align}
The last two are $+\infty$ when $\mu_1\not\ll\mu_0$. The same notation
applies to densities through their induced measures.
The $\chi^2$-divergence and the variances of the integrable importance weights
below are understood as extended nonnegative quantities and may equal
$+\infty$.

If a continuous process $X$ is defined under a construction probability $\mathbf P$, we denote its induced path law on $C([0,1];\overline{\cD})$ by $\Law(X)=X_\#\mathbf P$. We write $\zeta$ for a generic path, with $\zeta(t)$ its value at time $t$, while $X_t$ denotes the corresponding coordinate of a random path $X$.

For a real-valued continuous semimartingale $X$, a level $a\in\R$, and time
$t\geq0$, we write $\LT_t^a(X)$ for the right local time of $X$ at $a$
accumulated up to $t$, using the version that is right-continuous in $a$.
Informally, it measures the quadratic variation accumulated while $X$ lies
immediately to the right of $a$. Its Tanaka characterization and one-sided
occupation approximation are given by
\citet[Propositions~9.2 and~9.9]{le2016brownian}.

Finally, we set
\begin{equation}
    A(t):=\int_0^t g(s)^2\,\dd s.
    \label{eq:rve_variance_clock}
\end{equation}

\textbf{Domain related definitions.} We use the following concrete meaning of boundary regularity. We say that
$\partial\cD$ is $C^1$ near $z\in\partial\cD$ if there are an open
neighborhood $\mathsf U_z$ of $z$ and a function
$b_z\in C^1(\mathsf U_z)$ such that
\[
    \cD\cap\mathsf U_z=\{x\in\mathsf U_z:b_z(x)<0\},\qquad
    \partial\cD\cap\mathsf U_z=\{x\in\mathsf U_z:b_z(x)=0\},
\]
and $\nabla b_z(x)\neq0$ on $\partial\cD\cap\mathsf U_z$. Any such $b_z$ is
called a \emph{local defining function}. The nonvanishing-gradient condition
makes its zero set a $C^1$ hypersurface, by the implicit function theorem, and
the gradient is perpendicular to that hypersurface.

The defining function is not unique, but the oriented unit normal is. Indeed,
at $z$ the gradients of any two local defining functions are nonzero and
perpendicular to the same tangent hyperplane, so they are scalar multiples.
Because both functions are negative on the interior side, that scalar is
positive. Consequently, we define the outward and inward unit normals by
\begin{equation}
    \mathbf n_{\mathrm{out}}(z)
    :=\frac{\nabla b_z(z)}{\lVert\nabla b_z(z)\rVert_2},
    \qquad
    \mathbf n_{\mathrm{in}}(z):=-\mathbf n_{\mathrm{out}}(z).
\end{equation}
These definitions do not depend on the choice of $b_z$.
Assumption~\ref{assumption:theoretical_domain} imposes the stronger
$C^\infty$ version of this condition at every boundary point.

\textbf{Solutions of reflected SDEs.} Fix $T>0$ and a continuous reflected process
$X=(X_t)_{0\leq t\leq T}$ in $\overline{\cD}$. A \emph{scalar regulator}
for $X$ is the real-valued boundary-push process
$\ell=(\ell_t)_{0\leq t\leq T}$: almost surely, the map $t\mapsto\ell_t$ is
continuous and nondecreasing, $\ell_0=0$, and its Stieltjes measure is
supported on the boundary-contact times,
\[
    \int_0^t\ind_{\{X_s\notin\partial\cD\}}\,\dd\ell_s=0,
    \qquad 0\leq t\leq T.
\]
The associated vector reflection term is
$L_t:=\int_0^t\mathbf n_{\mathrm{in}}(X_s)\,\dd\ell_s$.
Fix a Borel drift
$f:\overline{\cD}\times[0,T]\to\R^d$, a deterministic
$g\in L^2([0,T])$, and an initial law $\mu$ on $\overline{\cD}$.
A \emph{weak solution} of the form \eqref{eq:reflected_sde} on $[0,T]$ is a tuple
\[
    (\Omega,\cF,(\cF_t)_{0\leq t\leq T},\mathbf P;
      \xi_0,W,X,\ell)
\]
such that the filtered probability space satisfies the usual conditions,
$\xi_0$ is an $\cF_0$-measurable $\overline{\cD}$-valued random variable
with law $\mu$, $W$ is a $d$-dimensional Brownian motion relative to
$(\cF_t)$, and $X$ and $\ell$ are continuous $(\cF_t)$-adapted processes.
They satisfy $X_0=\xi_0$, $X_t\in\overline{\cD}$, the preceding regulator
conditions,
$\int_0^T\lVert f(X_s,s)\rVert_2\,\dd s<\infty$ almost surely, and, almost
surely, for every $t\leq T$,
\[
    X_t=X_0+\int_0^t f(X_s,s)\,\dd s
      +\int_0^t g(s)\,\dd W_s
      +\int_0^t\mathbf n_{\mathrm{in}}(X_s)\,\dd\ell_s.
\]
We call the full tuple a weak realization and the law of $X$ its weak
solution law. 

For a strong solution, we are given
$
    (\Omega,\cF,(\cF_t)_{0\leq t\leq T},\mathbf P;\xi_0,W),
$
where the filtered probability space satisfies the usual conditions,
$\xi_0$ is an $\cF_0$-measurable $\overline{\cD}$-valued random variable with
law $\mu$, and $W$ is a $d$-dimensional Brownian motion relative to
$(\cF_t)$. A \emph{strong solution} for these data is a continuous adapted
pair $(X,\ell)$ satisfying the state, regulator, integrability, and
integral-equation conditions in the weak-solution definition above, including
$X_0=\xi_0$, and such that, for every $t\leq T$,
\[
    \sigma(X_s,\ell_s:0\leq s\leq t)
    \subseteq
    \sigma(\xi_0,W_s:0\leq s\leq t)\vee\mathcal N_{\mathbf P}.
\]
Here $\mathcal A\vee\mathcal B:=\sigma(\mathcal A\cup\mathcal B)$ for
sigma-fields $\mathcal A$ and $\mathcal B$, and $\mathcal N_{\mathbf P}$ consists of the $\mathbf P$-null events and
their complements. Thus the condition allows changes on null events, but no
additional randomness: both $X$ and $\ell$ up to time $t$ are determined by
the prescribed initial variable and Brownian path up to time $t$.

This is the normal-reflection specialization of
\citet[Definitions~2.2.2 and~2.3.1 and
Remark~2.2.4]{pilipenko2014introduction}. The strong formulation is the one
used in the source's discussion preceding Theorem~2.5.4. The source keeps
$\xi_0$ separate and requires $X_0=\xi_0$. We retain that convention for both
weak and strong solutions. Since our SDE displays prescribe only an initial
law, the notation $X_0\sim\mu$ is shorthand for including a variable
$\xi_0\sim\mu$ in the realization and imposing $X_0=\xi_0$.

A displayed \emph{law definition} is the primary specification of a path
measure. If it is given by an SDE, existence and uniqueness in law must be
established separately. An \emph{SDE characterization} instead describes an
already-defined law: weak existence identifies that law as one solution law,
whereas pathwise uniqueness, strong existence, and uniqueness in law are
additional conclusions only when explicitly stated.

\subsection{Assumptions for theoretical results}
\label{sec:theoretical_assumptions}
\begin{assumption}
    \label{assumption:bounded_convex_domain}
    The domain $\cD\subset\R^d$ is bounded and convex.
\end{assumption}
Assumption~\ref{assumption:bounded_convex_domain} is the standing domain
condition throughout the appendix.

The integration problem and the algorithmic construction can be formulated
without assuming a smooth boundary. The local $C^1$ condition in
Section~\ref{sec:notation} is enough to define the unit normal. Some of our
theoretical arguments, however, invoke standard results on normal reflection,
Neumann heat equations, and reflected time reversal under stronger boundary
regularity~\citep{holk2026statistical, pilipenko2014introduction, burdzy2004heat,cattiaux1988time,williams1987time}.
We therefore isolate the following $C^\infty$ condition as an additional
theory-only assumption rather than include it in the general problem
description.
\begin{assumption}
    \label{assumption:theoretical_domain}
    The boundary of $\cD$ is $C^\infty$: for every
    $z\in\partial\cD$, there are a neighborhood $\mathsf U_z$ of $z$ and a
    function $b_z\in C^\infty(\mathsf U_z)$ such that
    \begin{equation}
        \begin{aligned}
            \cD\cap\mathsf U_z
            =\{x\in\mathsf U_z:b_z(x)<0\},\;
            &\partial\cD\cap\mathsf U_z
            =\{x\in\mathsf U_z:b_z(x)=0\},\\
            \nabla b_z(x)\neq0\; 
            &\text{ for }x\in\mathsf U_z\cap\partial\cD.
        \end{aligned}
    \end{equation}
\end{assumption}
Assumption~\ref{assumption:theoretical_domain} is stronger than the local
$C^1$ condition in Section~\ref{sec:notation}, so the unit normals are defined
at every boundary point and can be used in the vector reflection term.

A density is determined only Lebesgue-almost everywhere: changing it on a
null set leaves its law unchanged, but makes pointwise boundary values and
$\nabla\log p_t$ ambiguous. Let $k_u(y,x)$ denote the transition density of
unit normally reflected Brownian motion after elapsed time $u>0$. For $t>0$
we choose the representative
\[
    p_t(x):=
    \int_{\cD}
    k_{A(t)}(y,x)\,\pi(y)\,\dd y.
\]
The schedule makes the elapsed time in this kernel strictly positive.
Under Assumption~\ref{assumption:theoretical_domain},
\citet[Theorems~2.4(i) and~2.5]{burdzy2004heat} show that the kernel is
strictly positive and jointly $C^2$ in its spatial variables up to the
boundary at positive elapsed times. Compactness therefore permits
differentiation under the integral, so this is a strictly positive
$C^2(\overline{\cD})$ representative of the marginal density. The static-domain
specialization of \citet[Theorem~2.9]{burdzy2004heat} also gives its
homogeneous Neumann boundary condition. This choice is needed for the
pointwise reverse drift in
Lemma~\ref{lemma:scheduled_reflected_time_reversal}. Throughout, a score on
$\overline{\cD}$ means its stated continuous extension from the interior, and
at $t=0$ the representative is $p_0=\pi$.

The integration problem \eqref{eq:problem} permits $r$ to vanish. Some results use the following stronger target
assumption.
\begin{assumption}
    \label{assumption:positive_target_reduction}
    The target $r$ admits a Borel measurable extension to
    $\overline{\cD}$, denoted by $r$ as well, and there are constants
    $0<m_r\leq M_r<\infty$ such that
    \[
        m_r\leq r(x)\leq M_r,
        \qquad x\in\overline{\cD}.
    \]
\end{assumption}
We henceforth use this extension and $\pi=r/Z$ on $\overline{\cD}$ as well.
The two-sided bounds directly make $r$ and the normalized density $\pi$
bounded above and away from zero on $\overline{\cD}$.
Proposition~\ref{proposition:optimal_proposal},
Theorem~\ref{theorem:kl_reference}, and
Theorem~\ref{theorem:reinforce_gradient} invoke this assumption.
Extending these results to a target with zeros requires a separate support
and absolute-continuity analysis and is not asserted here.

\subsection{Score and learned-control assumptions}
Let $p_t$ be the forward marginal density from the
target-started law \eqref{eq:optimal_forward_path_law}.

\begin{assumption}
    \label{assumption:closed_endpoint_score_continuity}
    For every $t\in[0,1]$, the chosen representative $p_t$ is strictly
    positive on $\overline{\cD}$, and its interior score admits an extension
    such that
    \[
        \bigl[(t,x)\mapsto\nabla\log p_t(x)\bigr]
        \in C\bigl([0,1]\times\overline{\cD};\R^d\bigr).
    \]
\end{assumption}
\begin{assumption}
    \label{assumption:closed_endpoint_score}
    \begin{equation}
        \sup_{t\in[0,1]}
        \operatorname{Lip}_{x\in\overline{\cD}}
        \bigl(\nabla\log p_t(x)\bigr)<\infty.
        \label{eq:closed_endpoint_score_lipschitz}
    \end{equation}
\end{assumption}

Note that these assumptions impose additional regularity on $\pi$.
Indeed, since $p_0=\pi$, Assumption~
\ref{assumption:closed_endpoint_score_continuity} requires
$\nabla\log\pi$ to exist in $\cD$ and to admit a continuous extension to
$\overline{\cD}$. Moreover, the positive-time densities satisfy the
homogeneous Neumann condition, so joint continuity down to $t=0$ forces
\[
    \mathbf n_{\mathrm{out}}^\top\nabla\log\pi=0
    \qquad\text{on }\partial\cD.
\]
Under Assumption~\ref{assumption:positive_target_reduction},
$\nabla\log\pi=-\nabla V$. Thus this requires the interior gradient of $V$
to extend continuously to $\overline{\cD}$ and to satisfy
$\mathbf n_{\mathrm{out}}^\top\nabla V=0$ on $\partial\cD$.
Assumption~\ref{assumption:closed_endpoint_score} additionally requires
$\nabla V$ to be Lipschitz on $\overline{\cD}$. The two assumptions are
stronger than these target-only consequences: they also require the
positive-time scores to converge jointly to the target score as $t\downarrow
0$ and to remain uniformly spatially Lipschitz over the full time interval.
None of these score conditions follows from $V\in C(\overline{\cD})$ alone.

The preceding assumptions concern the exact score and hence the ideal reverse
law. We separately impose the following regularity conditions on the learned
field used in the implemented law.

\begin{assumption}
    \label{assumption:control_continuity}
    \[
        \bigl[(t,x)\mapsto s_\theta(x,t)\bigr]
        \in C\bigl([0,1]\times\overline{\cD};\R^d\bigr).
    \]
\end{assumption}

Since $[0,1]\times\overline{\cD}$ is compact, joint
continuity makes $s_\theta$ bounded and Borel.

\begin{assumption}
    \label{assumption:control_spatial_lipschitz}
    \[
        \sup_{t\in[0,1]}
        \operatorname{Lip}_{x\in\overline{\cD}}
        \bigl(s_\theta(x,t)\bigr)<\infty.
    \]
\end{assumption}
Assumption~\ref{assumption:control_spatial_lipschitz} is used to prove
uniqueness for the learned SDE.

\section{Existence and reversal of the reflected laws}
\label{sec:formal_reflected_reduction}

The arguments in this section separate two tasks. First, the reflected-SDE
and Skorokhod results of
\citet[Theorems~2.5.1--2.5.3]{pilipenko2014introduction} and the
Yamada--Watanabe discussion preceding
\citet[Theorem~2.5.4]{pilipenko2014introduction} provide
well-posedness for equations with prescribed coefficients. We use them to
construct the zero-drift forward law and the continuous learned law, and
to upgrade the ideal reverse equation to a unique strong solution when its
score drift is spatially Lipschitz. The required domain hypotheses are
verified in Lemma~\ref{lemma:manuscript_domain_pilipenko_geometry}.

Second, well-posedness of an equation with a prescribed score drift does not
by itself identify that equation with the reversal of the forward law. We
first define the ideal reverse law by reversing the already constructed
target-started forward process. The time-homogeneous reflected-reversal
statement of \citet[Section~2, equation~(4)]{holk2026statistical}, followed by
the local clock calculation in
Lemma~\ref{lemma:scheduled_reflected_time_reversal},
identifies this particular reversal as a weak solution of the scheduled score
SDE on $[0,1)$. Score continuity is needed only to extend this
representation to the closed endpoint, while spatial Lipschitzness supplies
the subsequent uniqueness upgrade.

\subsection{Geometric prerequisites for reflected SDEs}
\label{sec:reflected_sde_geometric_prerequisites}

The following lemma supplies a technical prerequisite for the existence and
uniqueness arguments in this section. It verifies that the standing domain
assumptions imply the geometric conditions required by the reflected
Skorokhod and SDE results used below.

For $x\in\partial\cD$ and $\rho>0$, define the inward normal sets used in
\citet[Section~2.5]{pilipenko2014introduction} by
\[
    \mathsf N_{x,\rho}
    :=\{n\in\R^d:\lVert n\rVert_2=1,
        \ B(x-\rho n,\rho)\cap\cD=\varnothing\},
    \qquad
    \mathsf N_x:=\bigcup_{\rho>0}\mathsf N_{x,\rho}.
\]
The uniform exterior-sphere condition requires a radius
$\rho_{\mathrm{ext}}>0$ such that
$\mathsf N_x=\mathsf N_{x,\rho_{\mathrm{ext}}}\ne\varnothing$ for every
$x\in\partial\cD$. The uniform-cone condition in
\citet[Theorem~2.5.2]{pilipenko2014introduction} requires constants
$\delta_{\mathrm c}>0$ and $\alpha\in[0,1)$ such that, for every
$x\in\partial\cD$, some unit vector $v_x$ satisfies, for every
$y\in B(x,\delta_{\mathrm c})\cap\partial\cD$,
\begin{equation}
    \{z\in\R^d:\langle z-y,v_x\rangle
        \geq\alpha\lVert z-y\rVert_2\}
      \cap B(x,\delta_{\mathrm c})
      \subset\overline{\cD}.
    \label{eq:manuscript_domain_uniform_cone}
\end{equation}
The closure in \eqref{eq:manuscript_domain_uniform_cone} is essential because
the cone contains its boundary vertex $y$.

\begin{lemma}
    \label{lemma:manuscript_domain_pilipenko_geometry}
Under Assumptions~\ref{assumption:bounded_convex_domain} and
\ref{assumption:theoretical_domain}, $\cD$ satisfies the uniform
exterior-sphere and uniform-cone conditions just defined.
\end{lemma}

\begin{proof}
We verify the two conditions separately. For the exterior-sphere condition,
we place a ball tangent to the boundary at $x$ on the exterior side of its
supporting hyperplane. Convexity will keep the whole domain outside that
ball. We then check that smoothness leaves no other possible direction in
$\mathsf N_x$. This second step is needed for the equality
$\mathsf N_x=\mathsf N_{x,\rho_{\mathrm{ext}}}$ in the definition.

Fix $x\in\partial\cD$. The inward unit normal
$\mathbf n_{\mathrm{in}}(x)$ is unique by
Assumption~\ref{assumption:theoretical_domain}.
The supporting hyperplane at $x$ leaves a convex domain entirely on its
inward side, so
\[
    \langle y-x,\mathbf n_{\mathrm{in}}(x)\rangle\geq0,
    \qquad y\in\overline{\cD}.
\]
For any $\rho>0$, the point
$x-\rho\mathbf n_{\mathrm{in}}(x)$ is the center of a radius-$\rho$ ball
tangent at $x$ and placed on the exterior side. For every
$y\in\overline{\cD}$,
\[
\begin{aligned}
    \lVert y-(x-\rho\mathbf n_{\mathrm{in}}(x))\rVert_2^2
    &=\lVert y-x\rVert_2^2
      +2\rho\langle y-x,\mathbf n_{\mathrm{in}}(x)\rangle+\rho^2 \geq \rho^2.
\end{aligned}
\]
Thus every point of the domain is at least $\rho$ from the center, and hence
$B(x-\rho\mathbf n_{\mathrm{in}}(x),\rho)\cap\cD=\varnothing$ for every
$\rho>0$.

It remains to show that no other unit vector belongs to $\mathsf N_x$.
Suppose that $n\in\mathsf N_{x,\rho}$ for some $\rho>0$. Disjointness from
the open ball gives the following inequality first for $y\in\cD$ and then,
by continuity and density of $\cD$ in $\overline{\cD}$, for every
$y\in\overline{\cD}$:
\[
    2\rho\langle n,y-x\rangle+\lVert y-x\rVert_2^2\geq0.
\]
To identify the direction $n$, let $\gamma$ be any smooth boundary curve
with $\gamma(0)=x$. Substituting $y=\gamma(t)$ and dividing by $t>0$ gives
$\langle n,\gamma'(0)\rangle\geq0$ in the limit $t\downarrow0$. Approaching
through $t<0$ reverses the inequality and gives
$\langle n,\gamma'(0)\rangle\leq0$. Hence
$\langle n,\gamma'(0)\rangle=0$. Therefore $n$ is orthogonal to every tangent
vector at $x$. The smooth boundary has a one-dimensional normal space, so
$n$ is either $\mathbf n_{\mathrm{in}}(x)$ or
$-\mathbf n_{\mathrm{in}}(x)$. To select the sign, a local defining function
$b$ at $x$ gives
\[
    b\bigl(x+\varepsilon\mathbf n_{\mathrm{in}}(x)\bigr)
    =-\varepsilon\lVert\nabla b(x)\rVert_2+o(\varepsilon)<0
\]
for all sufficiently small $\varepsilon>0$. Hence
$x+\varepsilon\mathbf n_{\mathrm{in}}(x)\in\cD$. Substituting these
interior points in the same inequality gives
\[
    2\rho\varepsilon
    \langle n,\mathbf n_{\mathrm{in}}(x)\rangle+\varepsilon^2\geq0.
\]
Dividing by $\varepsilon$ and letting $\varepsilon\downarrow0$ gives
$\langle n,\mathbf n_{\mathrm{in}}(x)\rangle\geq0$. Thus
$n=\mathbf n_{\mathrm{in}}(x)$, and
\[
    \mathsf N_x=\mathsf N_{x,\rho}=\{\mathbf n_{\mathrm{in}}(x)\}
    \qquad \rho>0.
\]
Any fixed $\rho_{\mathrm{ext}}>0$ therefore verifies the uniform exterior-sphere condition.

For the cone condition, set $\alpha:=1/2$. Around each boundary point $x_0$,
choose a local defining function $b$ and use the inward normal at $x_0$ as a
fixed cone axis. After shrinking the neighborhood, continuity of $\nabla b$
makes its variation at most $\lVert\nabla b(x_0)\rVert_2/4$. Consequently, if
a displacement $h$ satisfies the cone inequality
$\langle h,v_{x_0}\rangle\geq\alpha\lVert h\rVert_2$, then the defining
function decreases strictly along $h$: the fixed-gradient term is at most
$-\lVert\nabla b(x_0)\rVert_2\lVert h\rVert_2/2$, whereas the gradient
variation contributes at most
$\lVert\nabla b(x_0)\rVert_2\lVert h\rVert_2/4$. Compactness of the boundary
will then make the neighborhood radius uniform. The value $1/2$ is only a
convenient fixed aperture. Any $\alpha\in(1/4,1)$ would work with the same
gradient-variation bound.

Fix $x_0\in\partial\cD$ and choose a defining function $b$ from
Assumption~\ref{assumption:theoretical_domain}, so locally
$\cD=\{b<0\}$ and $\nabla b(x_0)\ne0$. Since $\nabla b(x_0)$ points outward,
set
\[
    m:=\lVert\nabla b(x_0)\rVert_2,
    \qquad
    v_{x_0}:=-\frac{\nabla b(x_0)}{m}.
\]
Thus $v_{x_0}$ points inward. By continuity of $\nabla b$, after shrinking
the defining-function neighborhood we may choose $\rho_{x_0}>0$ such that
$B(x_0,2\rho_{x_0})$ lies in that neighborhood and
\[
    \sup_{w\in B(x_0,2\rho_{x_0})}
    \lVert\nabla b(w)-\nabla b(x_0)\rVert_2\leq\frac m4.
\]
This bound says that throughout the neighborhood, the gradient remains close
to the fixed outward direction at $x_0$.

Let $y\in B(x_0,\rho_{x_0})\cap\partial\cD$ and suppose that
$z\in B(x_0,\rho_{x_0})$ satisfies
\[
    \langle z-y,v_{x_0}\rangle
        \geq\alpha\lVert z-y\rVert_2.
\]
In other words, the displacement from $y$ to $z$ lies in a cone pointing in
the fixed inward direction $v_{x_0}$. Write $h:=z-y$. The segment from $y$
to $z$ remains in $B(x_0,\rho_{x_0})$. At every point $w$ on that segment,
\[
\begin{aligned}
    \langle\nabla b(w),h\rangle
    &\leq \langle\nabla b(x_0),h\rangle
       +\frac m4\lVert h\rVert_2 =-m\langle v_{x_0},h\rangle
       +\frac m4\lVert h\rVert_2
      \leq-m\left(\alpha-\frac14\right)\lVert h\rVert_2
      =-\frac m4\lVert h\rVert_2.
\end{aligned}
\]
The first term represents movement in the ideal inward direction at $x_0$. The gradient variation can cost at most $m\lVert h\rVert_2/4$, so the net
directional derivative remains negative. Since $b(y)=0$, the fundamental
theorem of calculus gives
\[
    b(z)-b(y)
    =\int_0^1\langle\nabla b(y+th),h\rangle\,\dd t
    \leq-\frac m4\lVert h\rVert_2.
\]
Hence $b(z)<0$ when $z\ne y$, while for $z=y$ the point already lies in
$\partial\cD$. Thus every such $z$ belongs to $\overline{\cD}$.

The radius $\rho_{x_0}$ constructed so far may depend on $x_0$. To obtain one
radius for the whole boundary, use that boundedness of $\cD$ makes
$\partial\cD$ compact. The sets $B(x_0,\rho_{x_0}/2)$ cover the boundary.
Choose a finite
subcover with centers $x_1,\ldots,x_N$ and set
\[
    \delta_{\mathrm c}:=\min_{1\leq i\leq N}\frac{\rho_{x_i}}2>0.
\]
For each $x\in\partial\cD$, choose $i$ such that
$x\in B(x_i,\rho_{x_i}/2)$ and set $v_x:=v_{x_i}$. Then
$B(x,\delta_{\mathrm c})\subset B(x_i,\rho_{x_i})$. Therefore both the cone
vertex $y$ and the point $z$ remain in the neighborhood where the preceding
local argument applies. It applies simultaneously to every
$y\in B(x,\delta_{\mathrm c})\cap\partial\cD$ and every point in the cone
with vertex $y$, axis $v_x$, aperture parameter $\alpha$, and intersection
with $B(x,\delta_{\mathrm c})$. This is exactly
\eqref{eq:manuscript_domain_uniform_cone}.
\end{proof}

\subsection{Forward reflected path law}

\begin{lemma}
    \label{lemma:forward_reflected_well_posedness}
    Under Assumptions~\ref{assumption:bounded_convex_domain} and
    \ref{assumption:theoretical_domain}, let $T\in(0,\infty)$ and let
    $b:[0,T]\times\overline{\cD}\to\R^d$ and
    $\sigma:[0,T]\times\overline{\cD}\to\R^{d\times d}$ be bounded,
    jointly continuous, and uniformly Lipschitz in space. For any
    probability law $\mu$ on $\overline{\cD}$, the normally reflected SDE
    \[
        \dd X_t=b(t,X_t)\,\dd t+\sigma(t,X_t)\,\dd W_t
        +\mathbf n_{\mathrm{in}}(X_t)\,\dd\ell_t,
        \qquad X_0\sim\mu,\quad 0\leq t\leq T,
    \]
    where $W$ is a standard $d$-dimensional Brownian motion and $\ell$ is
    the scalar normal regulator, has a pathwise-unique strong solution
    and is unique in law. In particular, the forward RVE SDE
    \eqref{eq:rve_sde} defines $\Pf$ and $\Qf^\star$, and unit normally
    reflected Brownian motion is well defined on $[0,\infty)$.
\end{lemma}

\begin{proof}
Extend $b$ and $\sigma$ constantly in time after $T$. This preserves
boundedness, joint continuity, and the uniform spatial Lipschitz bounds.
Assumptions~\ref{assumption:bounded_convex_domain} and
\ref{assumption:theoretical_domain}, through
Lemma~\ref{lemma:manuscript_domain_pilipenko_geometry}, supply the uniform
exterior-sphere and uniform-cone conditions. Together with boundedness and
continuity of the coefficients, these give weak existence by
\citet[Theorem~2.5.3]{pilipenko2014introduction}. The exterior-sphere
condition and uniform spatial Lipschitz bounds give pathwise uniqueness
of the process and scalar regulator by
\citet[Theorem~2.5.1 and Exercise~2.5.3]{pilipenko2014introduction}. The
Yamada--Watanabe discussion immediately preceding
\citet[Theorem~2.5.4]{pilipenko2014introduction} gives strong existence and
uniqueness in law. Restricting to $[0,T]$ proves the finite-interval claim.
No density assumption on $\mu$ is used.

For \eqref{eq:rve_sde}, take $T=1$, $b=0$, and $\sigma(t,x)=g(t)I_d$.
Equation~\eqref{eq:rve_noise_schedule} verifies the coefficient hypotheses.
Taking $b=0$ and $\sigma=I_d$ in the same global construction gives unit
normally reflected Brownian motion on $[0,\infty)$.
\end{proof}

\subsection{Uniform marginals and reversal of the reference law}
\label{sec:clock_stationarity}
Let $P_u(x,\dd y)$ be the transition kernel of unit normally reflected
Brownian motion on $\overline{\cD}$. We use $P_uf(x)$ for its action on
functions and $\mu P_u$ for its action on measures.

\begin{proposition}
    \label{proposition:rve_sde_marginal}
    Suppose Assumptions~\ref{assumption:bounded_convex_domain} and
    \ref{assumption:theoretical_domain} hold, and let $X$ follow
    \eqref{eq:rve_sde}. If $X_0\sim\nu$, then
    $X_t\sim\nu$ for every $t\in[0,1]$. If $X_0\sim\mu_0$ is arbitrary, then
    for every fixed $t\in(0,1]$,
    \begin{equation}
        A(t)\longrightarrow\infty
        \quad\Longrightarrow\quad
        \lVert\Law(X_t)-\nu\rVert_{\mathrm{TV}}\longrightarrow0
        \label{eq:rve_arbitrary_initial_tv_mixing}
    \end{equation}
    along any family of noise schedules. The limit varies the schedule, not
    $t$.
\end{proposition}

\begin{proof}
The two domain assumptions have different roles. Bounded convexity is the
hypothesis used for the stationarity and uniform-ergodicity conclusions of
\citet[Lemma~1]{loper2020uniform}, whereas smoothness keeps
\eqref{eq:rve_sde} within the manuscript's single-valued inward-normal and
unique-law framework. Moreover,
\eqref{eq:rve_noise_schedule} gives $g\in C([0,1])$ with
$\inf_{[0,1]}g>0$, so the clock $A$ is finite and strictly increasing.
For $u\in[0,A(1)]$, set $R_u:=X_{A^{-1}(u)}$.
The time-changed martingale
$\int_0^{A^{-1}(u)}g(s)\,\dd W_s$ has bracket $uI_d$, so it is Brownian. The
increasing clock preserves the boundary support of the scalar regulator
$\ell^R_u:=\ell_{A^{-1}(u)}$. Hence $R$ is unit normally reflected Brownian
motion and
\begin{equation}
    X_t=R_{A(t)},\qquad 0\leq t\leq1.
    \label{eq:rve_reference_clock}
\end{equation}
This constructs only the segment $0\leq u\leq A(1)$.
Assumptions~\ref{assumption:bounded_convex_domain} and
\ref{assumption:theoretical_domain} supply the domain hypotheses of
Lemma~\ref{lemma:forward_reflected_well_posedness}. The constant
coefficients $b=0$ and $\sigma=I_d$ satisfy its coefficient hypotheses.
The lemma therefore gives a global unit normally reflected Brownian
solution with the same initial law.
Uniqueness in law identifies the constructed segment with the restriction of
this global solution. By \citet[Lemma~1(1)]{loper2020uniform}, the global
solution is a homogeneous Markov process with transition kernel $P_u$. Hence,
conditionally on $R_0$, the law of $R_u$ is $P_u(R_0,\cdot)$. The limit
$u\to\infty$ below concerns this global Markov family, not an extension of the
finite realization.
If $X_0\sim\nu$, stationarity in
\citet[Lemma~1(1)]{loper2020uniform} proves the first claim.
For a general $\mu_0$, the uniform ergodicity in
\citet[Lemma~1(2)]{loper2020uniform} gives
\begin{equation}
    \lVert\mu_0P_u-\nu\rVert_{\mathrm{TV}}
    \leq\sup_{x\in\overline{\cD}}
       \lVert P_u(x,\cdot)-\nu\rVert_{\mathrm{TV}}
    \longrightarrow0
    \qquad (u\to\infty).
    \label{eq:rve_mixture_tv_bound}
\end{equation}
For fixed physical time $t$, \eqref{eq:rve_reference_clock} gives
$\Law(X_t)=\mu_0P_{A(t)}$. Substituting $u=A(t)$ into
\eqref{eq:rve_mixture_tv_bound} therefore proves
\eqref{eq:rve_arbitrary_initial_tv_mixing} whenever a family of noise
schedules makes $A(t)\to\infty$. Thus this is a limit in accumulated clock
time, not a limit $t\to\infty$.
\end{proof}

\begin{lemma}
    \label{lemma:rve_stationary_reversal}
    Suppose Assumptions~\ref{assumption:bounded_convex_domain} and
    \ref{assumption:theoretical_domain} hold. The reflected SDE
    \[
        \dd Y_{\rt}
        =g(1-\rt)\,\dd\overline W_{\rt}
        +\mathbf n_{\mathrm{in}}(Y_{\rt})\,\dd\overline\ell_{\rt},
        \qquad Y_0\sim\nu,\quad 0\leq\rt\leq1,
    \]
    admits a strong solution, has pathwise uniqueness, and is unique in law.
    Its unique solution law is the reversed reference law $\Pb$ defined in
    \eqref{eq:reverse_path_law_definitions}. The Brownian motion
    $\overline W$ and scalar regulator $\overline\ell$ in a realization of this
    law need not be literal pathwise reversals of the forward drivers.
\end{lemma}
\begin{proof}
Assumptions~\ref{assumption:bounded_convex_domain} and
\ref{assumption:theoretical_domain} supply the domain hypotheses of
Lemma~\ref{lemma:forward_reflected_well_posedness}. With $b=0$,
$\sigma(t,x)=g(t)I_d$, and initial law $\nu$, the schedule
\eqref{eq:rve_noise_schedule} supplies its coefficient hypotheses, so
$\Pf$ is well defined. Bounded convexity is used again in
\citet[Lemma~1]{loper2020uniform}, while $g>0$ makes the forward and reverse
clock changes invertible.
Let $X\sim\Pf$ and set $Y_{\rt}:=X_{1-\rt}$. By
\eqref{eq:reverse_path_law_definitions}, $\Law(Y)=\Pb$ by definition. By
the definition of $\Pf$, $X_0\sim\nu$. Equation
\eqref{eq:rve_reference_clock} gives a unit normally reflected Brownian motion
$R$, with $R_0=X_0\sim\nu$, such that
\[
    X_t=R_{A(t)},\qquad 0\leq t\leq1.
\]
Thus, for a Brownian motion $W'$ and a scalar regulator $\ell'$, the
unit-clock process satisfies
\[
    \dd R_u=\dd W'_u+\mathbf n_{\mathrm{in}}(R_u)\,\dd\ell'_u,
    \qquad R_0\sim\nu,\quad 0\leq u\leq A(1).
\]
By \citet[Lemma~1(1)]{loper2020uniform}, the initial law
$R_0\sim\nu$ makes $R$ stationary. We first reverse $R$ on the unit-clock
interval $[0,A(1)]$ by setting
\[
    \widehat R_u:=R_{A(1)-u},\qquad 0\leq u\leq A(1).
\]
By the reversibility conclusion of
\citet[Lemma~1(1)]{loper2020uniform},
\[
    \Law\bigl((\widehat R_u)_{0\leq u\leq A(1)}\bigr)
    =\Law\bigl((R_u)_{0\leq u\leq A(1)}\bigr).
\]
Indeed, repeated detailed balance reverses every finite-dimensional Markov
product, and path continuity upgrades equality of finite-dimensional
distributions to equality of laws on
$C([0,A(1)];\overline{\cD})$.

We next return from unit-clock time $u$ to reverse physical time $\rt$. Define
\[
    \overline A(\rt):=A(1)-A(1-\rt)
    =\int_0^{\rt}g(1-s)^2\,\dd s,
    \qquad 0\leq\rt\leq1.
\]
Since $A(0)=0$ and $A$ is increasing, $\overline A(0)=0$,
$\overline A(1)=A(1)$, and $\overline A$ maps $[0,1]$ increasingly onto
$[0,A(1)]$.
Then
\[
    Y_{\rt}=X_{1-\rt}=R_{A(1-\rt)}
    =R_{A(1)-\overline A(\rt)}
    =\widehat R_{\overline A(\rt)}.
\]
Applying the deterministic time change $u=\overline A(\rt)$ to both processes
in the preceding path-law identity therefore gives
\[
    \Law\bigl((Y_{\rt})_{0\leq\rt\leq1}\bigr)
    =\Law\bigl((R_{\overline A(\rt)})_{0\leq\rt\leq1}\bigr).
\]
We now time-change the unit-clock SDE displayed above. Define
\[
    Z_{\rt}:=R_{\overline A(\rt)},\qquad
    M_{\rt}:=W'_{\overline A(\rt)},\qquad
    \overline\ell_{\rt}:=\ell'_{\overline A(\rt)}.
\]
Because $\overline A'(\rt)=g(1-\rt)^2$, the continuous local martingale
$M$ has bracket $\overline A(\rt)I_d$. Hence
\[
    \overline W_{\rt}
    :=\int_0^{\rt}g(1-s)^{-1}\,\dd M_s
\]
has bracket $\rt I_d$ and is Brownian by the characterization in
\citet[Theorem~5.12]{le2016brownian}, while stochastic-integral associativity
from \citet[Theorem~5.6]{le2016brownian} gives
\[
    M_{\rt}=\int_0^{\rt}g(1-s)\,\dd\overline W_s.
\]
The same increasing time change preserves the boundary support of the scalar
regulator. Consequently,
\[
    \dd Z_{\rt}=g(1-\rt)\,\dd\overline W_{\rt}
       +\mathbf n_{\mathrm{in}}(Z_{\rt})\,\dd\overline\ell_{\rt},
    \qquad Z_0\sim\nu.
\]
Since $Z$ has path law $\Pb$ by the preceding law identity, this proves that
$\Pb$ is a weak solution law of the displayed SDE.

By \eqref{eq:rve_noise_schedule}, the coefficients $b=0$ and
$\sigma(\rt,x)=g(1-\rt)I_d$ are bounded and continuous. Independence of
$x$ gives spatial Lipschitz constant zero. Together with Assumptions~
\ref{assumption:bounded_convex_domain} and \ref{assumption:theoretical_domain},
these verify the hypotheses of
Lemma~\ref{lemma:forward_reflected_well_posedness} for $T=1$ and initial
law $\nu$, giving strong existence, pathwise uniqueness, and uniqueness
in law. Since the weak
solution just constructed has law $\Pb$, this is the unique solution law.
\end{proof}

\subsection{Ideal reverse path law}

\begin{lemma}
    \label{lemma:scheduled_reflected_time_reversal}
    Suppose Assumptions~\ref{assumption:bounded_convex_domain} and
    \ref{assumption:theoretical_domain} hold. Let $X\sim\Qf^\star$, let
    $p_t$ be the density of $X_t$, and set $Y_{\rt}:=X_{1-\rt}$ for
    $0\leq\rt\leq1$. Then the path law
    $\Qb^\star:=\Law((Y_{\rt})_{0\leq\rt\leq1})$ is well defined.

    If Assumption~\ref{assumption:closed_endpoint_score_continuity} also
    holds, then $\Qb^\star$ is a weak solution law of
    \eqref{eq:full_interval_ideal_sde} on $[0,1]$. If, in addition,
    Assumption~\ref{assumption:closed_endpoint_score} holds, then
    \eqref{eq:full_interval_ideal_sde} has pathwise uniqueness, a unique
    strong solution, and uniqueness in law. In particular, $\Qb^\star$ is
    its unique solution law.
\end{lemma}

\begin{proof}
\emph{Step 1: define the law.}
Assumptions~\ref{assumption:bounded_convex_domain} and
\ref{assumption:theoretical_domain} supply the domain hypotheses of
Lemma~\ref{lemma:forward_reflected_well_posedness}. With $b=0$ and
$\sigma(t,x)=g(t)I_d$, \eqref{eq:rve_noise_schedule} supplies its coefficient
hypotheses. Applying it with $T=1$ and initial law $\pi$ constructs
$X\sim\Qf^\star$ with continuous paths. Thus $Y_{\rt}:=X_{1-\rt}$ is
continuous on $[0,1]$ and $\Qb^\star=\Law(Y)$ is well defined.

\emph{Step 2: identify the SDE on $[0,1)$.}
Set $T:=A(1)$. By \eqref{eq:rve_reference_clock}, there is a unit normally
reflected Brownian motion $R$, started from the density $\pi$, with
\[
    X_t=R_{A(t)}.
\]
At this stage we need only a weak-solution representation of the already
defined reversed path: we must exhibit a filtered extension carrying a
Brownian motion and an adapted scalar normal regulator for which the reverse
SDE holds. Pathwise uniqueness and construction from a prescribed Brownian
motion are addressed separately in Step~4.

Our domain is open, connected, bounded, and has $C^\infty$ boundary. In the
source notation of \citet[Equation~(6)]{holk2026statistical}, take
the conductivity $f_{\mathrm H}\equiv1/2$ in their constrained forward model.
Then $\nabla f_{\mathrm H}=0$, $\sqrt{2f_{\mathrm H}}I_d=I_d$, and their
conormal coefficient is
$\nu_{\mathrm H}=f_{\mathrm H}\mathbf n_{\mathrm{in}}
=\mathbf n_{\mathrm{in}}/2$. Thus their forward model is unit normally
reflected Brownian motion, and its vector reflection term is
$\mathbf n_{\mathrm{in}}\,\dd(\ell^{\mathrm H}/2)$, which has exactly our
inward-normal form. The reflected time-reversal statement in
\citet[Equation~(4)]{holk2026statistical} and its subsequent
specialization to the conductivity model therefore apply to the law of $R$.
Under $f_{\mathrm H}\equiv1/2$, the source's specialized reverse drift
$\nabla f_{\mathrm H}+2f_{\mathrm H}\nabla\log q_{T-u}$ reduces to
$\nabla\log q_{T-u}$.

If $q_u$ denotes the density of $R_u$ and
$\widehat R_u:=R_{T-u}$, the cited result provides an enlargement of the
reversed natural filtration carrying a Brownian motion
$B^{\leftarrow}$ and the reversed boundary clock
$\overline\ell^{\mathrm H}$ such that, for $0\leq u<T$,
\[
    \dd\widehat R_u
      =\nabla\log q_{T-u}(\widehat R_u)\,\dd u
       +\dd B_u^{\leftarrow}
       +\frac12\mathbf n_{\mathrm{in}}(\widehat R_u)\,
        \dd\overline\ell_u^{\mathrm H}.
\]
Consequently, the law of the actual reversed path $\widehat R$ is a weak
solution law of the displayed SDE on $[0,T)$. The cited result makes no
uniqueness assertion for these coefficients.
Set
\[
    \lambda_u^{\leftarrow}:=\frac12\overline\ell_u^{\mathrm H}.
\]
Then the last term is
$\mathbf n_{\mathrm{in}}(\widehat R_u)\,\dd\lambda_u^{\leftarrow}$.
This forced rescaling only converts scalar local-time conventions. It does
not choose or change the reflected path. Define the reverse clock
\[
    \overline A(\rt):=T-A(1-\rt)
      =\int_0^{\rt}g(1-s)^2\,\dd s.
\]
Then $Y_{\rt}=\widehat R_{\overline A(\rt)}$ and
$q_{A(t)}=p_t$. Since
$\overline A'(\rt)=g(1-\rt)^2$, the standard deterministic time change
provides a Brownian motion $\overline W$ such that
\[
    \dd B^{\leftarrow}_{\overline A(\rt)}
      =g(1-\rt)\,\dd\overline W_{\rt}.
\]
Set $\overline\ell_{\rt}:=
\lambda^{\leftarrow}_{\overline A(\rt)}$. The time-changed boundary term is
therefore
$\int_0^{\rt}\mathbf n_{\mathrm{in}}(Y_s)\,\dd\overline\ell_s$.
Since $\overline A$ is increasing, $\overline\ell$ is supported on the
boundary times of $Y$. Substituting
$T-\overline A(s)=A(1-s)$ in the source drift therefore gives
\begin{equation}
    Y_{\rt}=Y_0
    +\int_0^{\rt}g(1-s)^2\nabla\log p_{1-s}(Y_s)\,\dd s
    +\int_0^{\rt}g(1-s)\,\dd\overline W_s
    +\int_0^{\rt}\mathbf n_{\mathrm{in}}(Y_s)\,
       \dd\overline\ell_s,
    \qquad 0\leq\rt<1.
    \label{eq:half_open_full_ideal_reverse_in_proof}
\end{equation}
No score assumption is used here: for $\rt<1$, the remaining forward clock
$A(1-\rt)$ is positive, and the representative fixed above supplies the
required positive-time score.

\emph{Step 3: include $\rt=1$.}
Equation~\eqref{eq:half_open_full_ideal_reverse_in_proof} has so far been
proved only for $\rt<1$. To obtain the SDE on $[0,1]$, we check that every
term has a limit as
$\rt\uparrow1$. The possible obstruction is the score as the forward time
$1-\rt$ approaches $0$.

Assumption~\ref{assumption:closed_endpoint_score_continuity} gives
\[
    K:=\sup_{(t,x)\in[0,1]\times\overline{\cD}}
       \lVert\nabla\log p_t(x)\rVert_2<\infty.
\]
Hence the drift integral converges. The Brownian motion has its standard
continuous extension to $1$, and its stochastic integral does too because
$g(1-\cdot)\in L^2([0,1])$.
The reversed boundary clock supplied by the cited result is continuous on
$[0,T]$. Therefore $\overline\ell$ from Step~2 extends continuously to
$[0,1]$ with the same monotonicity and boundary-support properties.
Thus the scalar regulator also extends correctly. Since $Y$ is continuous,
passing to the limit in
\eqref{eq:half_open_full_ideal_reverse_in_proof} proves
\eqref{eq:full_interval_ideal_sde} on $[0,1]$. Since
$Y_{\rt}=X_{1-\rt}$, its law is $\Qb^\star$, with
$Y_0\sim p_1$ and $Y_1\sim p_0=\pi$.

\emph{Step 4: prove uniqueness.}
Set
\[
    b(\rt,x):=g(1-\rt)^2\nabla\log p_{1-\rt}(x),
    \qquad
    \sigma(\rt,x):=g(1-\rt)I_d.
\]
Assumption~\ref{assumption:closed_endpoint_score_continuity}, compactness
of $[0,1]\times\overline{\cD}$, and \eqref{eq:rve_noise_schedule} make $b$
bounded and jointly continuous. Assumption~
\ref{assumption:closed_endpoint_score} makes it uniformly Lipschitz in
space. The same schedule makes $\sigma$ bounded and continuous, with
spatial Lipschitz constant zero. Assumptions~
\ref{assumption:bounded_convex_domain} and \ref{assumption:theoretical_domain}
supply the domain hypotheses of
Lemma~\ref{lemma:forward_reflected_well_posedness}. Applying it with $T=1$
and initial law $p_1$ gives strong existence, pathwise uniqueness, and
uniqueness in law. The weak-solution
identification on $[0,1]$ in Step~3 then shows that $\Qb^\star$ is the
unique solution law of \eqref{eq:full_interval_ideal_sde}.
\end{proof}

\section{Girsanov change of drift for reflected diffusions}
\label{sec:reflected_girsanov_path_space}

This section proves the path-law change of measure used in
Theorem~\ref{theorem:kl_reference}.  We first apply ordinary Girsanov theory
on a filtered probability space carrying a reflected solution. We then show
that the resulting density depends only on the observed path, so it descends
to a Radon--Nikodym derivative between coordinate path laws.  The second step
is where reflection requires additional work.  Theorem~\ref{theorem:reflected_girsanov}
combines the two steps, after which we specialize it to the learned and
reference diffusions.

\subsection{Construction and the path density}

\begin{lemma}
    \label{lemma:construction_space_reflected_girsanov}
    Fix $T>0$.  Let
    $(\Omega,\cF,(\cF_t)_{0\leq t\leq T},\mathbf P)$ satisfy the usual
    conditions. Suppose it carries an $\cF_0$-measurable
    $\overline{\cD}$-valued random variable $\xi_0$, an $\R^d$-valued
    Brownian motion $W$ relative to $(\cF_t)_{0\leq t\leq T}$, a continuous
    adapted $\overline{\cD}$-valued process $X$ with $X_0=\xi_0$, and a
    continuous adapted nondecreasing process $\ell$ with $\ell_0=0$. Set
    $L_t:=\int_0^t\mathbf n_{\mathrm{in}}(X_s)\,\dd\ell_s$, and suppose
    \begin{equation}
        \dd X_t=b_t\,\dd t+g(t)\,\dd W_t
        +\dd L_t,
        \qquad 0\leq t\leq T,
        \label{eq:generic_reflected_construction}
    \end{equation}
    where $b$ is progressively measurable with
    $\int_0^T\lVert b_t\rVert_2\,\dd t<\infty$ almost surely,
    $g\in L^2([0,T])$ is deterministic, and $\dd\ell_s$ is supported on $\{s:X_s\in\partial\cD\}$.
    Let $\vartheta$ be a bounded predictable
    $\R^d$-valued process and set
    \begin{equation}
        \mathcal E_t^\vartheta:=\exp\left(
        \int_0^t\vartheta_s^\top\dd W_s
        -\frac12\int_0^t\lVert\vartheta_s\rVert_2^2\,\dd s
        \right).
        \label{eq:generic_construction_density}
    \end{equation}
    Then $\mathcal E^\vartheta$ is a uniformly integrable
    $\mathbf P$-martingale, and
    $\dd\widetilde{\mathbf Q}:=\mathcal E_T^\vartheta\,\dd\mathbf P$
    defines a probability
    equivalent to $\mathbf P$.  Under $\widetilde{\mathbf Q}$,
    \[
        \widetilde W_t:=W_t-\int_0^t\vartheta_s\,\dd s
    \]
    is Brownian relative to $(\cF_t)_{0\leq t\leq T}$, the initial law of
    $X$ is unchanged, and the same solution pair $(X,\ell)$ satisfies the
    stated regulator conditions and
    \begin{equation}
        \dd X_t=(b_t+g(t)\vartheta_t)\,\dd t
        +g(t)\,\dd\widetilde W_t
        +\dd L_t.
        \label{eq:generic_changed_reflected_construction}
    \end{equation}
\end{lemma}

\begin{proof}
By \citet[Section~5.1.3(iv)]{le2016brownian},
$\int_0^\cdot\vartheta_s^\top\dd W_s$ is a continuous local martingale.
Boundedness of $\vartheta$ gives
$\E_{\mathbf P}\exp(\frac12\int_0^T\lVert\vartheta_s\rVert_2^2\,\dd s)<\infty$.
The finite-horizon form of \citet[Theorem~5.23]{le2016brownian}, obtained by
stopping at $T$, therefore makes $\mathcal E^\vartheta$ a uniformly
integrable martingale. \citet[Theorem~5.22]{le2016brownian} makes
$\widetilde W$ a continuous local martingale under
$\widetilde{\mathbf Q}$. Its quadratic covariations are unchanged, so
\citet[Theorem~5.12]{le2016brownian} makes it Brownian. Substitution gives
\eqref{eq:generic_changed_reflected_construction}.  The probability change
does not alter the paths $(X,\ell)$ or their adaptedness, and equivalence
preserves the regulator conditions. Finally,
since $\xi_0=X_0$ is $\cF_0$-measurable by hypothesis,
$\E_{\mathbf P}[\mathcal E_T^\vartheta\mid\cF_0]
=\mathcal E_0^\vartheta=1$, so the law of the
initial variable $\xi_0$ is unchanged.
\end{proof}

Let $\omega\in\Omega$ denote a construction-space outcome
and let $\zeta\in C([0,T];\overline{\cD})$ denote a generic path. The observed
path associated with $\omega$ is $X(\omega)$.  Our goal is to identify a Borel
functional $z$ on path space such that, for every bounded Borel functional
$F$,
\begin{equation}
\begin{aligned}
    \E_{\widetilde{\Q}}[F]
    &=\int F(\zeta)\,\widetilde{\Q}(\dd\zeta)=\int z(\zeta)F(\zeta)\,\P(\dd\zeta)
      =\E_{\P}[zF].
\end{aligned}
    \label{eq:pushforward_density_target}
\end{equation}
This is the defining identity for
$z=\dd\widetilde{\Q}/\dd\P$.  Lemma~\ref{lemma:construction_space_reflected_girsanov} initially gives only
\[
\begin{aligned}
    \int F(\zeta)\,\widetilde{\Q}(\dd\zeta)
    &=\int_\Omega F(X(\omega))\,
      \widetilde{\mathbf Q}(\dd\omega)=\int_\Omega \mathcal E_T^\vartheta(\omega)
      F(X(\omega))\,\mathbf P(\dd\omega).
\end{aligned}
\]
The final expression is perfectly well defined, but it is still an
expectation on the construction space.  We cannot yet rewrite it as an
expectation under $\P$, because $\mathcal E_T^\vartheta$ need not be a
function of the path $X$.

Conditioning converts this construction-space weight into one that retains
only the information contained in $X$:
\[
\begin{aligned}
    \int_\Omega \mathcal E_T^\vartheta(\omega)
      F(X(\omega))\,\mathbf P(\dd\omega)&=\int_\Omega
      \left(
        \E_{\mathbf P}[\mathcal E_T^\vartheta F(X)\mid\sigma(X)]
      \right)(\omega)\,\mathbf P(\dd\omega)\\
      &=\int_\Omega
      \left(
        \E_{\mathbf P}[\mathcal E_T^\vartheta\mid\sigma(X)]
      \right)(\omega)
      F(X(\omega))\,\mathbf P(\dd\omega).
\end{aligned}
\]
The first equality is the tower property.
The second holds because $F(X)$
is $\sigma(X)$-measurable and can therefore be pulled out of the inner
conditional expectation.  The conditioning is introduced precisely because
pushforward through $X$ discards the remaining construction-space
information: the conditional expectation is the $\sigma(X)$-measurable
weight that preserves all averages against bounded functionals of $X$.

Since $X_\#\mathbf P=\P$ is a probability measure and
$\mathcal E_T^\vartheta\geq 0$, the conditional-expectation form of the
Doob--Dynkin lemma in \citet[Lemma~3]{taraldsen2018optimal}, applied with
$\Gamma=\mathcal E_T^\vartheta$ and $Y=X$, provides a Borel path functional
$z$ such that
\[
    \left(
      \E_{\mathbf P}[\mathcal E_T^\vartheta\mid\sigma(X)]
    \right)(\omega)
    =z(X(\omega))
    \quad\text{for $\mathbf P$-almost every }\omega\in\Omega.
\]
The definition of the pushforward law $\P=X_\#\mathbf P$ now gives
\[
\begin{aligned}
    \int_\Omega z(X(\omega))F(X(\omega))\,\mathbf P(\dd\omega)
    &=\int z(\zeta)F(\zeta)\,\P(\dd\zeta)=\E_{\P}[zF].
\end{aligned}
\]
Together with the preceding identities, this is
\eqref{eq:pushforward_density_target} and therefore proves
\[
    \frac{\dd\widetilde{\Q}}{\dd\P}=z
    \qquad\P\text{-almost surely}.
\]
There is no reason in general for $z(X(\omega))$ to equal
$\mathcal E_T^\vartheta(\omega)$.  If the latter is not determined by $X$,
two outcomes of the construction can have the same observed path but
different values of the proposed likelihood.  Writing
$\dd\widetilde{\Q}/\dd\P(X(\omega))
=\mathcal E_T^\vartheta(\omega)$ would then be ill-defined on coordinate
path space.  

In our case $(X,\ell)$ is supplied as the
solution pair. After passage to coordinate path space only $X$ remains, and
the definition of a solution does not by itself show that $\ell$ is a
function of $X$. Establishing this is the purpose of the following recovery
argument. For
$L_t:=\int_0^t\mathbf n_{\mathrm{in}}(X_s)\,\dd\ell_s$, the reflected SDE
gives
\[
    W_t=\int_0^t g(s)^{-1}
    \bigl(\dd X_s-b(X_s,s)\,\dd s-\dd L_s\bigr).
\]
This does not reconstruct $W$ from $X$ until $\ell$, and hence $L$, has been
reconstructed from $X$. For intuition on how to establish that $\ell$ is a
function of the observed path $X$, see the one-dimensional half-line
argument in \citet[Theorem~1.3.1]{pilipenko2014introduction} and the preceding discussion in \citet[Section~1.3]{pilipenko2014introduction}.

The higher-dimensional argument uses the same idea after reducing the
boundary behavior to one dimension. For the sake of completeness, we provide full construction and proof of the argument. Similar arguments are present in \citet[Eq.~3]{holk2026statistical}, but the exact form differs from the one obtained in our Lemma~\ref{lemma:normal_regulator_recovery} due to differences in problem formulation. 

We choose a smooth function $\varphi$
that vanishes precisely on $\partial\cD$ and whose gradient there is the
inward unit normal.  Tanaka's formula applied to the nonnegative scalar
process $\varphi(X)$ recovers the scalar boundary regulator from the local
time of $\varphi(X)$ at zero. For
$L_t:=\int_0^t\mathbf n_{\mathrm{in}}(X_s)\,\dd\ell_s$, this also recovers
$L$. Substituting it into the reflected SDE then
recovers $W$.  The next three lemmas make the boundary coordinate, regulator
recovery, and Brownian-coordinate recovery precise.

\begin{lemma}
    \label{lemma:normalized_boundary_coordinate}
    Suppose $\cD$ is bounded and has $C^\infty$ boundary.  There is
    $\varphi\in C^\infty(\R^d)$ such that
    \begin{equation}
        \varphi\geq0\ \text{on }\overline{\cD},
        \qquad
        \{x\in\overline{\cD}:\varphi(x)=0\}=\partial\cD,
        \qquad
        \nabla\varphi=\mathbf n_{\mathrm{in}}
        \ \text{on }\partial\cD.
        \label{eq:normalized_boundary_coordinate}
    \end{equation}
\end{lemma}

\begin{proof}
Define the signed distance from the boundary, with the sign chosen positive
inside $\cD$, by
\[
    d_{\mathrm s}(x):=
    \begin{cases}
        \operatorname{dist}(x,\partial\cD),
          &x\in\overline{\cD},\\
        -\operatorname{dist}(x,\partial\cD),
          &x\notin\overline{\cD}.
    \end{cases}
\]
Theorem~1 and the preliminary lemma of
\citet{foote1984regularity} give a neighborhood of the compact
smooth hypersurface $\partial\cD$ with a unique nearest-point projection, on
which the unsigned distance is smooth away from $\partial\cD$.  In every
local defining chart, \citet[Remark~2]{foote1984regularity} shows that assigning a
sign according to the two sides extends the distance smoothly across the
hypersurface.  With the sign positive in $\cD$, every local construction is
the restriction of the same globally defined function $d_{\mathrm s}$, so no gluing is
needed: the local neighborhoods can be united to give an open neighborhood
$N$ of $\partial\cD$ on which $d_{\mathrm s}$ is smooth.  Moreover, for
$z\in\partial\cD$ and all sufficiently small $s>0$, the point
$z+s\mathbf n_{\mathrm{in}}(z)$ lies in $\cD$ and has $z$ as its nearest
boundary point.
Consequently
\[
    d_{\mathrm s}\bigl(z+s\mathbf n_{\mathrm{in}}(z)\bigr)=s.
\]
Since $d_{\mathrm s}$ vanishes on $\partial\cD$, its derivative there annihilates every
tangent vector. The last identity gives its derivative in the inward-normal
direction.  Hence
\begin{equation}
    \nabla d_{\mathrm s}(z)=\mathbf n_{\mathrm{in}}(z),
    \qquad z\in\partial\cD.
    \label{eq:signed_distance_inward_gradient}
\end{equation}

It remains only to extend this local coordinate smoothly away from the
boundary.  Compactness of $\partial\cD$ allows us to choose
$\varepsilon>0$ such that
$\{x:\operatorname{dist}(x,\partial\cD)<3\varepsilon\}\subset N$, and choose
a nondecreasing function $\rho\in C^\infty(\R)$ satisfying
\[
    \rho(s)=s\quad (|s|\leq\varepsilon),
    \qquad
    \rho(s)=2\varepsilon\quad (s\geq2\varepsilon),
    \qquad
    \rho(s)=-2\varepsilon\quad (s\leq-2\varepsilon).
\]
This is the static version of the modified signed-distance construction in
\citet[Section~5.1, equations~(5.2)--(5.3)]{alfaro2008singular}.  Writing
$N_{3\varepsilon}:=
\{x:\operatorname{dist}(x,\partial\cD)<3\varepsilon\}$, define
\[
    \varphi(x):=
    \begin{cases}
        \rho(d_{\mathrm s}(x)), &x\in N_{3\varepsilon},\\
        2\varepsilon, &x\in\cD\setminus N_{3\varepsilon},\\
        -2\varepsilon,
          &x\in\R^d\setminus(\overline{\cD}\cup N_{3\varepsilon}).
    \end{cases}
\]
Near the boundary of $N_{3\varepsilon}$, the first expression is already the
corresponding constant, so these definitions give a function in
$C^\infty(\R^d)$.  Monotonicity of $\rho$ and $\rho(s)=s$ near zero show that
$\varphi>0$ in $\cD$ and $\varphi=0$ on $\partial\cD$.  Finally,
\eqref{eq:signed_distance_inward_gradient} and $\rho'(0)=1$ give
$\nabla\varphi=\mathbf n_{\mathrm{in}}$ on $\partial\cD$.
\end{proof}

\begin{lemma}
    \label{lemma:normal_regulator_recovery}
    Fix $T>0$, let $\varphi$ be as in
    Lemma~\ref{lemma:normalized_boundary_coordinate}, and let
    $g\in C([0,T])$ satisfy
    $\inf_{0\leq t\leq T}\lvert g(t)\rvert>0$.
    Let $(\Omega,\cF,(\cF_t)_{0\leq t\leq T},\mathbf P)$ satisfy the usual
    conditions. Suppose it carries an $\R^d$-valued Brownian motion $W$
    relative to $(\cF_t)_{0\leq t\leq T}$, a continuous adapted
    $\overline{\cD}$-valued process $X$ with $X_0$ being
    $\cF_0$-measurable, a bounded progressively measurable process $b$, and a
    continuous adapted nondecreasing process $\ell$ with $\ell_0=0$. Suppose
    \begin{equation}
        X_t=X_0+\int_0^t b_s\,\dd s
        +\int_0^t g(s)\,\dd W_s+L_t,
        \qquad
        L_t=\int_0^t\mathbf n_{\mathrm{in}}(X_s)\,\dd\ell_s,
        \label{eq:recovery_reflected_sde}
    \end{equation}
    where $\dd\ell$ is supported on
    $\{t:X_t\in\partial\cD\}$.
    Then, almost surely, $X$ spends zero Lebesgue time on the boundary and
    \begin{equation}
        \ell_t=\frac12\LT_t^0(\varphi(X))
        =\lim_{n\to\infty}\frac n2\int_0^t
        \ind_{\{0\leq\varphi(X_s)\leq1/n\}}
        g(s)^2\lVert\nabla\varphi(X_s)\rVert_2^2\,\dd s
        \quad(0\leq t\leq T).
        \label{eq:regulator_local_time_recovery}
    \end{equation}
    Consequently, there are fixed Borel maps
    \[
        \mathfrak l:C([0,T];\overline{\cD})\to C([0,T];\R),
        \qquad
        \mathfrak L:C([0,T];\overline{\cD})
        \to C([0,T];\R^d)
    \]
    such that $\ell=\mathfrak l(X)$ and $L=\mathfrak L(X)$ almost surely.
\end{lemma}

\begin{proof}
Set $Y_t^\partial:=\varphi(X_t)\geq0$. This is the boundary-distance
coordinate. The multidimensional It\^o formula in
\citet[Theorem~5.10]{le2016brownian} gives
\begin{equation}
    \dd Y_t^\partial
    =\langle\nabla\varphi(X_t),b_t\rangle\,\dd t
    +g(t)\langle\nabla\varphi(X_t),\dd W_t\rangle
    +\frac12g(t)^2\Delta\varphi(X_t)\,\dd t
    +\dd\ell_t.
    \label{eq:boundary_coordinate_ito}
\end{equation}
Indeed, $\dd\ell$ is supported on the boundary, where
$\nabla\varphi=\mathbf n_{\mathrm{in}}$.  The quadratic variation is
\[
    \langle Y^\partial\rangle_t
    =\int_0^t g(s)^2
      \lVert\nabla\varphi(X_s)\rVert_2^2\,\dd s.
\]
The occupation-density formula in
\citet[Corollary~9.7]{le2016brownian} assigns zero quadratic
variation to the single level $Y^\partial=0$.  Since
$\lVert\nabla\varphi\rVert_2=1$ on the boundary and
$\inf_t\lvert g(t)\rvert>0$,
\begin{equation}
    \int_0^T\ind_{\{X_s\in\partial\cD\}}\,\dd s=0.
    \label{eq:zero_boundary_lebesgue_time}
\end{equation}

Apply the negative-part Tanaka identity in
\citet[Proposition~9.2]{le2016brownian} to the nonnegative process
$Y^\partial$.
It gives
\[
    \frac12\LT_t^0(Y^\partial)
    =\int_0^t\ind_{\{Y_s^\partial=0\}}\,\dd Y_s^\partial.
\]
Inserting \eqref{eq:boundary_coordinate_ito}, the restricted martingale has
zero quadratic variation, both ordinary-time terms vanish by
\eqref{eq:zero_boundary_lebesgue_time} because their integrands are bounded,
and the regulator term equals $\ell_t$.  This proves the first equality in
\eqref{eq:regulator_local_time_recovery}. \citet[Proposition~9.9]{le2016brownian} gives the
second.

\emph{Measurability.}
For a path $\zeta\in C([0,T];\overline{\cD})$, let
\[
    \mathfrak l_n(\zeta)_t
    :=\frac n2\int_0^t
      \ind_{\{0\leq\varphi(\zeta(s))\leq1/n\}}
      g(s)^2\lVert\nabla\varphi(\zeta(s))\rVert_2^2\,\dd s.
\]
Each $\mathfrak l_n$ is a Borel map into $C([0,T];\R)$, and every resulting
time path is nondecreasing.
Define $\mathfrak l(\zeta)$ as their uniform limit when they are uniformly
Cauchy, and as the zero path otherwise.  This is a Borel map.  On the
probability-one event in \eqref{eq:regulator_local_time_recovery}, the
approximants converge pointwise to the continuous nondecreasing path $\ell$. Such convergence is uniform on a compact interval.  Hence
$\mathfrak l(X)=\ell$.  Finally define
\[
    \mathfrak L(\zeta)_t
    :=\int_0^t\nabla\varphi(\zeta(s))\,\dd\mathfrak l(\zeta)_s.
\]
Deterministic Riemann--Stieltjes sums show that $\mathfrak L$ is Borel, and
the boundary support of $\dd\ell$ gives $\mathfrak L(X)=L$.
\end{proof}

\begin{lemma}
    \label{lemma:regulator_brownian_coordinate_recovery}
    In Lemma~\ref{lemma:normal_regulator_recovery}, suppose in addition that
    $b_t=b(X_t,t)$ for some
    $b\in C(\overline{\cD}\times[0,T];\R^d)$.  Then the paths $L$ and $W$
    are almost surely Borel functions of $X$.  Moreover, for every
    $v\in C(\overline{\cD}\times[0,T];\R^d)$, the terminal It\^o integral
    \begin{equation}
        \int_0^T v(X_t,t)^\top\dd W_t
        \label{eq:coordinate_recovered_stochastic_integral}
    \end{equation}
    has a Borel version on $C([0,T];\overline{\cD})$.
\end{lemma}

\begin{proof}
The preceding lemma already gives $L=\mathfrak L(X)$.  Put
\[
    B_t(X):=\int_0^t b(X_s,s)\,\dd s,
    \qquad
    S_t(X):=X_t-L_t-B_t(X).
\]
Then $S_t=X_0+\int_0^t g(s)\,\dd W_s$.  For the schedule in
\eqref{eq:rve_noise_schedule}, $h:=g^{-1}$ is $C^1$, and
integration by parts gives the explicit pathwise formula
\[
    W_t=h(t)S_t-h(0)S_0-\int_0^t S_s\,\dd h(s).
\]
For a general continuous nonvanishing $g$, stochastic-integral associativity
from \citet[Theorem~5.6]{le2016brownian} gives
\begin{equation}
    W_t=\int_0^t g(s)^{-1}\,\dd S_s.
    \label{eq:generic_brownian_recovery}
\end{equation}
\citet[Proposition~5.9]{le2016brownian} approximates this integral
by deterministic left-point sums.  A single diagonal almost-sure subsequence
over rational times and continuity give one Borel version of the full path
$W$: define the continuous extension on the Borel set where it exists, and
use the zero path elsewhere.  The same argument applied to $v(X_t,t)$ gives a Borel version of
\eqref{eq:coordinate_recovered_stochastic_integral}.
The selected subsequence, and hence the Borel version, may depend on the fixed
weak law and on $v$. No version simultaneous over all laws and integrands is
asserted.
\end{proof}

We can now combine the standard construction-space change with the recovery
lemmas.  Existence and uniqueness in law are assumed here rather than derived:
their only role is to make the two SDEs denote unambiguous path laws and to
identify the weak solution produced by Girsanov. An analogous result for a slightly different setting is available at \citet[Appendix~A]{holk2026statistical}.

\begin{theorem}
\label{theorem:reflected_girsanov}
Let $\cD\subset\R^d$ be bounded and convex with $C^\infty$ boundary.  Let
$\alpha,\beta\in C(\overline{\cD}\times[0,1];\R^d)$, let
$g\in C([0,1])$ satisfy $\inf_t\lvert g(t)\rvert>0$. In the two equations
below, $L^{\P}$ and $L^{\Q}$ denote the vector reflection terms derived from
the scalar regulators in their respective weak realizations:
\begin{equation}
\label{eq:reflected_girsanov_sdes}
    \dd X_t=\alpha(X_t,t)\,\dd t+g(t)\,\dd W_t^{\P}
      +\dd L_t^{\P},
    \qquad
    \dd X_t=\beta(X_t,t)\,\dd t+g(t)\,\dd W_t^{\Q}
      +\dd L_t^{\Q}
\end{equation}
Suppose these equations have a common initial law $\mu$, admit weak solutions,
and have unique path laws $\P$ and $\Q$ on
$C([0,1];\overline{\cD})$.
Then
$\P\ll\Q\ll\P$. Let
$(\Omega,\cF,(\cF_t),\mathbf Q;\xi_0,W^{\Q},X,\ell)$ be any
weak realization
of the second SDE in \eqref{eq:reflected_girsanov_sdes}. Then,
$\mathbf Q$-almost surely,
\begin{equation}
\label{eq:reflected_girsanov_density}
    \frac{\dd\Q}{\dd\P}(X)
    =\exp\left(
      -\int_0^1u(X_t,t)^\top\dd W_t^{\Q}
      +\frac12\int_0^1\lVert u(X_t,t)\rVert_2^2\,\dd t
    \right),
    \qquad
    u:=\frac{\alpha-\beta}{g}.
\end{equation}
Equivalently, on any weak realization of the first SDE with construction
probability $\mathbf P$, $\mathbf P$-almost surely,
\begin{equation}
    \frac{\dd\Q}{\dd\P}(X)
    =\exp\left(
      -\int_0^1u(X_t,t)^\top\dd W_t^{\P}
      -\frac12\int_0^1\lVert u(X_t,t)\rVert_2^2\,\dd t
    \right).
    \label{eq:reflected_girsanov_density_under_p}
\end{equation}
\end{theorem}

\begin{proof}
Start from an arbitrary weak realization
$(\Omega,\cF,(\cF_t),\mathbf P;\xi_0,W^{\P},X,\ell)$ of the first
SDE. The function $u=(\alpha-\beta)/g$ is bounded and continuous on the
compact cylinder. By the definition of a weak realization, $X$ and
$\ell$ are continuous and adapted, $\xi_0=X_0$ is
$\cF_0$-measurable, and $W^{\P}$ is Brownian relative to the construction
filtration. Thus the hypotheses of
Lemma~\ref{lemma:construction_space_reflected_girsanov} hold with
$\vartheta_t=-u(X_t,t)$.  Under the resulting equivalent probability,
\[
    W_t^{\Q}=W_t^{\P}+\int_0^t u(X_s,s)\,\dd s
\]
is Brownian, and
$(\Omega,\cF,(\cF_t),\widetilde{\mathbf Q};
  \xi_0,W^{\Q},X,\ell)$
is a weak realization of the second SDE. Its initial law is still $\mu$, and
uniqueness in law identifies its coordinate law with $\Q$.

Lemmas~\ref{lemma:normal_regulator_recovery} and
\ref{lemma:regulator_brownian_coordinate_recovery} show that $\ell$, hence
$L^{\P}$, as well as $W^{\P}$ and the stochastic integral in
\eqref{eq:reflected_girsanov_density_under_p} have Borel versions determined
by $X$.  Thus the construction-space density equals $z(X)$ for a Borel path
functional $z$.  Hence, for every bounded Borel path functional $F$,
\[
    \E_{\Q}[F]
    =\E_{\P}[zF].
\]
This is exactly the statement that $\dd\Q/\dd\P=z$ on coordinate path space.
Strict positivity of $z$ gives equivalence.

To obtain the formula on an arbitrary realization of the second SDE, its
weak-solution properties provide the same adaptedness, scalar-regulator,
Brownian-motion, and $\cF_0$-measurability hypotheses. Apply the same argument in the
opposite direction with
$\vartheta_t=u(X_t,t)$.  The changed drift is
$\beta+gu=\alpha$, so uniqueness in law identifies the changed coordinate
law with $\P$.  Recovery and pushforward therefore give
\[
    \frac{\dd\P}{\dd\Q}(X)
    =\exp\left(
      \int_0^1u(X_t,t)^\top\dd W_t^{\Q}
      -\frac12\int_0^1\lVert u(X_t,t)\rVert_2^2\,\dd t
    \right).
\]
Taking the reciprocal proves \eqref{eq:reflected_girsanov_density} on every
such realization.
\end{proof}

\section{Formal statements and proofs for other main-text results}

\subsection{Optimal path proposal and terminal tilt}

\begin{proposition}
    \label{proposition:optimal_proposal}
    Suppose Assumptions~\ref{assumption:bounded_convex_domain},
    \ref{assumption:theoretical_domain}, and
    \ref{assumption:positive_target_reduction} hold. Then
    \begin{equation}
        \frac{\dd\Qf^\star}{\dd\Pf}(X)
        =\frac{\pi(X_0)}{\nu(X_0)}.
        \label{eq:optimal_forward_density}
    \end{equation}
    Reversing both path laws gives
    \[
        \frac{\dd\Qb^\star}{\dd\Pb}(Y)
        =\frac{\pi(Y_1)}{\nu(Y_1)}.
    \]
    For every proposal law $\Qf$ equivalent to $\Pf$, the importance weight
    in \eqref{eq:path_importance_weight} satisfies
    \[
        W_{\Qf}(X)
        =Z\frac{\dd\Qf^\star}{\dd\Qf}(X),
        \qquad
        \Var_{\Qf}(W_{\Qf})
        =Z^2\chi^2(\Qf^\star\Vert\Qf).
    \]
    Consequently, $\Qf^\star$ is the unique variance-minimizing proposal and $W_{\Qf^\star}=Z$ almost surely.
\end{proposition}

\begin{proof}
The laws $\Pf$ and $\Qf^\star$ have the same reflected dynamics and differ
only in their initial distributions, which are $\nu$ and $\pi$, respectively.
Conditional on $X_0=x$, their trajectory laws therefore coincide. Formally,
the deterministic-start solution map supplies a common measurable version of
these conditional path laws. To turn this observation into a path-law
identity, let $F:\cC\to\R$ be bounded and Borel measurable.
Conditioning on $X_0$ gives
\[
\begin{aligned}
    \EEs{\Qf^\star}{F(X)}
    &=\int_{\cD}
      \EEs{\Qf^\star}{F(X)\mid X_0=x}\,\pi(x)\,\dd x\\
    &=\int_{\cD}
      \EEs{\Pf}{F(X)\mid X_0=x}
      \frac{\pi(x)}{\nu(x)}\,\nu(x)\,\dd x\\
    &=\EEs{\Pf}{
      F(X)\frac{\pi(X_0)}{\nu(X_0)}}.
\end{aligned}
\]
Since this holds for every bounded Borel $F$, it proves
\eqref{eq:optimal_forward_density}. Assumption~
\ref{assumption:positive_target_reduction} directly makes $\pi/\nu$ strictly
positive and bounded above and away from zero on $\overline{\cD}$, so $\Pf$
and $\Qf^\star$ are equivalent.

Under path reversal, $Y_{\rt}=X_{1-\rt}$ and therefore $X_0=Y_1$. Applying the
preceding expectation identity to the reversed-path functional gives, for
every bounded Borel $F$,
\[
    \EEs{\Qb^\star}{F(Y)}
    =\EEs{\Pb}{
      F(Y)\frac{\pi(Y_1)}{\nu(Y_1)}}.
\]
Thus the terminal factor
changes the entire reverse path law from $\Pb$ to $\Qb^\star$, whose initial
and terminal marginals are $p_1$ and $\pi$.

Now let $\Qf$ be any path law equivalent to $\Pf$. Using
$r=Z\pi$, \eqref{eq:path_importance_weight},
\eqref{eq:optimal_forward_density}, and the Radon--Nikodym chain rule gives
\[
\begin{aligned}
    W_{\Qf}(X)
    &=\frac{r(X_0)}{\nu(X_0)}
      \frac{\dd\Pf}{\dd\Qf}(X)=Z\frac{\dd\Qf^\star}{\dd\Pf}(X)
      \frac{\dd\Pf}{\dd\Qf}(X)
      =Z\frac{\dd\Qf^\star}{\dd\Qf}(X).
\end{aligned}
\]
In particular, $\mathbb E_{\Qf}[W_{\Qf}]=Z$. Allowing infinite variance,
\[
\begin{aligned}
    \Var_{\Qf}(W_{\Qf})
    &=\mathbb E_{\Qf}\left[(W_{\Qf}-Z)^2\right]=Z^2\mathbb E_{\Qf}\left[
      \left(\frac{\dd\Qf^\star}{\dd\Qf}-1\right)^2\right]
      =Z^2\chi^2(\Qf^\star\Vert\Qf).
\end{aligned}
\]
The right-hand side is
nonnegative and is zero exactly when
$\dd\Qf^\star/\dd\Qf=1$ almost surely, that is, exactly when
$\Qf=\Qf^\star$. For this proposal,
$W_{\Qf^\star}=Z$ almost surely.
\end{proof}

\subsection{Continuous training KL}
\label{theorem:kl_reference_proof}

\begin{theorem}
\label{theorem:kl_reference}
Suppose Assumptions~\ref{assumption:bounded_convex_domain},
\ref{assumption:theoretical_domain},
\ref{assumption:positive_target_reduction}, \ref{assumption:control_continuity}, and
\ref{assumption:control_spatial_lipschitz} hold. Then the reflected SDEs
\eqref{eq:reverse_reference_sde} and
\eqref{eq:learnable_reverse_rve_sde} have unique strong solutions and are
unique in law. Their unique path laws are $\Pb$ and $\Qb_\theta$,
respectively, and the training identity \eqref{eq:training_loss} holds.
\end{theorem}

\begin{proof}
Under Assumptions~\ref{assumption:bounded_convex_domain} and
\ref{assumption:theoretical_domain},
Lemma~\ref{lemma:rve_stationary_reversal} gives a pathwise-unique strong
solution of \eqref{eq:reverse_reference_sde} and identifies its unique
solution law as $\Pb$.

For the learned equation, take
\[
    b(\rt,y):=g(1-\rt)^2s_\theta(y,1-\rt),
    \qquad
    \sigma(\rt,y):=g(1-\rt)I_d.
\]
Assumption~\ref{assumption:control_continuity}, compactness of
$[0,1]\times\overline{\cD}$, and \eqref{eq:rve_noise_schedule} make $b$
bounded and jointly continuous. Assumption~
\ref{assumption:control_spatial_lipschitz} makes it uniformly Lipschitz in
space. The same schedule makes $\sigma$ bounded and continuous, with
spatial Lipschitz constant zero. Together with the domain hypotheses in
Assumptions~\ref{assumption:bounded_convex_domain} and
\ref{assumption:theoretical_domain}, these allow
Lemma~\ref{lemma:forward_reflected_well_posedness} to give a pathwise-unique
strong solution and uniqueness in law for $T=1$ and initial law $\nu$.
This defines
$\Qb_\theta$ and,
together with Lemma~\ref{lemma:rve_stationary_reversal}, verifies the
existence-and-uniqueness-in-law hypotheses of
Theorem~\ref{theorem:reflected_girsanov}.

The continuity and positive lower bound of the schedule in
\eqref{eq:rve_noise_schedule}, together with the coefficient continuity just
checked, verify the remaining hypotheses of
Theorem~\ref{theorem:reflected_girsanov}. Both $\Pb$ and $\Qb_\theta$ start
from $\nu$, so the theorem applies to this pair. Specifically, take reference
drift $\alpha=0$, learned drift
$\beta(y,\rt)=g(1-\rt)^2s_\theta(y,1-\rt)$, and diffusion coefficient
$g(1-\rt)$. Then $u=-g(1-\rt)s_\theta$, so under $\Qb_\theta$,
\begin{equation}
    \log\frac{\dd\Qb_\theta}{\dd\Pb}(Y)
    =\int_0^1g(1-\rt)s_\theta(Y_{\rt},1-\rt)^\top
      \dd\overline W_{\rt}
    +\frac12\int_0^1g(1-\rt)^2
      \lVert s_\theta(Y_{\rt},1-\rt)\rVert_2^2\,\dd\rt.
    \label{eq:learned_to_reference_girsanov_density}
\end{equation}
The theorem is not applied directly
to $\Qb_\theta$ and $\Qb^\star$, whose initial distributions are $\nu$ and
$p_1$.
Here $\overline W$ is Brownian under the construction probability inducing
$\Qb_\theta$. Its integrand is predictable because $Y$ is adapted and
continuous, and Assumption~\ref{assumption:control_continuity}, compactness,
and boundedness of $g$ make it bounded. The stochastic integral is therefore
a square-integrable martingale and has mean zero.
Proposition~\ref{proposition:optimal_proposal} gives the terminal tilt. Using $\Pb$ as the intermediate law, the
Radon--Nikodym chain rule gives
\[
    \frac{\dd\Qb_\theta}{\dd\Qb^\star}(Y)
    =\frac{\dd\Qb_\theta}{\dd\Pb}(Y)
      \frac{\nu(Y_1)}{\pi(Y_1)}.
\]
Taking logarithms and expectation under $\Qb_\theta$ yields
\[
    \KL{\Qb_\theta}{\Qb^\star}
    =\EEs{\Qb_\theta}{
      \log\frac{\dd\Qb_\theta}{\dd\Pb}(Y)
      +\log\frac{\nu(Y_1)}{\pi(Y_1)}}.
\]
The strict positivity in
Assumption~\ref{assumption:positive_target_reduction} gives
$\log(\nu/\pi)=\log(\nu/r)+\log Z$, while its two-sided bounds make this
terminal logarithm bounded. This proves \eqref{eq:training_loss}.
\end{proof}

\subsection{Normalizing-constant accuracy for the continuous reflected sampler}
\label{sec:continuous_normalizing_constant_accuracy}

Theorem~\ref{theorem:continuous_normalizing_constant_accuracy} gives a
total-variation accuracy bound for the normalizing-constant estimator
\eqref{eq:importance_sampling_estimator}. It is a simplified reflected
version of \citet[Theorem~5]{guo2026complexity}, formulated for continuous
ideal and learned path laws on $[0,1]$.

In that argument, the formal derivation of the mixed path-density formula
introduces stationary Wiener measure as an intermediate reference.
This is the Brownian path measure with Lebesgue marginals.
The use of a Wiener reference also appears in the stochastic-control
formulations of \citet[Section~III]{chen2016relation} and
\citet[Section~4.4]{chen2021stochastic}.
A general forward--backward factorization through a common reference is
given by \citet[Proposition~2.2 and Appendix~E.1]{vargas2024transport}.

We instead use $\Pb$, the path law of a deterministic time change of
stationary normally reflected Brownian motion initialized from the uniform
distribution $\nu$. This is a probability measure whose marginals remain
uniform throughout $[0,1]$, as shown in
Lemma~\ref{lemma:rve_stationary_reversal}.
Proposition~\ref{proposition:optimal_proposal} and
Theorem~\ref{theorem:reflected_girsanov} provide the two density factors
relative to $\Pb$ needed for the comparison.

\begin{theorem}
    \label{theorem:continuous_normalizing_constant_accuracy}
    Suppose Assumptions~\ref{assumption:bounded_convex_domain},
    \ref{assumption:theoretical_domain},
    \ref{assumption:positive_target_reduction},
    \ref{assumption:control_continuity}, and
    \ref{assumption:control_spatial_lipschitz} hold.
    Fix $\theta$ and an integer $N\geq1$, and let
    $Y^{(1)},\ldots,Y^{(N)}$ be paths, each with law $\Qb_\theta$.
    Let $\widehat Z$ be the estimator
    \eqref{eq:importance_sampling_estimator}. Expectations and probabilities
    below are over these sampled paths. Then $\mathbb E[\widehat Z]=Z$ and,
    for every $\varepsilon>0$,
    \begin{equation}
        \Pr\!\left(\left|\frac{\widehat Z}{Z}-1\right|
            \geq\varepsilon\right)
        \leq\frac{2\|\Qb^\star-\Qb_\theta\|_{\mathrm{TV}}}{\varepsilon}.
        \label{eq:continuous_normalizing_constant_tail}
    \end{equation}

    In particular, for every $\varepsilon>0$ and $\delta\in(0,1)$,
    \begin{equation}
        \|\Qb^\star-\Qb_\theta\|_{\mathrm{TV}}
        \leq\frac{\varepsilon\delta}{2}
        \quad\Longrightarrow\quad
        \Pr\!\left(\left|\frac{\widehat Z}{Z}-1\right|
            \geq\varepsilon\right)\leq\delta.
        \label{eq:continuous_normalizing_constant_confidence}
    \end{equation}
\end{theorem}

\begin{proof}
Assumption~\ref{assumption:positive_target_reduction} and boundedness of
$\cD$ give $0<Z<\infty$ and make $\pi/\nu$ bounded above and away from zero.
Together with Assumptions~\ref{assumption:bounded_convex_domain} and
\ref{assumption:theoretical_domain}, these are the hypotheses of
Proposition~\ref{proposition:optimal_proposal}, which gives
$\Qb^\star\ll\Pb\ll\Qb^\star$ and
\[
    \frac{\dd\Qb^\star}{\dd\Pb}(Y)
    =\frac{\pi(Y_1)}{\nu(Y_1)}.
\]
Under the same domain and target assumptions,
Assumptions~\ref{assumption:control_continuity} and
\ref{assumption:control_spatial_lipschitz} allow
Theorem~\ref{theorem:kl_reference} to give the unique laws $\Pb$ and
$\Qb_\theta$. The Girsanov comparison in its proof establishes
$\Qb_\theta\ll\Pb\ll\Qb_\theta$ and the density formula
\eqref{eq:learned_to_reference_girsanov_density}. Define
\begin{equation}
    F(Y):=\frac{\dd\Qb^\star}{\dd\Qb_\theta}(Y)
    =\frac{r(Y_1)}{Z\nu(Y_1)}
       \frac{\dd\Pb}{\dd\Qb_\theta}(Y),
    \qquad \Qb_\theta\text{-almost surely}.
    \label{eq:continuous_normalized_weight}
\end{equation}
Since $\Qb^\star$ is a
probability measure,
\[
    \mathbb E_{\Qb_\theta}[F(Y)]
    =\int \frac{\dd\Qb^\star}{\dd\Qb_\theta}\,\dd\Qb_\theta
    =\int \dd\Qb^\star=1.
\]
Linearity of expectation therefore gives
\[
    \mathbb E\!\left[\frac{\widehat Z}{Z}\right]
    =\frac1N\sum_{i=1}^N\mathbb E_{\Qb_\theta}[F(Y)]
    =1,
\]
which proves unbiasedness. Then, we have the following bound:
\begin{equation}
    \mathbb E_{\Qb_\theta}\bigg[\left|\frac{\widehat Z}{Z}-1\right|\bigg]
    =\mathbb E_{\Qb_\theta}\bigg[\left|\frac1N\sum_{i=1}^N(F(Y^{(i)})-1)\right|\bigg]
    \leq\frac1N\sum_{i=1}^N\mathbb E_{\Qb_\theta}[|F(Y^{(i)})-1|]=\mathbb E_{\Qb_\theta}[|F(Y^{(1)})-1|].
    \label{eq:continuous_average_absolute_error}
\end{equation}

Under the total-variation
convention in Section~\ref{sec:notation},
\[
    \mathbb E_{\Qb_\theta}[|F(Y)-1|]
    =\int\left|\frac{\dd\Qb^\star}{\dd\Qb_\theta}-1\right|
        \,\dd\Qb_\theta
    =2\|\Qb^\star-\Qb_\theta\|_{\mathrm{TV}}.
\]
Substituting this identity into
\eqref{eq:continuous_average_absolute_error} yields
\begin{equation}
    \mathbb E
        \!\left[\left|\frac{\widehat Z}{Z}-1\right|\right]
    \leq2\|\Qb^\star-\Qb_\theta\|_{\mathrm{TV}}.
    \label{eq:continuous_relative_mean_absolute_error}
\end{equation}
Finally, Markov's inequality gives
\[
    \Pr
        \!\left(\left|\frac{\widehat Z}{Z}-1\right|\geq\varepsilon\right)
    \leq\frac{\mathbb E
        [|\widehat Z/Z-1|]}{\varepsilon}
    \leq\frac{2\|\Qb^\star-\Qb_\theta\|_{\mathrm{TV}}}{\varepsilon}.
\]
This proves \eqref{eq:continuous_normalizing_constant_tail}. If
$\|\Qb^\star-\Qb_\theta\|_{\mathrm{TV}}\leq\varepsilon\delta/2$, then
$2\|\Qb^\star-\Qb_\theta\|_{\mathrm{TV}}/\varepsilon\leq\delta$, proving
\eqref{eq:continuous_normalizing_constant_confidence}.
\end{proof}

The following theorem expresses the ideal-to-learned KL in terms of the
initial-law mismatch and score error.

\begin{theorem}
    \label{theorem:continuous_ideal_to_learned_kl}
    Let $Z$ be the normalizing constant defined in \eqref{eq:problem}.
    Suppose $Z>0$ and Assumptions~\ref{assumption:bounded_convex_domain},
    \ref{assumption:theoretical_domain},
    \ref{assumption:control_continuity},
    \ref{assumption:control_spatial_lipschitz},
    \ref{assumption:closed_endpoint_score_continuity}, and
    \ref{assumption:closed_endpoint_score} hold. 
    Fix $\theta$. Then
    \begin{equation}
        \KL{\Qb^\star}{\Qb_\theta}
        =\KL{p_1}{\nu}
        +\frac12\int_0^1g(t)^2
        \mathbb E_{X\sim p_t}\!\left[
            \left\|s_\theta(X,t)-\nabla\log p_t(X)\right\|_2^2
        \right]\,\dd t.
        \label{eq:continuous_ideal_to_learned_kl}
    \end{equation}
\end{theorem}

\begin{proof}
\emph{Step 1: separate the initial-law change from the drift change.}
The standing integrability of $r$ in \eqref{eq:problem}, together with $Z>0$,
makes $\pi=r/Z$ a probability density. Assumptions~
\ref{assumption:bounded_convex_domain} and \ref{assumption:theoretical_domain}
together with Assumptions~\ref{assumption:closed_endpoint_score_continuity}
and \ref{assumption:closed_endpoint_score} are precisely the hypotheses of
the full-interval uniqueness conclusion in
Lemma~\ref{lemma:scheduled_reflected_time_reversal}. Hence $\Qb^\star$ is
the unique law of \eqref{eq:full_interval_ideal_sde}.
Theorem~\ref{theorem:reflected_girsanov} requires a common initial law,
so we cannot apply it directly to $\Qb^\star$ and $\Qb_\theta$:
the former starts from $p_1$ and the latter from $\nu$.

To handle this difference, let $\Qb_\theta^{p_1}$ be the law of the
learned SDE \eqref{eq:learnable_reverse_rve_sde} started from $p_1$.
Assumption~\ref{assumption:control_continuity}, compactness of
$[0,1]\times\overline{\cD}$, and \eqref{eq:rve_noise_schedule} make its
drift bounded and jointly continuous. Assumption~
\ref{assumption:control_spatial_lipschitz} gives uniform spatial
Lipschitzness. Its diffusion coefficient $g(1-\rt)I_d$ is bounded,
continuous, and independent of space by \eqref{eq:rve_noise_schedule}.
Assumptions~\ref{assumption:bounded_convex_domain} and
\ref{assumption:theoretical_domain} supply the domain hypotheses of
Lemma~\ref{lemma:forward_reflected_well_posedness}. Applying it with $T=1$
and initial laws $\nu$ and $p_1$ gives the unique strong solutions and
unique laws $\Qb_\theta$ and $\Qb_\theta^{p_1}$, respectively.
These laws have the same conditional trajectory law given their
initial point: for a fixed initial value they solve the same uniquely
solvable SDE.

For every bounded Borel path functional $H$, conditioning on the initial
point and using this common conditional law gives
\[
\begin{aligned}
    \mathbb E_{\Qb_\theta^{p_1}}[H(Y)]
    &=\int_{\cD}\mathbb E_{\Qb_\theta}[H(Y)\mid Y_0=y]\,p_1(y)\,\dd y\\
    &=\int_{\cD}\mathbb E_{\Qb_\theta}[H(Y)\mid Y_0=y]
        \frac{p_1(y)}{\nu(y)}\,\nu(y)\,\dd y\\
    &=\mathbb E_{\Qb_\theta}
        \!\left[H(Y)\frac{p_1(Y_0)}{\nu(Y_0)}\right].
\end{aligned}
\]
This is the defining identity for
\begin{equation}
    \frac{\dd\Qb_\theta^{p_1}}{\dd\Qb_\theta}(Y)
    =\frac{p_1(Y_0)}{\nu(Y_0)}.
    \label{eq:continuous_learned_initial_tilt}
\end{equation}
For the probability density $\pi$, the positive-time density properties in
Section~\ref{sec:theoretical_assumptions} give a continuous, strictly positive
$p_1$ on the compact closure. Thus $p_1/\nu$ and its reciprocal are bounded,
and its logarithm is integrable.

It remains to compare $\Qb_\theta^{p_1}$ with $\Qb^\star$.
These laws now have the same initial law $p_1$, diffusion coefficient
$g(1-\rt)$, and normal reflection. Their drifts are
\[
    \alpha(y,\rt)=g(1-\rt)^2s_\theta(y,1-\rt),
    \qquad
    \beta(y,\rt)=g(1-\rt)^2\nabla\log p_{1-\rt}(y).
\]
Assumptions~\ref{assumption:control_continuity} and
\ref{assumption:closed_endpoint_score_continuity}, compactness, and the
schedule \eqref{eq:rve_noise_schedule} make both drifts continuous and bounded.
The domain conditions
and the common nonvanishing diffusion coefficient therefore allow
Theorem~\ref{theorem:reflected_girsanov} to compare these two laws.

Define the drift difference in diffusion coordinates by
\[
    e(y,\rt):=\frac{\beta(y,\rt)-\alpha(y,\rt)}{g(1-\rt)}
    =g(1-\rt)\bigl(\nabla\log p_{1-\rt}(y)-s_\theta(y,1-\rt)\bigr).
\]
In Theorem~\ref{theorem:reflected_girsanov}, the field
$u=(\alpha-\beta)/g(1-\rt)$ is therefore $-e$.
Its formula, written under a weak realization of $\Qb^\star$, gives
\begin{equation}
    \log\frac{\dd\Qb^\star}{\dd\Qb_\theta^{p_1}}(Y)
    =\int_0^1e(Y_{\rt},\rt)^\top\,\dd\overline W^\star_{\rt}
     +\frac12\int_0^1\|e(Y_{\rt},\rt)\|_2^2\,\dd\rt,
    \label{eq:continuous_same_initial_drift_density}
\end{equation}
where $\overline W^\star$ is the Brownian motion in that ideal-law
realization. In particular, the positive half-energy term is the one in
the density formula under the numerator law $\Qb^\star$.

We can now combine the two changes by the Radon--Nikodym chain rule:
\[
    \frac{\dd\Qb^\star}{\dd\Qb_\theta}
    =\frac{\dd\Qb^\star}{\dd\Qb_\theta^{p_1}}
      \frac{\dd\Qb_\theta^{p_1}}{\dd\Qb_\theta}.
\]
Taking logarithms and substituting
\eqref{eq:continuous_learned_initial_tilt} and
\eqref{eq:continuous_same_initial_drift_density} yields
\begin{equation}
    \log\frac{\dd\Qb^\star}{\dd\Qb_\theta}(Y)
    =\log\frac{p_1(Y_0)}{\nu(Y_0)}
     +\int_0^1e(Y_{\rt},\rt)^\top\,\dd\overline W^\star_{\rt}
     +\frac12\int_0^1\|e(Y_{\rt},\rt)\|_2^2\,\dd\rt.
    \label{eq:continuous_ideal_to_learned_density}
\end{equation}

\emph{Step 2: take the ideal-law expectation.}
KL is the expectation of this log density under its numerator law.
The initial term is
\[
    \mathbb E_{\Qb^\star}\!\left[\log\frac{p_1(Y_0)}{\nu(Y_0)}\right]
    =\int_{\cD}p_1(y)\log\frac{p_1(y)}{\nu(y)}\,\dd y
    =\KL{p_1}{\nu},
\]
because $Y_0\sim p_1$ under $\Qb^\star$.
The field $e$ is continuous on the compact domain
$\overline{\cD}\times[0,1]$, so there is a finite constant $M_e$ with
$\|e(y,\rt)\|_2\leq M_e$. Hence
$\mathbb E_{\Qb^\star}\int_0^1\|e(Y_{\rt},\rt)\|_2^2\,\dd\rt
\leq M_e^2<\infty$.
The integrand is also predictable. By
\citet[Eq.~(5.13)]{le2016brownian} and the subsequent moment identities,
its It\^o integral is therefore a square-integrable martingale started at
zero. In particular,
\[
    \mathbb E_{\Qb^\star}
    \!\left[\int_0^1e(Y_{\rt},\rt)^\top\,\dd\overline W^\star_{\rt}\right]=0,
    \qquad
    \mathbb E_{\Qb^\star}\int_0^1\|e(Y_{\rt},\rt)\|_2^2\,\dd\rt
    \leq M_e^2<\infty.
\]
Taking $\Qb^\star$-expectations in \eqref{eq:continuous_ideal_to_learned_density}
therefore gives
\[
    \KL{\Qb^\star}{\Qb_\theta}
    =\KL{p_1}{\nu}
     +\frac12\mathbb E_{\Qb^\star}
        \int_0^1\|e(Y_{\rt},\rt)\|_2^2\,\dd\rt.
\]
Finally, $Y_{\rt}\sim p_{1-\rt}$ under the ideal law. Tonelli's theorem
allows us to move the expectation inside the nonnegative energy integral,
and the substitution $t=1-\rt$ returns to forward time:
\[
\begin{aligned}
    \mathbb E_{\Qb^\star}\int_0^1\|e(Y_{\rt},\rt)\|_2^2\,\dd\rt
    &=\int_0^1g(1-\rt)^2
        \mathbb E_{\Qb^\star}\!\left[
        \left\|\nabla\log p_{1-\rt}(Y_{\rt})
                    -s_\theta(Y_{\rt},1-\rt)\right\|_2^2\right]\,\dd\rt\\
    &=\int_0^1g(t)^2\mathbb E_{X\sim p_t}\!\left[
        \left\|\nabla\log p_t(X)-s_\theta(X,t)\right\|_2^2\right]\,\dd t.
\end{aligned}
\]
Substituting this identity proves \eqref{eq:continuous_ideal_to_learned_kl}.
\end{proof}

\begin{remark}
    \label{remark:continuous_kl_directions}
Training uses $\KL{\Qb_\theta}{\Qb^\star}$ because its expectation is
under the learned law, from which we can generate trajectories.
The other direction, $\KL{\Qb^\star}{\Qb_\theta}$, is useful for describing
approximation error. Equation~\eqref{eq:continuous_ideal_to_learned_kl}
separates the mismatch between $p_1$ and $\nu$ from the error in the learned
score. That score error is averaged under the ideal marginals $p_t$,
which are fixed as $\theta$ changes. This gives an interpretable
sufficient condition for accurate estimation, but not a directly usable
training loss: the true scores and ideal-law expectations are generally
unavailable.

The two KLs compare the same continuous path laws on $[0,1]$, but neither
direction generally bounds the other. Pinsker's inequality gives
\begin{equation}
    2\|\Qb^\star-\Qb_\theta\|_{\mathrm{TV}}
    \leq\sqrt{2\min\{\KL{\Qb^\star}{\Qb_\theta},\KL{\Qb_\theta}{\Qb^\star}\}}.
    \label{eq:continuous_pinsker_comparison}
\end{equation}
Combining \eqref{eq:continuous_pinsker_comparison} with
\eqref{eq:continuous_normalizing_constant_tail} yields the KL accuracy bounds
in either direction. In particular, for $\varepsilon>0$ and $\delta\in(0,1)$,
either divergence being at most $\varepsilon^2\delta^2/2$ suffices for the
confidence guarantee \eqref{eq:continuous_normalizing_constant_confidence}.
\end{remark}

\begin{remark}
    \label{remark:dependence_on_n}
It can be noted that there is no dependence on $N$ in bound \eqref{eq:continuous_normalizing_constant_tail}. Dependence on $N$ can be recovered by using the variance.
For independent sampled paths, when $\chi^2(\Qb^\star\Vert\Qb_\theta)<\infty$, the weight identity
\eqref{eq:continuous_normalized_weight} and independence give
$\operatorname{Var}_{\Qb_\theta^{\otimes N}}(\widehat Z/Z)
=\chi^2(\Qb^\star\Vert\Qb_\theta)/N$.
Chebyshev's inequality then yields, for every $\varepsilon>0$,
\begin{equation}
    \PPs{\Qb_\theta^{\otimes N}}
        {\left|\frac{\widehat Z}{Z}-1\right|\geq\varepsilon}
    \leq\frac{\chi^2(\Qb^\star\Vert\Qb_\theta)}{N\varepsilon^2}.
    \label{eq:continuous_chebyshev_bound}
\end{equation}
This gives an explicit dependence on $N$, but requires control of the
$\chi^2$ divergence, which small KL alone does not provide.
\end{remark}

\subsection{Estimating gradients of the objective via REINFORCE}
\label{theorem:reinforce_gradient_proof}
Throughout this subsection, $\nabla_\theta f(y;\theta)$ denotes the
gradient with respect to the explicit parameter dependence of $f$, with
the state argument $y$ held fixed. Dependence of sampled states on
$\theta$ is accounted for by the score-function terms. We use
detached trajectories when constructing an autodifferentiable surrogate loss.
\begin{lemma} 
\label{lemma:reflected_score_identity}
Let $X^{\prime} \mid X=x \sim p(\cdot\mid x; \theta)$, where \(p(x'\mid x; \theta)\) is differentiable with respect to \(\theta\), and let $Y = T(X^{\prime})$, where \(T\) is a deterministic reflection map. Under regularity conditions, we have
\begin{equation} \label{eq:reflected_fisher_identity} 
\nabla_\theta \log q(y\mid x; \theta) = \E\left[ \left. \nabla_\theta \log p(X'\mid x; \theta) \,\right|\, Y=y,\;X=x \right]. \end{equation} Consequently, for any integrable function \(h:\mathcal D\to\mathbb R\), \begin{equation} \label{eq:reflected_score_expectation} \E\left[ h(Y) \nabla_\theta \log q(Y\mid X; \theta) \right] = \E\left[ h(Y) \nabla_\theta \log p(X'\mid X; \theta) \right]. 
\end{equation} 
\end{lemma}
\begin{proof}
Since $Y=T(X')$ is a deterministic map of $X'$, its conditional density can be represented as
\begin{equation} \label{eq:pushforward_density_reflection} q(y\mid x; \theta) = \int \delta\!\left(T(x')-y\right) p(x'\mid x; \theta) \,\dd x', \end{equation} 
where $\delta$ denotes the Dirac delta. Therefore, differentiating \eqref{eq:pushforward_density_reflection} with respect to \(\theta\) gives 
\begin{align} 
\nabla_\theta q(y\mid x; \theta) &= \int \delta\!\left(T(x')-y\right) \nabla_\theta p(x'\mid x; \theta) \,\dd x' \notag\\ &= \int \delta\!\left(T(x')-y\right) p(x'\mid x; \theta) \nabla_\theta\log p(x'\mid x; \theta) \,\dd x'. \label{eq:differentiate_pushforward} 
\end{align}
Dividing by \(q(y\mid x; \theta)\), we obtain \begin{align} 
\nabla_\theta\log q(y\mid x; \theta) &= \frac{ \int \delta(T(x')-y) p(x'\mid x; \theta) \nabla_\theta\log p(x'\mid x; \theta) \,\dd x' }{ q(y\mid x; \theta) } \notag\\ &= \E\left[ \left. \nabla_\theta\log p(X'\mid x; \theta) \,\right|\, Y=y,\;X=x \right],
\end{align}
where the last equality follows from the Bayes formula for $p(x^{\prime} \mid y, x; \theta)$. Then for any integrable function \(h:\mathcal D\to\mathbb R\), we have
\begin{align} 
&\EE{h(Y) \nabla_\theta\log p(X'\mid X; \theta)} = \EE{\EE{h(Y) \nabla_\theta\log p(X'\mid X; \theta) \mid Y, X}}\notag\\
&=\EE{h(Y)\,\EE{\nabla_\theta\log p(X'\mid X; \theta) \mid Y, X}} = \EE{h(Y) \nabla_\theta\log q(Y \mid X; \theta)}.
\end{align} 
\end{proof}

\begin{theorem}
\label{theorem:reinforce_gradient}
Suppose Assumptions~\ref{assumption:bounded_convex_domain} and
\ref{assumption:positive_target_reduction} hold. The gradient of
\eqref{eq:discrete_objective_reinforce} satisfies
\begin{equation}
    \label{eq:reinforce_gradient}
    \nabla_\theta\cL_{\text{RP}}(\theta)
    =
    \EEs{\widetilde{\tau} \sim \widetilde{q}(\cdot;\theta)}{
        \sum_{k=0}^{K-1}
        \nabla_\theta
        \phi_k(Y_k;\theta)
        +
        \sum_{k=0}^{K-1}
        c_k\,
        \nabla_\theta
        \log
        \widetilde q_k(
            \widetilde Y_{k+1}\mid Y_k;\theta
        )
    },
\end{equation}
where
\begin{equation}
    c_k :=
    \sum_{j=k+1}^{K-1}\phi_j(Y_j;\theta)
    +\psi(Y_K).
\end{equation}
\end{theorem}
\begin{proof}
Assumption~\ref{assumption:positive_target_reduction} makes
$\psi=\log(\nu/r)$ bounded and Borel measurable. For brevity, define
$\mu(\tau;\theta):=\sum_{j=0}^{K-1}\phi_j(Y_j;\theta)+\psi(Y_K)$. The density of the reflected trajectory factorizes as
\begin{equation}
    q(\tau;\theta)
    =
    \nu(Y_0)
    \prod_{k=0}^{K-1}
    q_k(Y_{k+1}\mid Y_k;\theta).
\end{equation}
Since $\nu$ does not depend on $\theta$, differentiating
    $\cL_{\mathrm{RP}}(\theta)
    =
    \E_{\tau\sim q(\cdot;\theta)}
    [\mu(\tau;\theta)]$
and applying the log-derivative identity gives
\begin{align}
\nabla_\theta \cL_{\mathrm{RP}}(\theta)
&=
\E_{\tau\sim q(\cdot;\theta)}
\left[
\sum_{j=0}^{K-1}
\nabla_\theta
\phi_j(Y_j;\theta)
+
\sum_{k=0}^{K-1}
\mu(\tau;\theta)
\nabla_\theta
\log q_k(Y_{k+1}\mid Y_k;\theta)
\right].
\label{eq:reinforce_proof_full_return}
\end{align}
For a fixed $k$, define 
\begin{equation}
    \mu(\tau;\theta)
    =
    \sum_{j=0}^{k}\phi_j(Y_j;\theta)
    +
    c_k,
\end{equation}
where $
    c_k
    =
    \sum_{j=k+1}^{K-1}\phi_j(Y_j;\theta)
    +\psi(Y_K)
$. Let $\mathcal F_k=\sigma(Y_0,\ldots,Y_k)$. The sum $\sum_{j=0}^{k}\phi_j(Y_j;\theta)$ is
$\mathcal F_k$-measurable, while the transition score has zero conditional
mean:
\begin{align}
\E\left[
\left.
\nabla_\theta
\log q_k(Y_{k+1}\mid Y_k;\theta)
\right|
\mathcal F_k
\right]
&=
\int
q_k(y\mid Y_k;\theta)
\nabla_\theta\log q_k(y\mid Y_k;\theta)
\,\dd y\notag\\
&=
\nabla_\theta
\int q_k(y\mid Y_k;\theta)\,\dd y
=0.
\end{align}
Consequently, by using a tower property of conditional expectation
\begin{equation}
\E\left[
\mu(\tau;\theta)
\nabla_\theta
\log q_k(Y_{k+1}\mid Y_k;\theta)
\right]
=
\E\left[
c_k
\nabla_\theta
\log q_k(Y_{k+1}\mid Y_k;\theta)
\right].
\label{eq:reinforce_proof_cost_to_go}
\end{equation}

It remains to replace the reflected transition score by the score of the
tractable pre-reflection Gaussian transition. According to \eqref{eq:euler_scheme}, the pre-reflection state is sampled as $
    \widetilde Y_{k+1}
    \sim
    \widetilde q_k(\cdot\mid Y_k;\theta)$, after which the reflected state is obtained by $
    Y_{k+1}
    =
    \operatorname{refl}_{\cD}(\widetilde Y_{k+1})
$. Define $h_k(Y_{k+1}):=\E[c_k\mid Y_{k+1}]$. By the Markov structure of the discretized process, once $Y_{k+1}$ is known,
the future reflected trajectory $(Y_{k+2},\ldots,Y_K)$ is conditionally
independent of the particular preimage $\widetilde Y_{k+1}$. Hence, $\E[c_k \mid Y_k,\widetilde Y_{k+1}]=h_k(Y_{k+1})$. Using the tower property,
\begin{equation}
\E_{\widetilde\tau\sim\widetilde q(\cdot;\theta)}
\left[
c_k
\nabla_\theta
\log
\widetilde q_k(
\widetilde Y_{k+1}\mid Y_k;\theta
)
\right]
=
\E_{\widetilde\tau\sim\widetilde q(\cdot;\theta)}
\left[
h_k(Y_{k+1})
\nabla_\theta
\log
\widetilde q_k(
\widetilde Y_{k+1}\mid Y_k;\theta
)
\right].
\label{eq:reinforce_proof_future_conditioning}
\end{equation}
Applying Lemma~\ref{lemma:reflected_score_identity} to $h_k$ with $X=Y_k, X^{\prime}=\widetilde Y_{k+1}$ and $Y=Y_{k+1}$ yields
\begin{align}
&\E_{\widetilde\tau\sim\widetilde q(\cdot;\theta)}
\left[
h_k(Y_{k+1})
\nabla_\theta
\log
\widetilde q_k(
\widetilde Y_{k+1}\mid Y_k;\theta
)
\right]\notag\\
&=
\E_{\tau \sim q(\cdot;\theta)}
\left[
h_k(Y_{k+1})
\nabla_\theta
\log q_k(
Y_{k+1}\mid Y_k;\theta
)
\right]\notag\\
&=
\E_{\tau \sim q(\cdot;\theta)}
\left[
c_k
\nabla_\theta
\log q_k(
Y_{k+1}\mid Y_k;\theta
)
\right].
\label{eq:reinforce_proof_fisher}
\end{align}

Substituting \eqref{eq:reinforce_proof_cost_to_go} and
\eqref{eq:reinforce_proof_fisher} into
\eqref{eq:reinforce_proof_full_return}, and observing that
$\widetilde q$ pushes forward to $q$ after discarding the pre-reflection
states, we obtain

$$
    \nabla_\theta\cL_{\text{RP}}(\theta)
    =
    \EEs{\widetilde{\tau}\sim\widetilde q(\cdot;\theta)}{
        \sum_{k=0}^{K-1}
        \nabla_\theta
        \phi_k(Y_k;\theta)
        +
        \sum_{k=0}^{K-1}
        c_k
        \nabla_\theta
        \log
        \widetilde q_k(
            \widetilde Y_{k+1}\mid Y_k;\theta
        )
    },
$$

which proves the claim.
\end{proof}
\subsection{Constructing unbiased estimate after discretization}
\begin{proposition}
\label{prop:discrete_estimator_unbiased}
Fix a discretization $\Delta$ as in Section~\ref{sec:practical}. Let
$(Y_k)_{k=0}^K$ be the Markov chain obtained by discretizing the reference
dynamics associated with $\Pb$, with $Y_0\sim\nu$ and
\[
    \widetilde Y_{k+1}=Y_k+g_k\xi_k,
    \qquad
    Y_{k+1}=\operatorname{refl}_{\cD}(\widetilde Y_{k+1}),
    \qquad
    \xi_k\stackrel{\mathrm{i.i.d.}}{\sim}\cN(0,I_d).
\]
Assume that each transition of this reference chain leaves $\nu$ invariant.
Then, for any fixed $\theta$, the estimator
\eqref{eq:discrete_Z_estimator}, computed from trajectories generated by
\eqref{eq:euler_scheme}, is unbiased:
\[
    \E[\widehat Z]=Z.
\]
\end{proposition}

\begin{proof}
Let $\widetilde Q_{\theta,\Delta}$ denote the law of the augmented
trajectory
\[
    \widetilde\tau
    =
    (Y_0,\widetilde Y_1,Y_1,\ldots,\widetilde Y_K,Y_K)
\]
generated by \eqref{eq:euler_scheme}, and let
$\widetilde P_\Delta$ denote the corresponding reference law obtained
by setting $s_\theta\equiv0$ in \eqref{eq:euler_scheme}. Their
pre-reflection transition densities are
\[
    \widetilde q_k(\widetilde y\mid y;\theta)
    =
    \cN\!\left(
        \widetilde y;
        y+g_k^2s_\theta(y,1-t_k),
        g_k^2I_d
    \right),
    \qquad
    \widetilde p_k(\widetilde y\mid y)
    =
    \cN(\widetilde y;y,g_k^2I_d).
\]
Since the two laws have the same initial distribution and use the same
deterministic reflection map,
\[
    \frac{\dd \widetilde P_\Delta}
         {\dd \widetilde Q_{\theta,\Delta}}
    (\widetilde\tau)
    =
    \prod_{k=0}^{K-1}
    \frac{
        \widetilde p_k(\widetilde Y_{k+1}\mid Y_k)
    }{
        \widetilde q_k(\widetilde Y_{k+1}\mid Y_k;\theta)
    }.
\]
Using the Gaussian representation in \eqref{eq:euler_scheme} gives
\[
    \log
    \frac{
        \widetilde p_k(\widetilde Y_{k+1}\mid Y_k)
    }{
        \widetilde q_k(\widetilde Y_{k+1}\mid Y_k;\theta)
    }
    =
    -g_k s_\theta(Y_k,1-t_k)^\top\xi_k
    -\frac12g_k^2
    \|s_\theta(Y_k,1-t_k)\|_2^2.
\]
Thus the exponential factor in \eqref{eq:discrete_Z_estimator} is
exactly
$\dd\widetilde P_\Delta/\dd\widetilde Q_{\theta,\Delta}$.
By assumption, each reference transition preserves $\nu$. Since
$Y_0\sim\nu$, it follows inductively that $Y_K\sim\nu$ under
$\widetilde P_\Delta$. For one trajectory, changing measure therefore gives
\begin{align*}
    &\E_{\widetilde Q_{\theta,\Delta}}
    \left[
        \frac{r(Y_K)}{\nu(Y_K)}
        \frac{\dd\widetilde P_\Delta}
             {\dd\widetilde Q_{\theta,\Delta}}
    \right]=
    \E_{\widetilde P_\Delta}
    \left[
        \frac{r(Y_K)}{\nu(Y_K)}
    \right]
    =\int_{\cD}\frac{r(y)}{\nu(y)}\nu(y)\,\dd y
    =Z.
\end{align*}
Averaging $N$ trajectories preserves the expectation, which proves
$\E[\widehat Z]=Z$.
\end{proof}
\begin{remark}
\label{remark:discrete_reference_invariance}
The invariance assumption in Proposition~\ref{prop:discrete_estimator_unbiased}
holds for the coordinate-wise folding reflection on
$\overline{\cD}=[0,1]^d$ used in our experiments. It is enough to consider one
coordinate. Let $\varphi_g$ denote the density of $\cN(0,g^2)$. For $x,y\in[0,1]$, the transition density of $Y=\operatorname{refl}_{[0,1]}(x+g\xi)$, where $\xi\sim\cN(0,1)$,
is
\[
    K_g(y\mid x)
    =
    \sum_{m\in\mathbb Z}
    \left[
        \varphi_g(y-x+2m)
        +
        \varphi_g(y+x+2m)
    \right].
\]
Since $\varphi_g(z)=\varphi_g(-z)$, reindexing $m\mapsto -m$ shows that
\[
    K_g(y\mid x)=K_g(x\mid y).
\]
As $K_g$ is a Markov kernel,
\[
    \int_0^1 K_g(x\mid y)\,\dd x=1,
\]
and therefore, by symmetry,
\[
    \int_0^1 K_g(y\mid x)\,\dd x=1.
\]
Hence the uniform distribution on $[0,1]$ is invariant under every
reference transition. Since both the Gaussian noise and the reflection
act coordinate-wise, the $d$-dimensional transition kernel factorizes
across coordinates, and $\nu=\operatorname{Unif}([0,1]^d)$ is invariant
as well.
\end{remark}

\section{Experimental details}\label{app_sec:experiments}

\columnratio{0.68}

\subsection{Benchmarks}
\label{app_sec:benchmarks}
\paragraph{Genz integrals.}
The Genz family provides a standard collection of numerical integration test functions, used in literature \citep{ott2023bayesian, warren2024fast}, designed to exhibit distinct forms of difficulty. We consider five out of six Genz integrals - we exclude the periodic integrand as it takes negative values, which does not fit our study. All Genz targets used here admit closed-form normalization constants \cite{kaarnioja2025explicit}, which allows to produce an exact reference value for any integration result. The five Genz targets span substantially different forms of concentration. Continuous and Gaussian Peak provide smooth interior modes, Product Peak has a factorized rational structure, Corner Peak concentrates mass at a boundary corner and Discontinuous combines a hard support constraint with an exponential trend within the support. The dimensional sweep from $5$ to $30$ further tests whether the sampler can preserve accurate mass allocation as the number of coordinates increases.

For the Genz experiments, several dimensions are considered, which are specified in the description of the corresponding problem. The concentration parameters are varied with dimension in order to prevent the benchmark from becoming either nearly uniform or excessively degenerate as $d$ increases. Unless otherwise stated, scalar parameters are shared across dimensions, so that the resulting target is symmetric with respect to coordinate permutations.

\begin{paracol}{2}
\paragraph{Continuous Genz.}
The Continuous Genz target is defined by
\[
    r(x) = \exp \left( -\sum_{k=1}^d a_k |x_k - u_k| \right), \qquad x\in[0,1]^d.
\]

The parameter $u_k$ specifies the location of the peak along dimension $k$, while $a_k$ controls its concentration. For all coordinates we take $u_k = 0.55$ and the scalar concentration parameter is chosen as $a_k = a \in \{8.0, 5.0, 4.0, 3.0 \}$ for $d \in \{5,10,20,30\}$, respectively. Geometrically, the target has a sharp non-smooth peak at $u=(u_k)_{k=1}^d$ and decays exponentially in the $\ell_1$ distance from it, producing diamond-shaped level sets with axis-aligned corners. As the dimension grows, its mass becomes increasingly concentrated in a small region of the domain, making accurate exploration and integration progressively more challenging.
\switchcolumn
\centering
\vspace{0.8em}
\includegraphics[width=\linewidth]{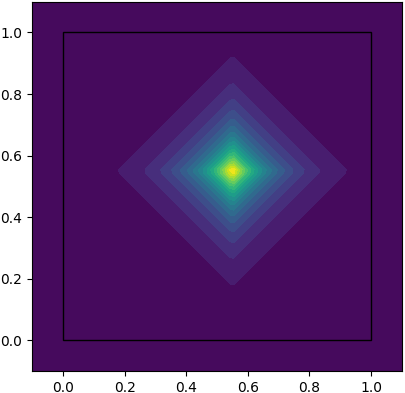}
\end{paracol}

\begin{paracol}{2}
\paragraph{Corner Peak Genz.}
The Corner Peak target is given by
\[
    r(x) = \left( 1 + \sum_{k=1}^d a_k x_k \right)^{-(d + 1)}, \qquad x\in[0,1]^d.
\]
The density is largest at the corner $x=0$ and decays as the weighted sum of the coordinates increases. For all coordinates we take $a_k=a \in \{ 0.50, 0.25, 0.10, 0.08 \}$ for $d \in \{5,10,20,30\}$, respectively. Because the dominant probability mass is associated with a boundary corner rather than an interior mode, this benchmark probes the ability of the sampler to discover and represent boundary-concentrated distributions. The progressively smaller $a_k$ values compensate for the increasing dimensionality and preserve a meaningful level of concentration across the dimensional sweep.
\switchcolumn
\centering
\vspace{0.8em}
\includegraphics[width=\linewidth]{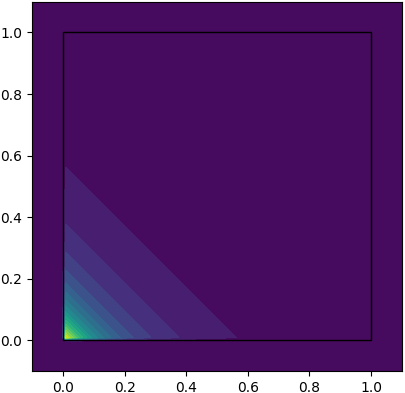}
\end{paracol}

\begin{paracol}{2}
\paragraph{Discontinuous Genz.}
The Discontinuous Genz is defined by
\[
    r(x) = \begin{cases}
        \exp\!\left( \sum_{k=1}^{d} a_k x_k \right), &x_k\leq u_k\ \forall k,\\
        0, &\text{otherwise}.
    \end{cases}
\]
We set $u_k=0.5$ for all coordinates. The shared concentration parameter is $a_k=a\in\{5.0,4.0,3.0,3.0\}$ for $d\in\{5,10,20,30\}$, respectively. This benchmark combines two distinct difficulties. First, the support is restricted to the subcube $[0,0.5]^d$, which introduces a sharp discontinuity at every boundary of the feasible region. Second, within the support the density increases exponentially with the coordinates, so that the mode is located at the upper corner $ (0.5,\ldots,0.5)^{\top}$ of the active region. 
\switchcolumn
\centering
\vspace{0.8em}
\includegraphics[width=\linewidth]{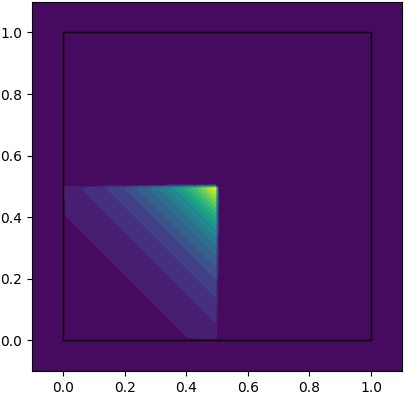}
\end{paracol}
Consequently, the sampler must simultaneously identify a discontinuous support restriction and concentrate probability near the boundary of that restricted support. In the implementation, we add a small positive correction of $\varepsilon=10^{-10}$ to make a target reward strictly positive and then subtract $\varepsilon \vol(\cD)$ in the final estimate.

\begin{paracol}{2}
\paragraph{Gaussian Peak Genz.}
The Gaussian Peak Genz is defined by 
\[
    r(x) = \exp \left( -\sum_{k=1}^d a_k^2 (x_k - u_k)^2 \right), \qquad x\in[0,1]^d.
\]
We use $u_k=0.5$ for all coordinates and set $a_k=a \in \{ 8.0, 5.0, 4.0, 3.0 \}$ for $d \in \{ 5, 10, 20, 30 \}$, respectively.

The resulting distribution is a sharply concentrated, smooth peak centred at the middle of the hypercube. Relative to the Continuous target, the Gaussian Peak is globally differentiable and has rapidly decaying tails away from its mode, making the pair useful for distinguishing the effect of nonsmoothness from that of concentration alone.
\switchcolumn
\centering
\vspace{0.8em}
\includegraphics[width=\linewidth]{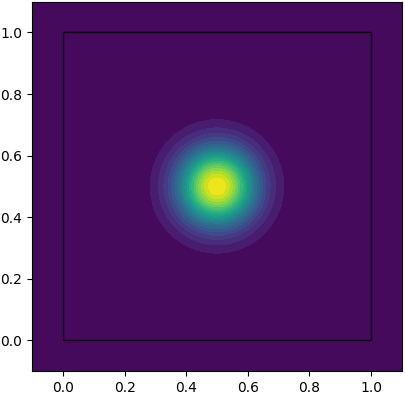}
\end{paracol}

\begin{paracol}{2}
\paragraph{Product Peak Genz.}
The Product Peak Genz target is given by
\[
    r(x) = \prod_{k=1}^d \frac{1}{a_k^{-2}+(x_k-u_k)^2}, \qquad x\in[0,1]^d.
\]

We take $u_k=0.5$ in every coordinate. The shared coefficient $a_k=a$ takes values $8.0$, $5.0$, $4.0$, and $4.0$ for dimensions $d=5$, $10$, $20$, and $30$, respectively.

The function factorizes across dimensions, but each factor exhibits a sharply localized peak around $u_k=0.5$. In the implementation, we rescale $r(x)$ by the maximum reward for numerical stability and adjust the analytic normalizing constant accordingly.
\switchcolumn
\centering
\vspace{0.8em}
\includegraphics[width=\linewidth]{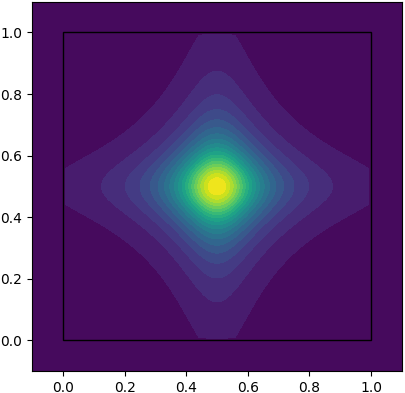}
\end{paracol}

\begin{paracol}{2}
\paragraph{GMM Cube.}
The GMM Cube target is a bounded multimodal distribution constructed by placing a mixture of five isotropic Gaussian components inside the unit hypercube. For a $d$-dimensional target
\[
    r(x) = \frac{1}{5} \sum_{j=1}^5 \mathcal{N}(x; \mu_j, \sigma^2 I_d), \qquad x\in[0,1]^d,
\]
with $\sigma = 0.1$. The component locations are sampled once from the interior region $[0.01, 0.99]^d$, thereby ensuring all modes remain separated from the boundary.
\switchcolumn
\centering
\vspace{0.8em}
\includegraphics[width=\linewidth]{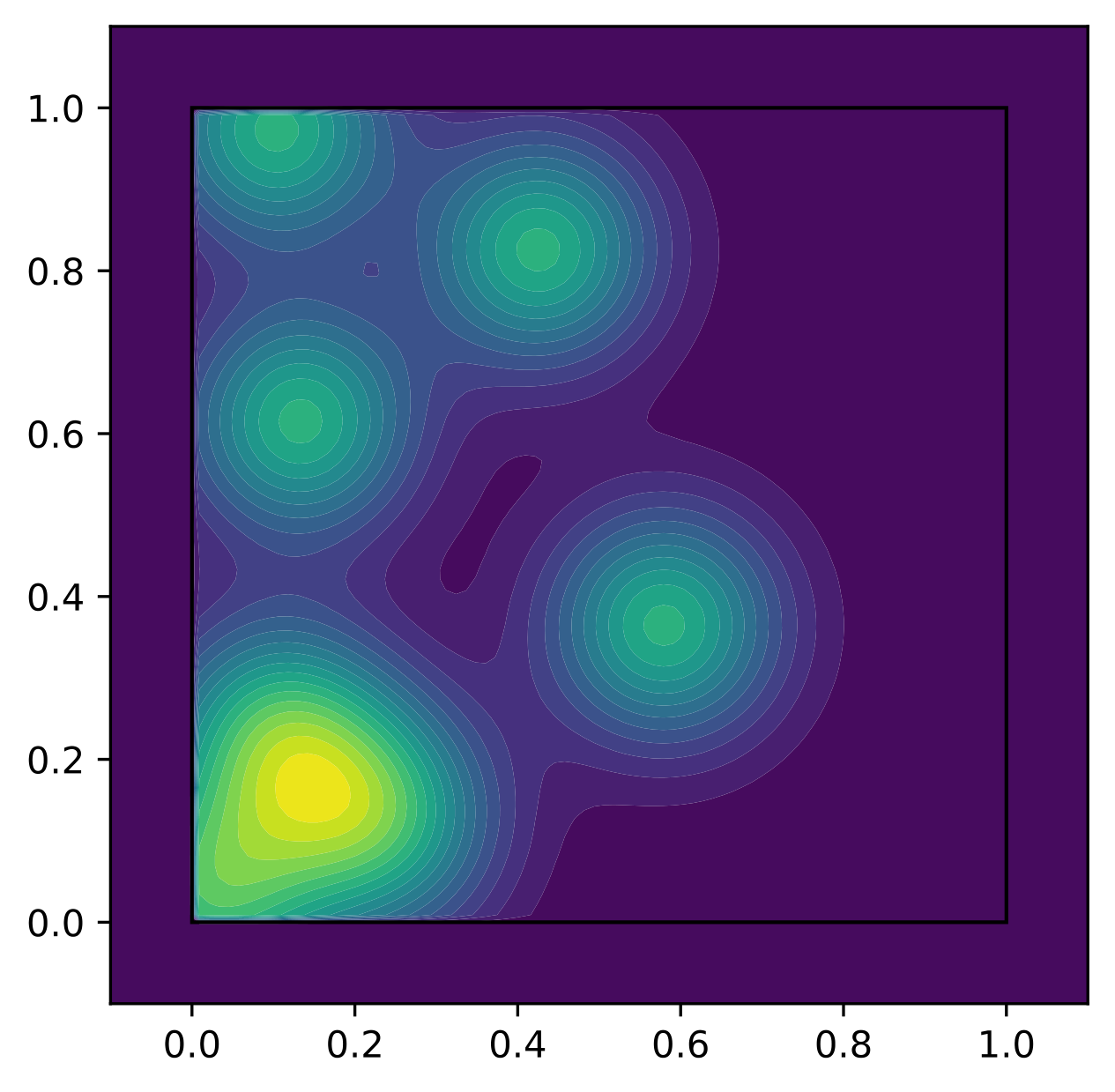}
\end{paracol}
The target is subsequently restricted to the unit cube. This construction produces a genuinely multimodal bounded target in which the sampler must discover several well-separated regions of high probability rather than a single dominant mode.

Because the Gaussian mixture is initially defined on the full Euclidean space and then restricted to the cube, its normalization constant is the probability that a draw from the unconstrained mixture falls inside $[0,1]^d$. This quantity is estimated in the implementation by Monte Carlo sampling from the unconstrained mixture. The target therefore tests both multimodal exploration and the ability to account for truncation by the bounded domain.

\paragraph{Stripe Cube.}
The Stripe Cube target is designed to concentrate probability near the boundary of the unit hypercube while retaining a small positive density throughout its interior. For a point $x \in [0,1]^d$, define the distance to the closest boundary face by $\delta(x) = \min_{i} \{ x_i, 1 - x_i \}$. We map this distance to a normalized boundary score $s(x) = 1 - 2 \delta(x)$, which equals zero at the centre of the cube and one on the boundary.
\begin{paracol}{2}
The target density is then
\[
    r(x) = \alpha + (\beta - \alpha) s^{\gamma}(x), \qquad x\in[0,1]^d,
\]
where $\alpha=10^{-10}, \beta=1$ and $\gamma=8$. Hence the density rises steeply as a point approaches any boundary face, producing a set of high-probability boundary ''stripes'' around the cube while retaining a numerically nonzero floor in its interior.

Whereas the Corner Peak benchmark concentrates mass near a single corner, Stripe Cube assigns high density collectively to all boundary faces.

\switchcolumn
\centering
\vspace{0.8em}
\includegraphics[width=\linewidth]{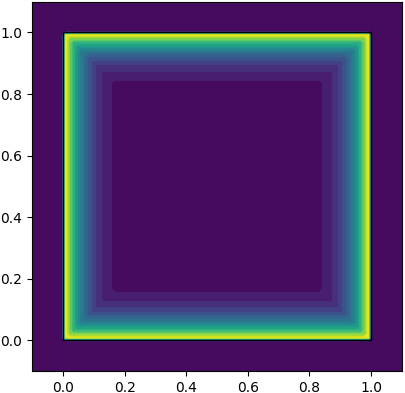}
\end{paracol}
The resulting target therefore contains a large, geometrically thin high-density region whose total structure becomes increasingly complex with dimension. The normalization constant is available analytically
\[
Z
=
\alpha
+
\left(\beta-\alpha\right)
\frac{d}{d+\gamma}.
\]

\begin{paracol}{2}
\paragraph{Funnel Cube.}
The Funnel Cube target is a bounded analogue of the classical funnel distribution and is designed to test adaptation to strongly varying conditional scales. The first coordinate follows a Gaussian density,
\[
    x_1 \sim \mathcal{N} (0.6, 0.22^2),
\]
restricted to the interval $[0, 1]$. Conditional on $x_1$, the remaining $d-1$ coordinates are independent Gaussian variables centered at $0.5$ with a standard deviation that changes exponentially across the first coordinate:
\[
    \log \sigma(x_1) = \log \sigma_{\mathrm{min}} + x_1 \left( \log \sigma_{\mathrm{max}} - \log \sigma_{\mathrm{min}} \right),
\]
where $\sigma_{\mathrm{min}} = 0.02$ and $\sigma_{\mathrm{max}} = 0.30$.
\switchcolumn
\centering
\vspace{0.8em}
\includegraphics[width=\linewidth]{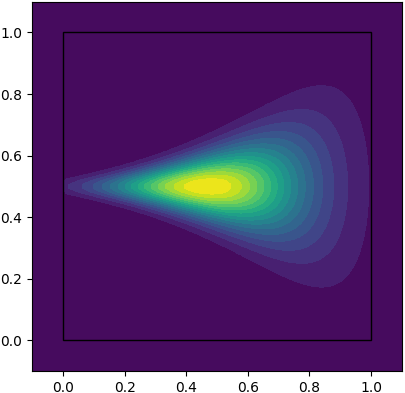}
\end{paracol}
The target density can therefore be written as
\[
    r(x) = \mathcal{N}(x_1;0.6,0.22^2) \prod_{k=2}^{d} \mathcal N\!\left( x_k;0.5,\sigma^2(x_1) \right), \qquad x\in[0,1]^d.
\]
The defining difficulty of this benchmark is the strong dependence of the conditional width on $x_1$. Samples associated with one region of the first coordinate lie in a very narrow tube around $0.5$, whereas samples associated with another region can occupy a considerably wider cross-section. The target therefore tests whether a learned sampler can represent pronounced heteroscedasticity rather than simply locating isolated modes. Its normalization constant is computed by reducing the $d$-dimensional integral to a one-dimensional numerical integration over $x_1$, where the probability that every remaining coordinate lies in $[0,1]$ is evaluated analytically.

\paragraph{Asian option.}
We consider Asian option pricing as an application-oriented high-dimensional benchmark with a nonsmooth payoff and a strong temporal dependence structure, following the standard construction \citet{l2009quasi}. The implementation is based on a discretised Black--Scholes asset path. A point $x=(x_1,\ldots,x_d)\in[0,1]^d$ is first transformed into independent standard-normal shocks through the inverse Gaussian cumulative distribution function, $\epsilon_i=\Phi^{-1}(x_i)$. For monitoring times $0=t_0<t_1<\cdots<t_d=T$, the Brownian path is then constructed as
\[
W_{t_j}=\sum_{i=1}^{j}\sqrt{t_i-t_{i-1}}\,\epsilon_i,
\]
and the corresponding asset price at time $t_j$ is
\[
S_{t_j}
=
S_0\exp\left[
\left(r-\frac12\sigma^2\right)t_j+\sigma W_{t_j}
\right].
\]
The default parameters are $S_0=100$, $r=\log(1.09)$, $\sigma=0.2$, and $T=\frac{120}{365}$. The reward is given by the discounted arithmetic-average call payoff
\[
r(x)=e^{-rT}
\max(0,\frac1d\sum_{i=1}^{d}S_{t_i}-K),
\]
where $K$ is the strike price. The arithmetic average couples the payoff to the entire simulated trajectory, while the positive-part operator introduces a nonsmooth boundary and makes the reward exactly zero over a substantial portion of the input domain.

We consider two principal configurations from Example 3 \citep{l2009quasi}. The easy problem has dimension $d=10$, uses a starting monitoring period of $111$ days, a strike of $K=90$, and has a ground-truth option value $Z=13.008$. The hard problem increases the dimension to $d=120$, starts monitoring at period $1$, retains the strike $K=90$, and has $Z=11.207$.

The difference between the two settings therefore changes both the number of Gaussian increments used to represent the asset path and the portion of the maturity interval entering the arithmetic average. The hard configuration consequently presents a substantially higher-dimensional integration problem with a longer temporal dependence structure, providing a stringent test of the sampler's scalability.

As for the Genz Discontinuous benchmark, we add a small constant $\varepsilon=10^{-10}$ to the reward to make it strictly positive. We remove this artificial contribution from the final integral estimate by subtracting $\varepsilon\vol(\mathcal D)$.

\paragraph{Digits.}
We adapt the Digits benchmark from \citet{blessing2024beyond}, where the target is a NICE normalizing-flow density \citep{dinh2014nice} trained on downsampled $14\times14$ MNIST images. The original target is a normalized continuous density $p_{\mathrm{NICE}}$ on $\mathbb{R}^{196}$. To obtain a bounded-domain integration problem, we apply the coordinate-wise logit transformation
and define the cube target by the corresponding change of variables,
\[
r(x)
=
p_{\mathrm{NICE}}(\operatorname{logit}(x))
\prod_{j=1}^{196}\frac{1}{x_j(1-x_j)}.
\]
This transformation preserves the normalizing constant, so $Z=1$. Samples generated on the cube can be mapped back through the logit transformation and reshaped into $14\times14$ grayscale images. Beyond integration, this benchmark provides a challenging high-dimensional sampling task with complex learned structure, allowing us to test whether RDS can recover meaningful samples from the target distribution itself.

\begin{paracol}{2}
\paragraph{Elliptic Two-Loop.}
The Elliptic Two-Loop target is a Feynman-parameter integral arising in
collider phenomenology. It corresponds to the planar two-loop master integral
$I^A_{1,1,0,1,1,1,1,0,0}$ with full heavy-quark mass dependence, which
contributes to the next-to-leading-order corrections to Higgs-plus-jet
production \citep{Bonciani_2016}. This sector is distinguished by the
appearance of elliptic functions in its analytic solution.

\switchcolumn
\centering
\includegraphics[width=\linewidth]{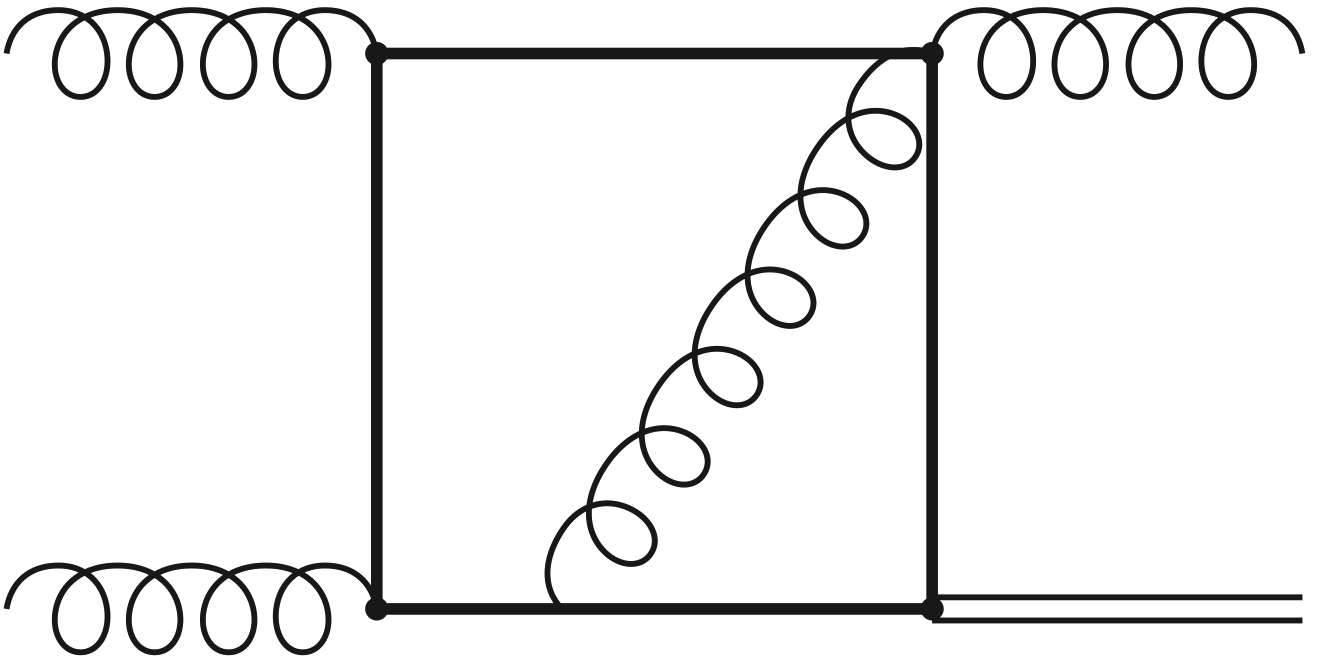}
\captionof{figure}{Two-loop six-propagator graph giving rise to elliptic functions \citep{Borowka_2018}.}
\end{paracol}

Using the propagator notation of \citet{Bonciani_2016}, the six active
propagators are
\[
    D_1 = k_1^2 - m^2, \qquad
    D_2 = (k_1 + p_1 + p_2)^2 - m^2, \qquad
    D_4 = (k_2 + p_1 + p_2)^2 - m^2,
\]
\[
    D_5 = (k_1 + p_1)^2 - m^2, \qquad
    D_6 = (k_1 - k_2)^2, \qquad
    D_7 = (k_2 - p_3)^2 - m^2.
\]
We evaluate the integral at the Euclidean kinematic point
\[
    s = (p_1+p_2)^2 = -\tfrac{4}{3}, \qquad
    t = (p_1+p_3)^2 = -\tfrac{16}{5}, \qquad
    p_4^2 = -\tfrac{100}{39}, \qquad
    m^2 = 1,
\]
used in the numerical benchmark of \citet{Borowka_2018}.

Feynman parametrization gives an integral over the five-dimensional simplex
$\{\alpha_i\geq0,\ \sum_{i=1}^6\alpha_i=1\}$. To express it on the unit
hypercube, we use the stick-breaking transformation
\[
    \alpha_i
    =
    z_i\prod_{j<i}(1-z_j),
    \qquad i=1,\ldots,5,
    \qquad
    \alpha_6
    =
    \prod_{j=1}^{5}(1-z_j),
\]
for $z\in[0,1]^5$, whose Jacobian is $J(z)=\prod_{j=1}^{4}(1-z_j)^{5-j}$.
In four spacetime dimensions, the corresponding cube target is
\[
    r(z)
    =
    \left(-\frac{s}{m^2}\right)^{3/2}
    \frac{J(z)}
         {\mathcal F(\alpha(z))^2},
    \qquad z\in[0,1]^5,
\]
where $\mathcal F$ is the second Symanzik polynomial. The Euclidean
kinematics make the target real and non-negative.

The benchmark combines a non-factorized Feynman integrand with nontrivial
boundary structure, providing an application-oriented test beyond the
synthetic targets considered above. Its normalization is known analytically
and, in the normalization used here, is
\[
    Z = 0.247074199140732131068066,
\]
in agreement with the numerical and analytic values reported by
\citet{Borowka_2018}.

\subsection{Evaluation protocol and training details}
\label{app_sec:experimental_details}

\paragraph{Evaluation protocol.}
We compare methods across different budgets corresponding to the number of points used for the integral estimation. For neural network-based methods, namely PIS, ZüNIS and RDS, for a fixed number of points used for evaluation, the same number of points is used for training. For these methods we use a batch size of $512$ in all experiments. We report the mean absolute error $|\widehat Z-Z|$ over independent runs as a function of the number of points. For all experiments we present mean and std of the error across 3 random seeds.

\paragraph{Quasi-Monte Carlo.}
As a non-adaptive baseline, we use randomized quasi-Monte Carlo (QMC) with scrambled Sobol sequences \cite{l2009quasi}. For a budget $N=2^m$, we generate $N$ points $x^{(1)},\ldots,x^{(N)}\in[0,1]^d$ from a scrambled Sobol sequence and estimate
\begin{equation}
\widehat Z_{\mathrm{QMC}}
=
\frac{1}{N}
\sum_{i=1}^{N} r(x^{(i)}).
\end{equation}
Independent random scramblings are used across runs. Using powers of two preserves the digital-net structure of the Sobol construction. QMC does not involve any optimization, and its target-evaluation budget is exactly $N$. We use SciPy implementation~\cite{virtanen2020scipy}.

\paragraph{Fast Fourier Bayesian Quadrature.}
We use the public implementation of Fast Fourier Bayesian Quadrature (FFBQ) by \citet{warren2024fast}. FFBQ approximates the integrand with a Gaussian process and computes the resulting Bayesian-quadrature estimate using scalable convolution operators. In our experiments, we fit a low-rank GP with an RBF kernel using $5{,}000$ uniformly sampled training points and use an independent set of $5{,}000$ points for quadrature evaluation. We employ the sparse QMC convolution operator provided by the authors. Low-rank GP was chosen, because of the scalability issues with dimension of FFBQ. The GP training and evaluation sets are fixed across experiments, while we vary the number $N$ of QMC samples used to approximate the final convolution integral. 

\paragraph{Path Integral Sampler.}
We include the Path Integral Sampler (PIS) \citep{zhang2022path}, an on-policy diffusion sampler derived from a stochastic-control formulation of sampling. To apply PIS on the bounded domain, we instead use a change of variables between $\mathbb{R}^d$ and $[0,1]^d$. Specifically, for $z\in\mathbb{R}^d$, we set $x=\sigma(z)$, where $\sigma(\cdot)$ is a sigmoid function. Accordingly, PIS is trained on the transformed reward
$
    \widetilde r(z)
    =
    r\!\left(\sigma(z)\right)
    \left|\det J_{\sigma}(z)\right|,
$ which preserves the original integral by the change-of-variables formula. The control is parameterized by an MLP with four hidden layers of width $256$. We train PIS with Adam using a learning rate of $10^{-3}$. Each diffusion trajectory is discretized using $64$ time steps. The control is trained using the standard path-space objective consisting of the accumulated quadratic control cost and the terminal target-density contribution. The resulting estimate of $Z$ is obtained using the path-integral importance weights associated with the learned controlled process. For PIS, we used a public implementation from \cite{blessing2024beyond}.

\paragraph{ZüNIS.}
We use the public implementation of ZüNIS \citep{deutschmann2024accelerating}. The proposal is parameterized by a bounded piecewise-quadratic coupling flow with $8$ coupling layers, $10$ bins per scalar transformation, and conditioner networks with two hidden layers of width $256$. We optimize the proposal with the reverse-KL objective using Adam with learning rate $10^{-3}$ and global gradient-norm clipping at $1$. The first batch is sampled uniformly from the cube, while all subsequent batches are drawn from the current flow proposal. Each evaluated survey batch is used for a single optimization update. At evaluation time, we draw fresh samples from the learned proposal and estimate the integral using importance sampling,
\begin{equation}
\widehat Z_{\mathrm{Z\ddot{u}NIS}}=
\frac{1}{N}
\sum_{i=1}^{N}
\frac{r(x^{(i)})}{q_\theta(x^{(i)})},
\qquad
x^{(i)}\sim q_\theta.
\end{equation}

\paragraph{RDS-RP and RDS-RF.}

Both variants use the same reflected diffusion and differ only in the gradient estimator used for training. We employ the Langevin
parameterization of the drift, following
\citet{zhang2022path} and \citet{blessing2024beyond}:
\begin{equation}
    s_\theta(x,t)
    =
    \operatorname{clip}
    \left(
        a_\theta(x,t)
        +
        b_\theta(t) \odot
        \operatorname{clip}
        \bigl(\nabla_x \log r(x), -10^2, 10^2\bigr), -10^{4}, 10^4
    \right),
\end{equation}
where $a_\theta(x,t)$ is a time-conditioned neural network,
$b_\theta(t)\in\mathbb R^d$ is a learned time-dependent
coordinate-wise scaling function, and $\odot$ denotes element-wise
multiplication. This parameterization incorporates the target score directly into
the drift while allowing the neural network to learn a flexible time-dependent correction. For targets with zero-valued regions we do not incorporate $\nabla_x \log r(x)$ to the parameterization. The score network is parameterized by a $4$-layer MLP with hidden width $256$. We use $\sigma_0=10^{-3}$ and $\sigma_1=1$, and discretize the diffusion into $64$ reflected steps for integration tasks. Both models are trained on-policy with Adam using the same optimization hyperparameters and global gradient-norm clipping at $1$. 

For RDS-RP, gradients are propagated through the simulated trajectory using the reparameterized Gaussian increments and the coordinate-wise reflection map. Since reflection on the cube is piecewise linear, it is differentiable almost everywhere and can be handled directly by automatic differentiation. RDS-RF instead uses the score-function estimator described in Section~\ref{sec:practical}. The sampled states are treated as stop-gradient quantities, while the dependence on the proposal parameters is captured through the scores of the Gaussian transitions and their corresponding cost-to-go terms. Thus, RDS-RP and RDS-RF use the same proposal family, discretization, and training objective, providing a direct comparison between pathwise and score-function optimization.

Our implementations are based upon the published codebase of~\cite{blessing2024beyond}.

\paragraph{Digits experiments.}
For the Digits benchmark, we use the PIS architecture and remaining training
hyperparameters following \citet{blessing2024beyond}. Both PIS and RDS are
discretized using $128$ diffusion steps and trained with Adam using a learning
rate of $5\times10^{-4}$, batch size of $2048$ and $40,000$ iterations. For RDS, we set the terminal noise scale to
$\sigma_1=0.5$.

We evaluate sample quality using the $2$-Wasserstein distance and the evidence
lower bound (ELBO), following \citet{blessing2024beyond}. Since the NICE target
admits direct sampling, we draw an independent set of reference samples from the
target and compare them with samples generated by each method. For RDS, samples
generated on $[0,1]^{196}$ are first mapped back to the original
$\R^{196}$ space using the coordinate-wise logit transformation. We compute the
entropy-regularized empirical $2$-Wasserstein distance between the generated and
reference samples using the squared Euclidean transport cost and regularization
$\varepsilon=10^{-3}$.

The ELBO can be computed directly from the same path-space importance weights
used for normalizing-constant estimation. For a trajectory
$\widetilde\tau$ generated by RDS, let
\[
    w(\widetilde\tau)
    :=
    \frac{r(Y_K)}{\nu(Y_K)}
    \exp\left(
        -\sum_{k=0}^{K-1}
        \left[
            g_k s_\theta(Y_k,1-t_k)^\top\xi_k
            +\frac12 g_k^2
            \|s_\theta(Y_k,1-t_k)\|_2^2
        \right]
    \right),
\]
so that $\E[w(\widetilde\tau)]=Z$ by
Proposition~\ref{prop:discrete_estimator_unbiased}. The corresponding
path-space ELBO is
\[
    \mathrm{ELBO}
    :=
    \E[\log w(\widetilde\tau)]
    \leq
    \log \E[w(\widetilde\tau)]
    =
    \log Z,
\]
where the inequality follows from Jensen's inequality. We estimate it using
\[
    \widehat{\mathrm{ELBO}}
    =
    \frac{1}{N}\sum_{i=1}^N
    \log w(\widetilde\tau^{(i)}).
\]
For PIS, we analogously use its path-space importance weights. Higher ELBO and
lower Wasserstein distance indicate better agreement with the target
distribution.

\subsection{Full experimental results}
\label{app_sec:experiments_results}
\begin{figure*}[h!]
    \centering

    \includegraphics[width=0.32\textwidth]{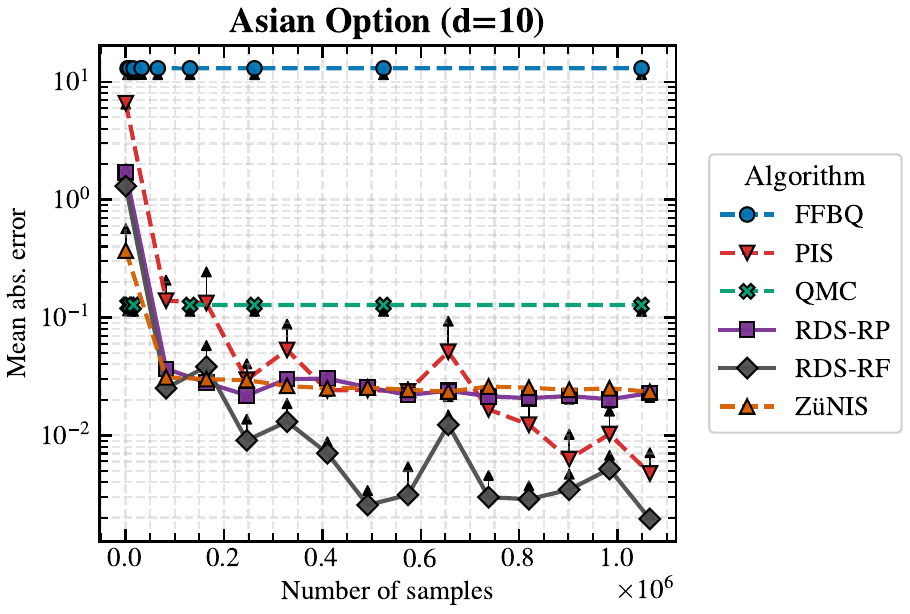}
    \hfill
    \includegraphics[width=0.32\textwidth]{figures/nfe_err/Asian_Option_d-120.pdf}
    \hfill
    \includegraphics[width=0.32\textwidth]{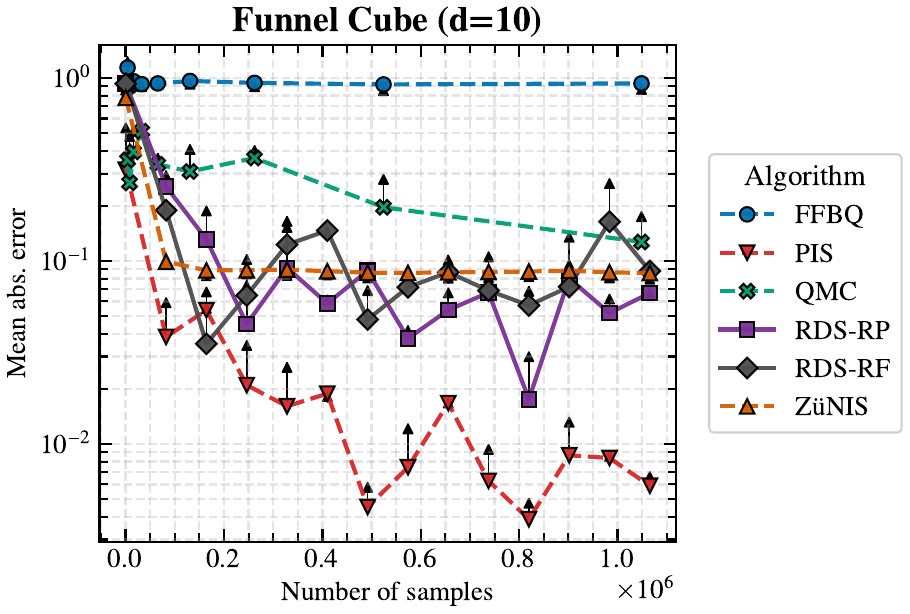}

    \vspace{0.3em}

    \includegraphics[width=0.32\textwidth]{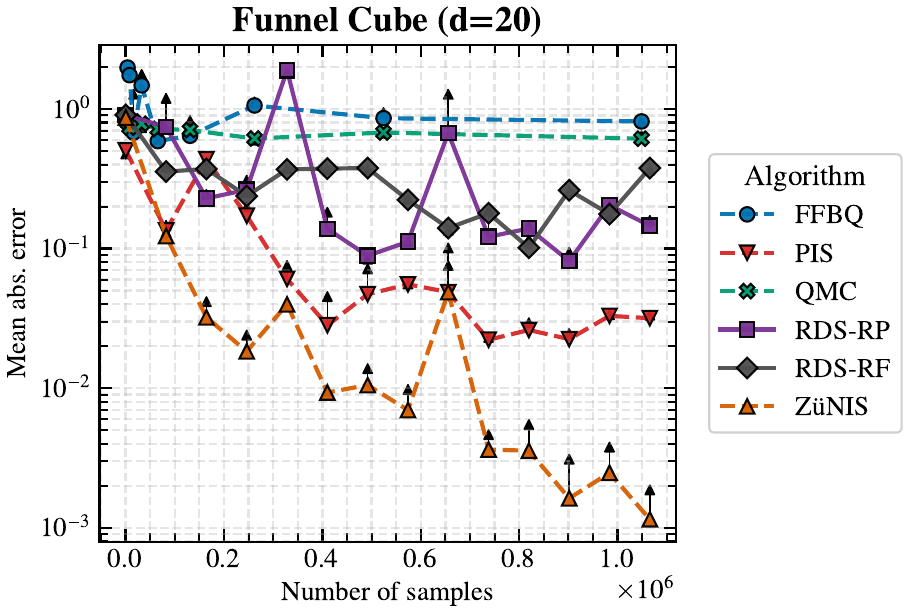}
    \hfill
    \includegraphics[width=0.32\textwidth]{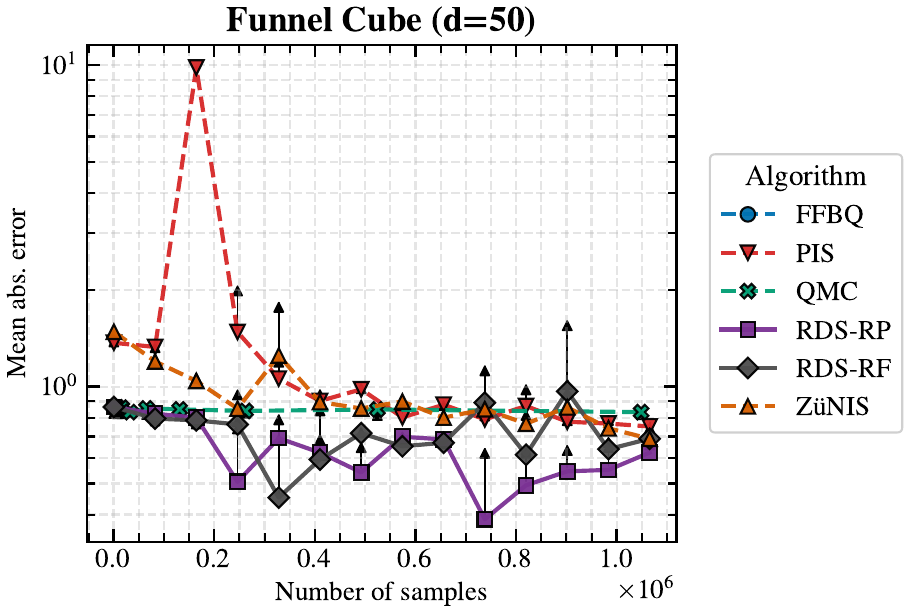}
    \hfill
    \includegraphics[width=0.32\textwidth]{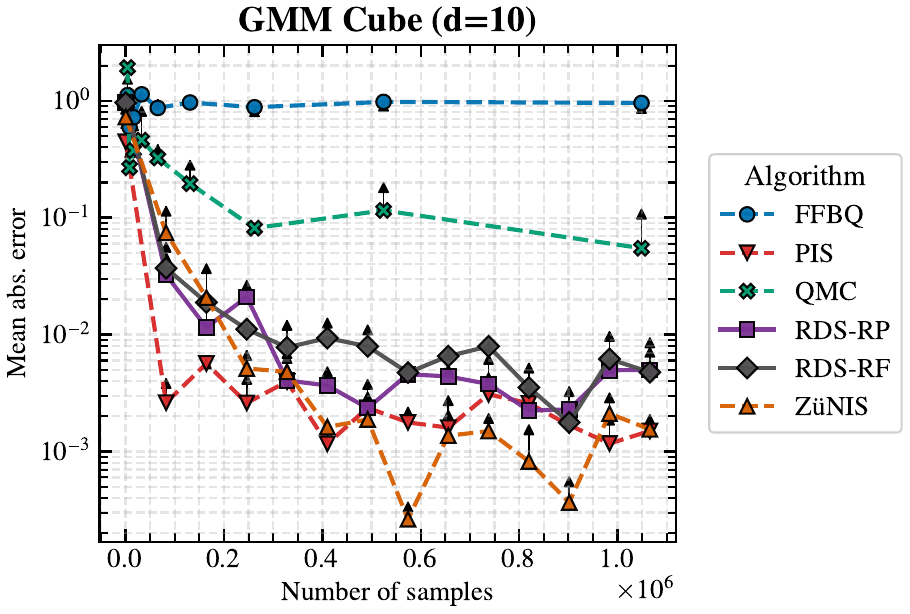}

    \vspace{0.3em}

    \includegraphics[width=0.32\textwidth]{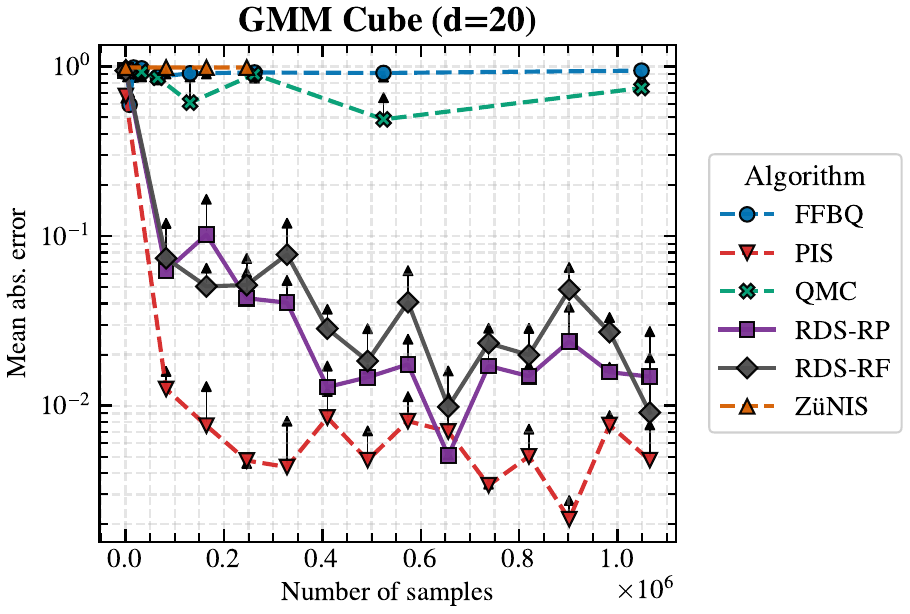}
    \hfill
    \includegraphics[width=0.32\textwidth]{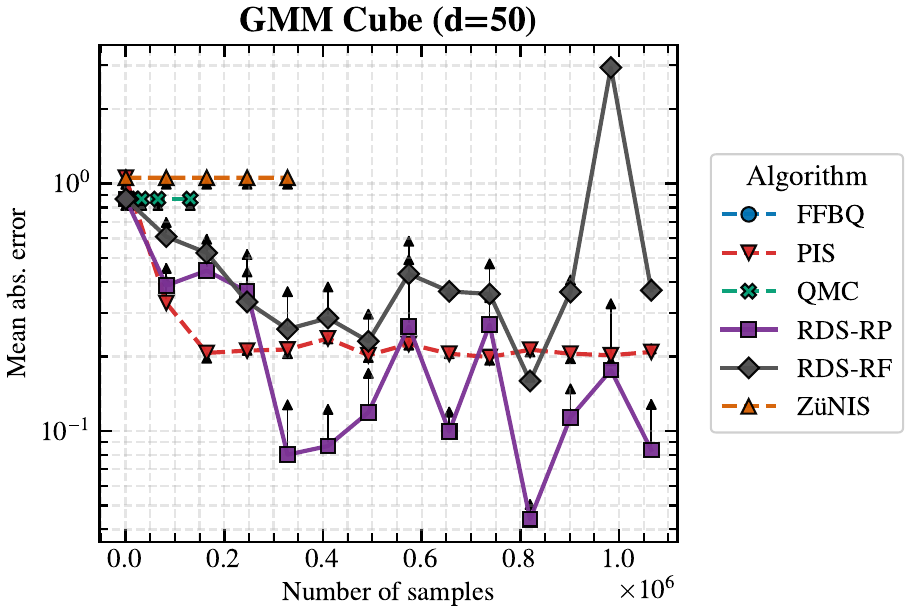}
    \hfill
    \includegraphics[width=0.32\textwidth]{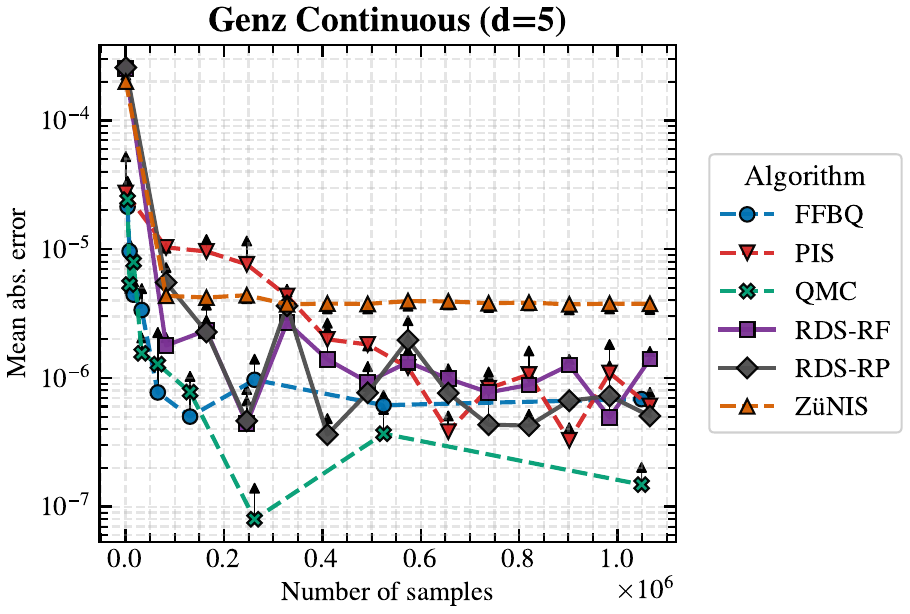}

    \vspace{0.3em}

    \includegraphics[width=0.32\textwidth]{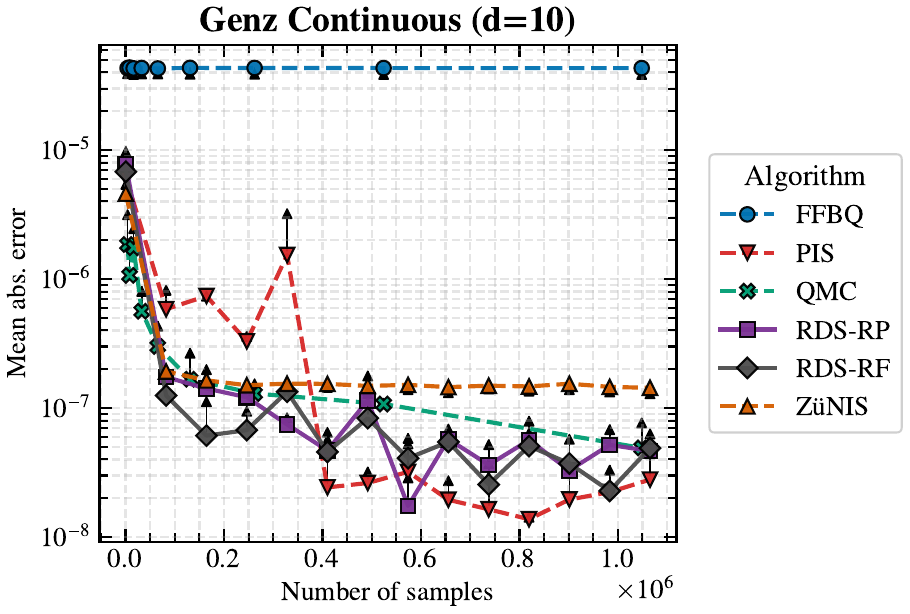}
    \hfill
    \includegraphics[width=0.32\textwidth]{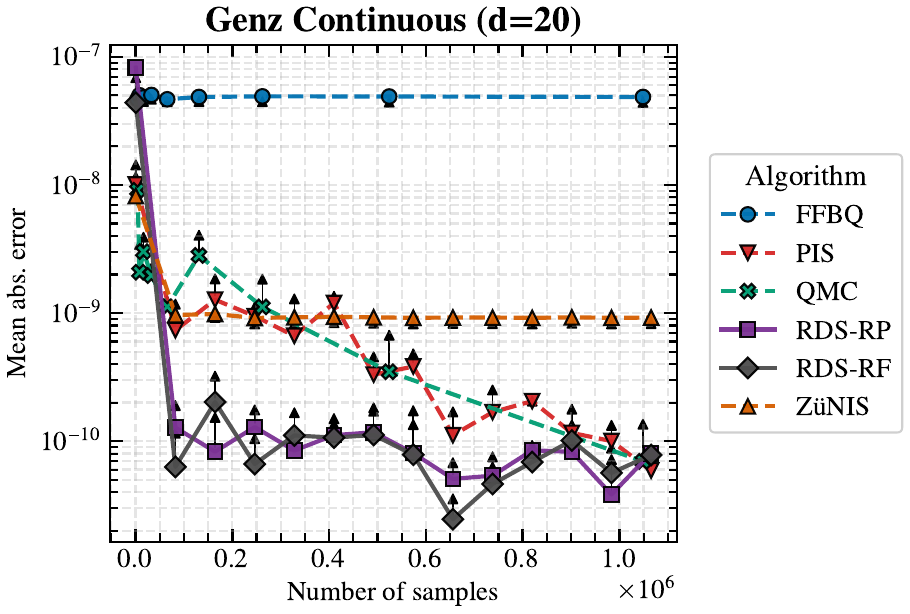}
    \hfill
    \includegraphics[width=0.32\textwidth]{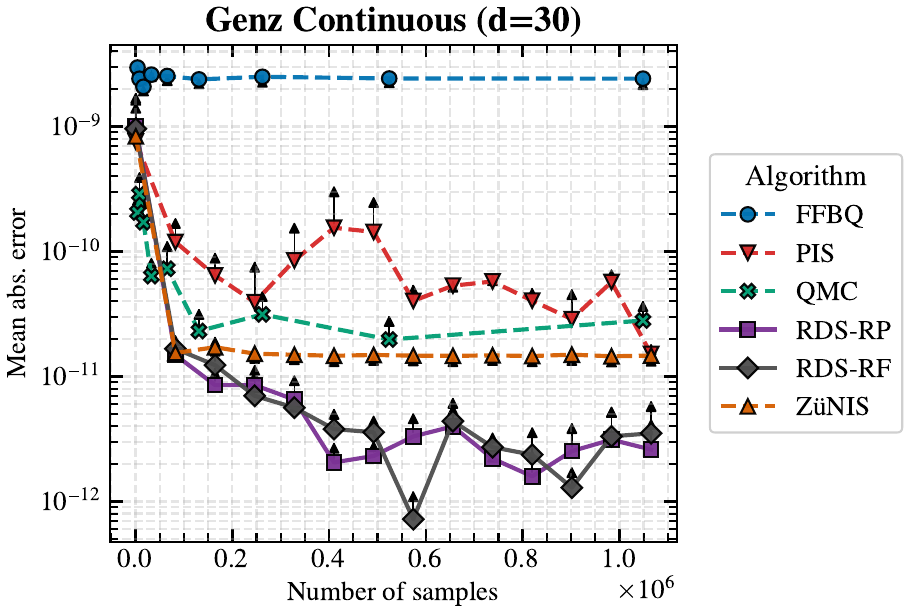}

    \caption{
        Mean absolute error as a function of the number of samples used for estimation
        across the considered integration benchmarks and dimensionalities.
    }
    \label{fig:nfe_scaling_full_1}
\end{figure*}

\begin{figure*}[p]
    \centering

    \includegraphics[width=0.32\textwidth]{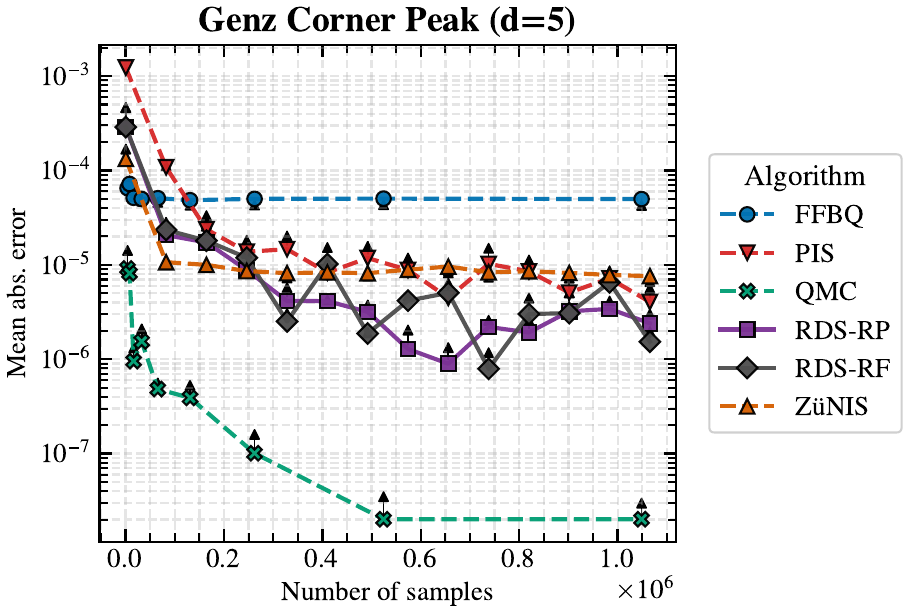}
    \hfill
    \includegraphics[width=0.32\textwidth]{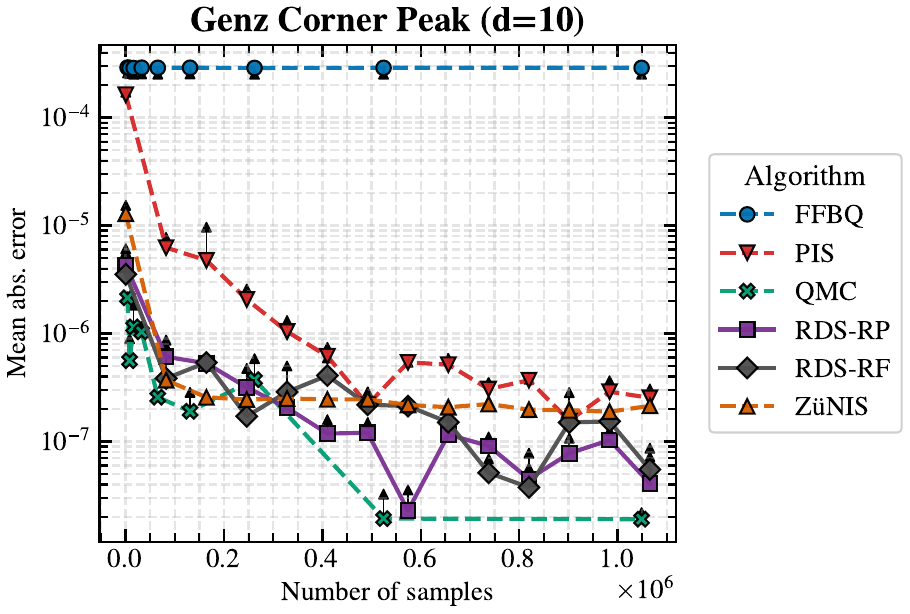}
    \hfill
    \includegraphics[width=0.32\textwidth]{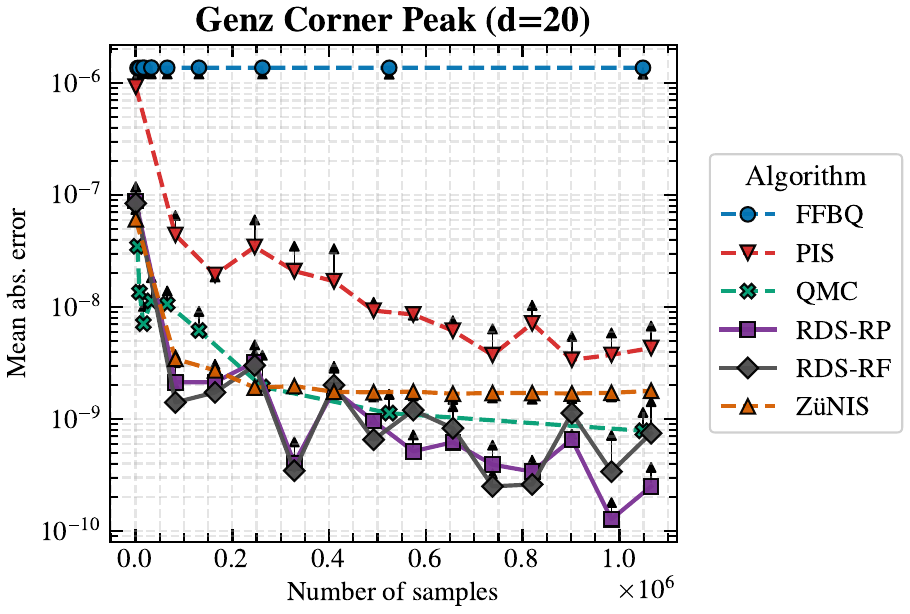}

    \vspace{0.3em}

    \includegraphics[width=0.32\textwidth]{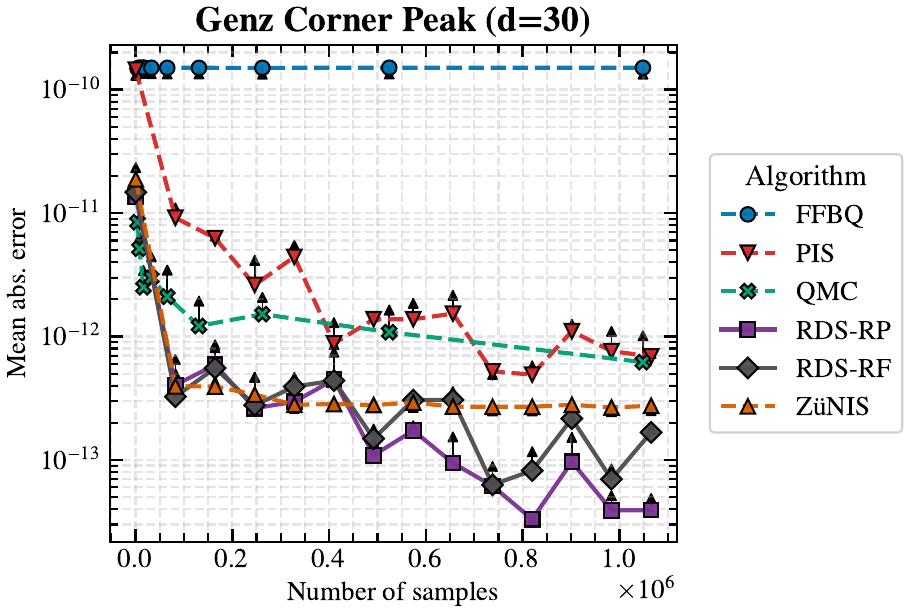}
    \hfill
    \includegraphics[width=0.32\textwidth]{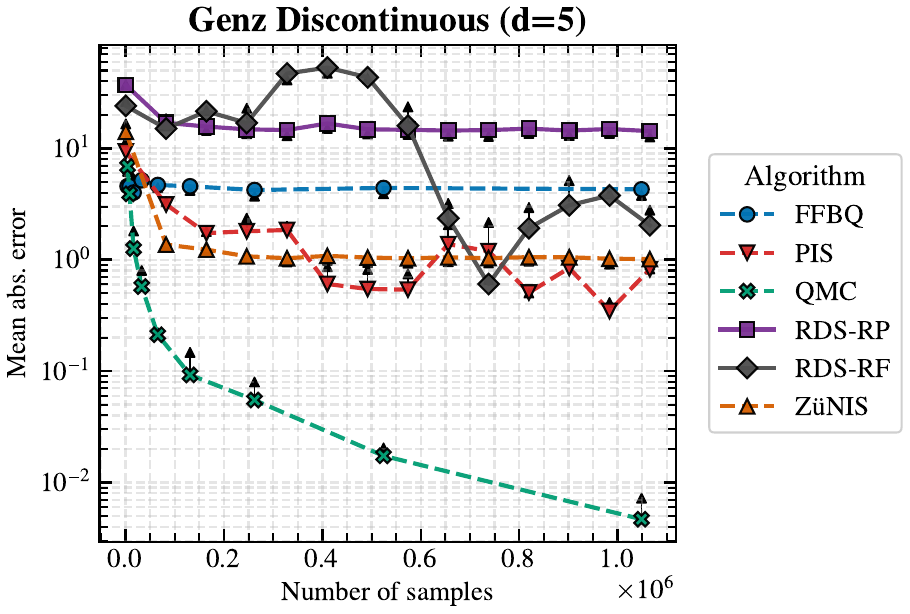}
    \hfill
    \includegraphics[width=0.32\textwidth]{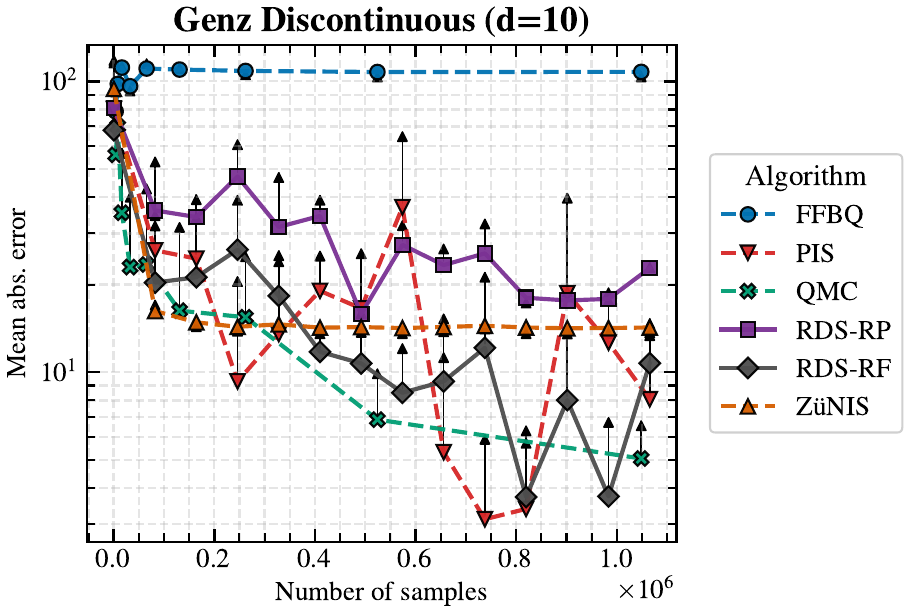}

    \vspace{0.3em}

    \includegraphics[width=0.32\textwidth]{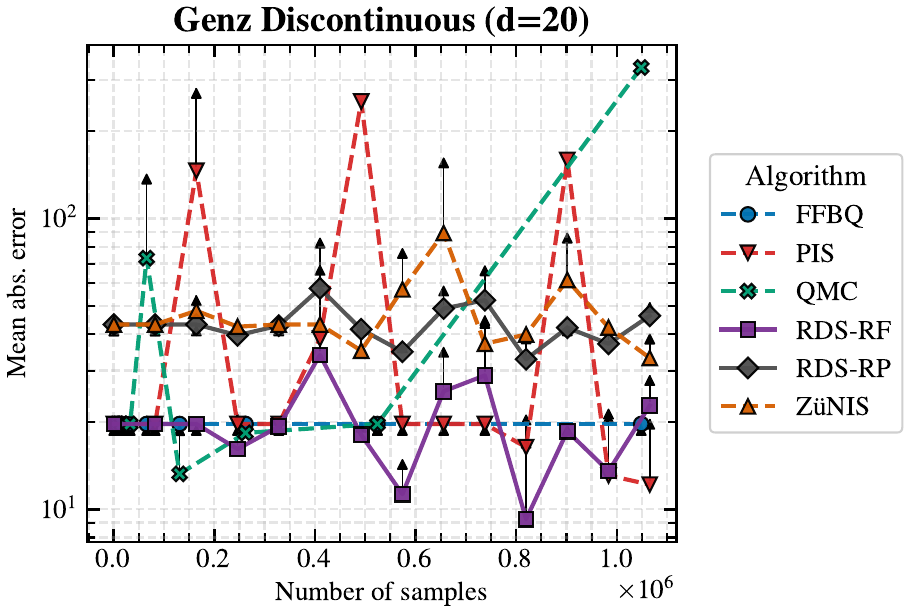}
    \hfill
    \includegraphics[width=0.32\textwidth]{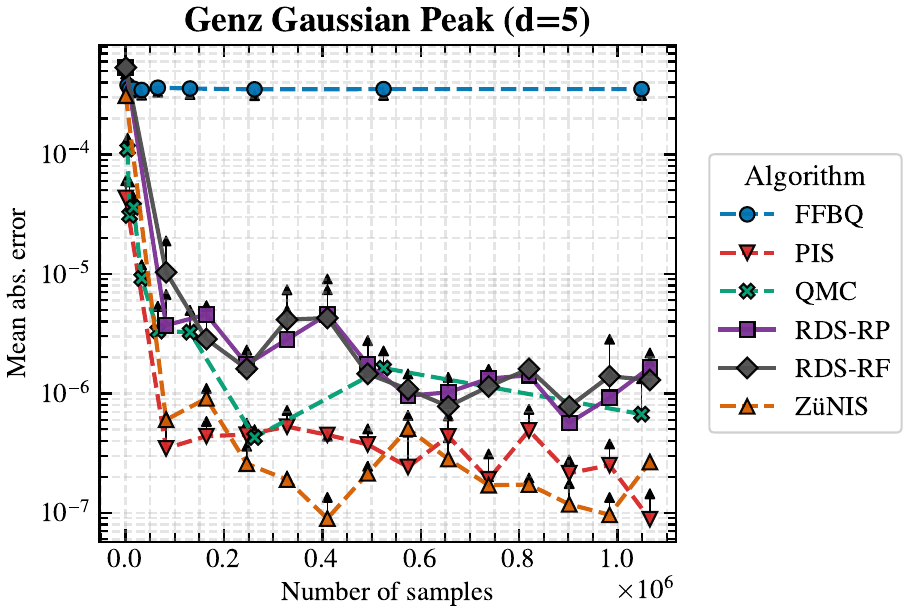}
    \hfill
    \includegraphics[width=0.32\textwidth]{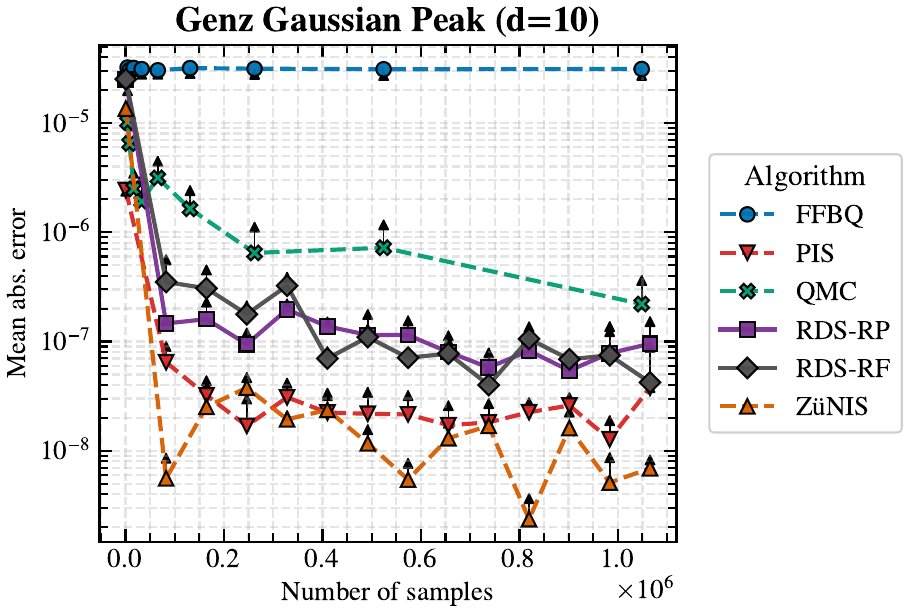}

    \vspace{0.3em}

    \includegraphics[width=0.32\textwidth]{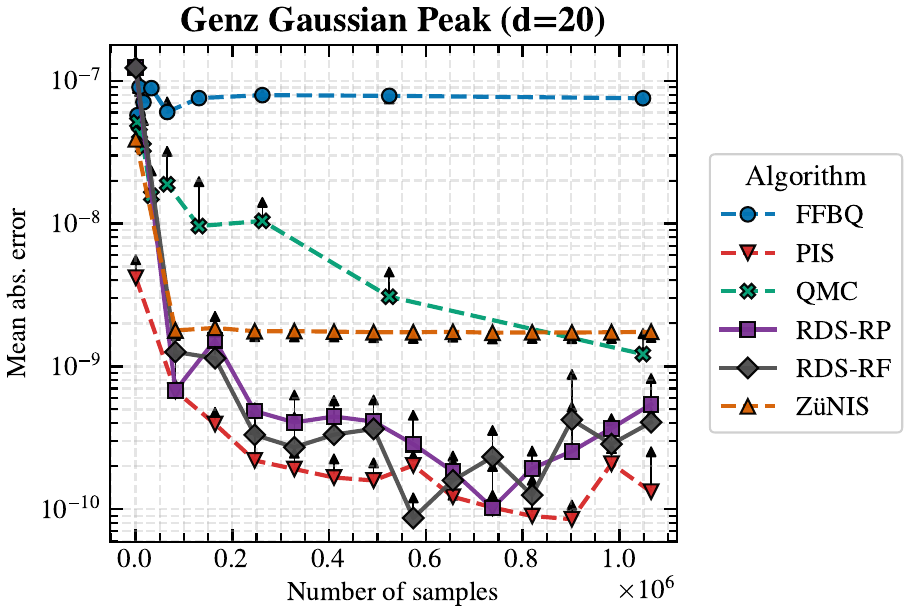}
    \hfill
    \includegraphics[width=0.32\textwidth]{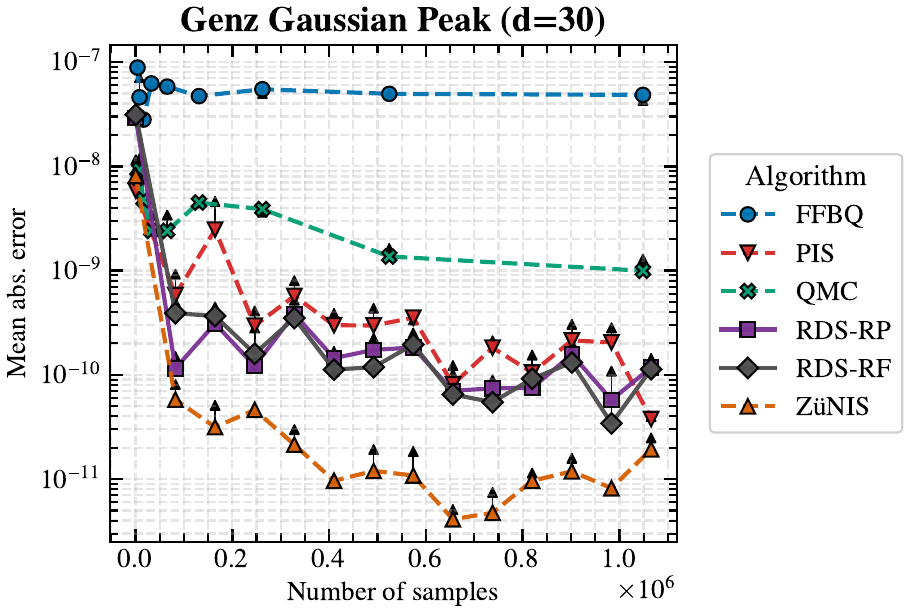}
    \hfill
    \includegraphics[width=0.32\textwidth]{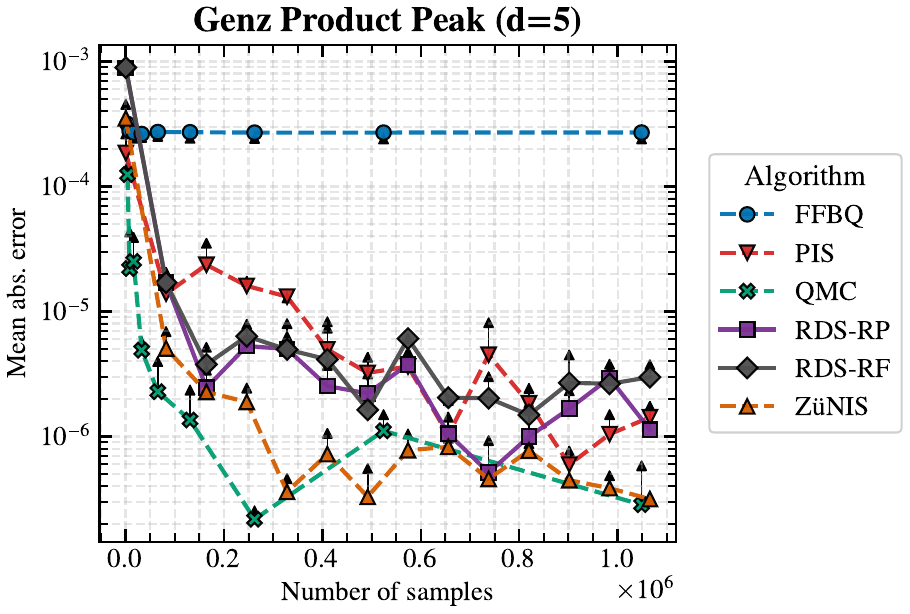}

    \vspace{0.3em}

    \includegraphics[width=0.32\textwidth]{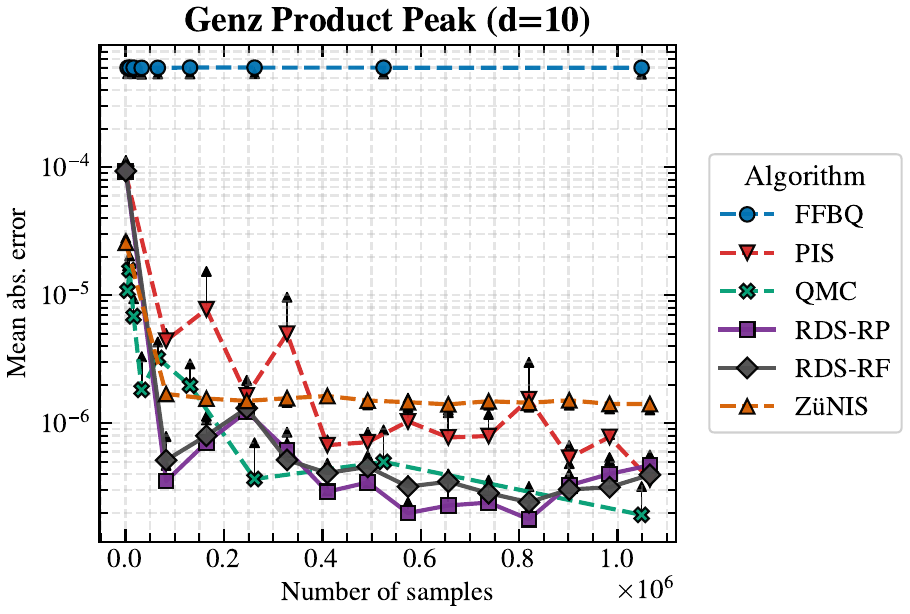}
    \hfill
    \includegraphics[width=0.32\textwidth]{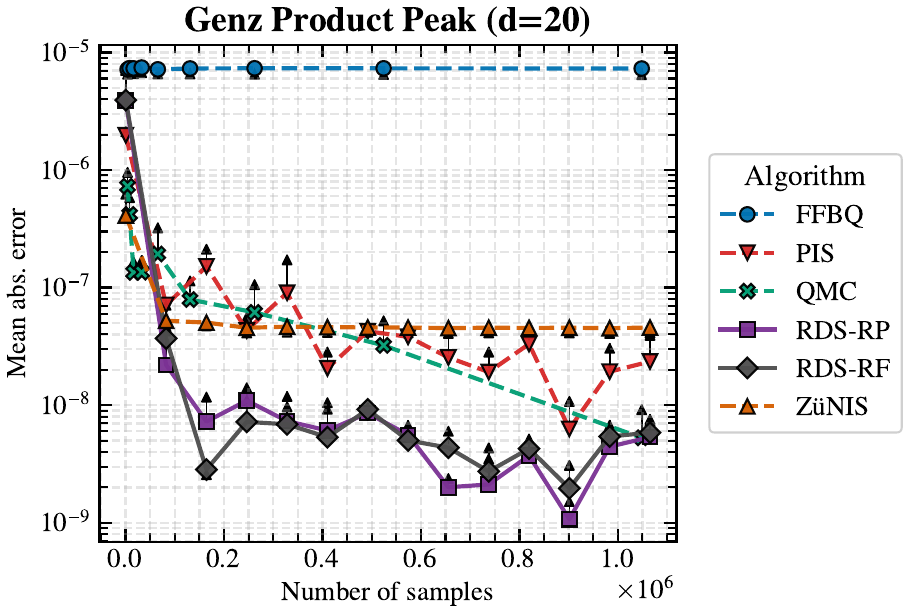}
    \hfill
    \includegraphics[width=0.32\textwidth]{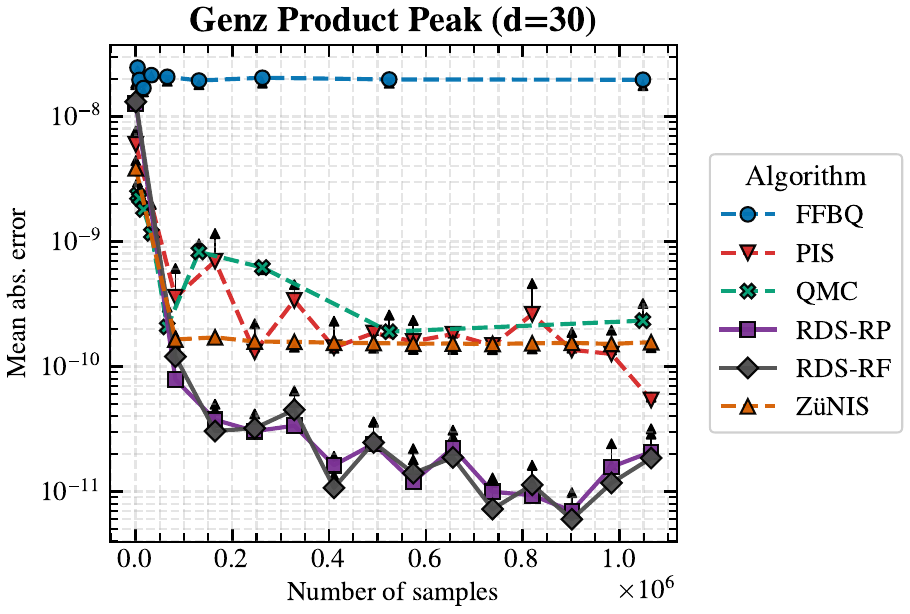}

    \vspace{0.3em}

    \includegraphics[width=0.32\textwidth]{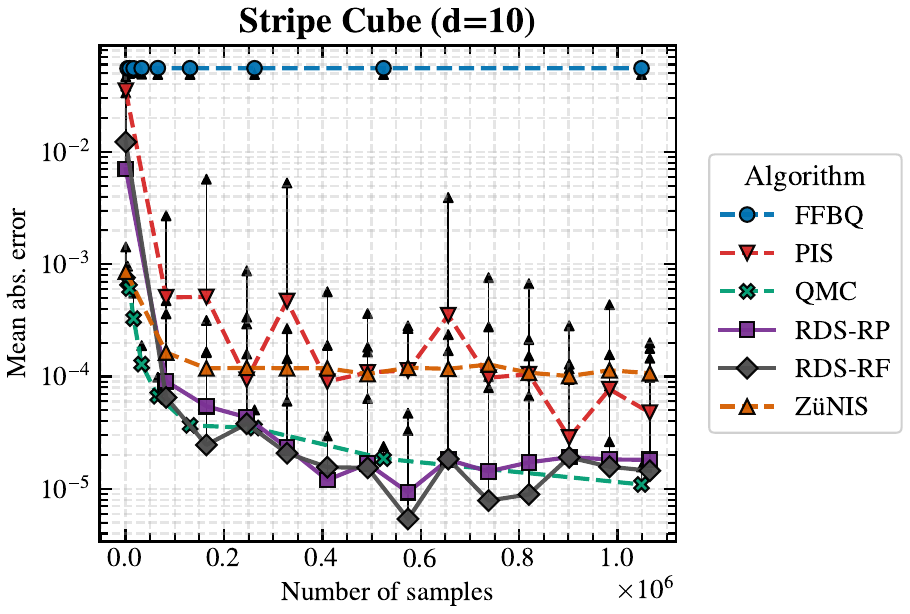}
    \hfill
    \includegraphics[width=0.32\textwidth]{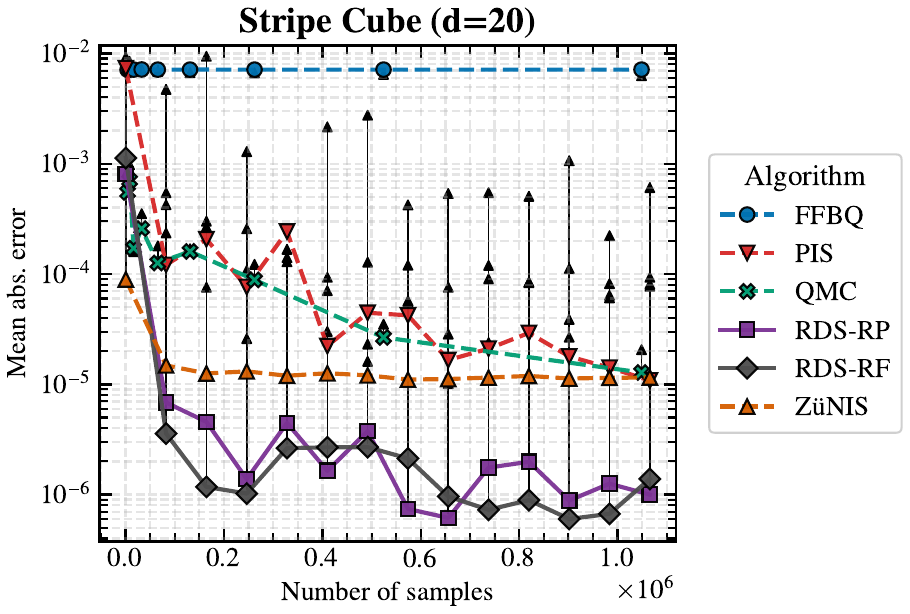}
    \hfill
    \includegraphics[width=0.32\textwidth]{figures/nfe_err/Stripe_Cube_d-50.pdf}

    \caption{
        Mean absolute error as a function of the number of samples used for estimation across the considered integration benchmarks and dimensionalities
        (continued).
    }
    \label{fig:nfe_scaling_full_2}
\end{figure*}

\begin{table*}[t]
\centering
\caption{Mean absolute error of the normalizing-constant estimate using $N=2^{20}$ samples. Values are mean $\pm$ standard deviation. Best result in each row is shown in bold.}
\label{tab:main_results}
\small
\setlength{\tabcolsep}{4pt}
\resizebox{\textwidth}{!}{%
\begin{tabular}{llrrrrrr}
\toprule
Target & $d$ & FFBQ & PIS & QMC & RDS-RF & RDS-RP & ZüNIS \\
\midrule
\multirow{4}{*}{Genz Continuous} & 5 & $6.84\mathrm{e}{-7} \pm 1.87\mathrm{e}{-7}$ & $2.31\mathrm{e}{-7} \pm 1.53\mathrm{e}{-7}$ & $\bm{1.48\mathrm{e}{-7} \pm 9.57\mathrm{e}{-8}}$ & $5.66\mathrm{e}{-7} \pm 4.46\mathrm{e}{-7}$ & $3.99\mathrm{e}{-7} \pm 3.00\mathrm{e}{-7}$ & $3.73\mathrm{e}{-6} \pm 5.46\mathrm{e}{-8}$ \\
 & 10 & $4.32\mathrm{e}{-5} \pm 3.58\mathrm{e}{-8}$ & $2.12\mathrm{e}{-8} \pm 1.15\mathrm{e}{-8}$ & $4.91\mathrm{e}{-8} \pm 4.35\mathrm{e}{-8}$ & $3.78\mathrm{e}{-8} \pm 9.14\mathrm{e}{-9}$ & $\bm{1.05\mathrm{e}{-8} \pm 8.67\mathrm{e}{-9}}$ & $1.44\mathrm{e}{-7} \pm 1.72\mathrm{e}{-9}$ \\
 & 20 & $4.86\mathrm{e}{-8} \pm 7.84\mathrm{e}{-10}$ & $3.89\mathrm{e}{-11} \pm 2.60\mathrm{e}{-11}$ & $6.96\mathrm{e}{-11} \pm 9.46\mathrm{e}{-11}$ & $3.86\mathrm{e}{-11} \pm 4.78\mathrm{e}{-11}$ & $\bm{2.28\mathrm{e}{-11} \pm 2.57\mathrm{e}{-11}}$ & $9.14\mathrm{e}{-10} \pm 2.56\mathrm{e}{-12}$ \\
 & 30 & $2.42\mathrm{e}{-9} \pm 2.00\mathrm{e}{-11}$ & $1.82\mathrm{e}{-11} \pm 5.50\mathrm{e}{-12}$ & $2.81\mathrm{e}{-11} \pm 1.62\mathrm{e}{-11}$ & $1.03\mathrm{e}{-12} \pm 6.53\mathrm{e}{-13}$ & $\bm{9.71\mathrm{e}{-13} \pm 3.39\mathrm{e}{-13}}$ & $1.46\mathrm{e}{-11} \pm 2.36\mathrm{e}{-13}$ \\
\midrule
\multirow{4}{*}{Genz Gaussian Peak} & 5 & $3.52\mathrm{e}{-4} \pm 4.76\mathrm{e}{-7}$ & $1.56\mathrm{e}{-7} \pm 1.37\mathrm{e}{-7}$ & $6.70\mathrm{e}{-7} \pm 9.57\mathrm{e}{-7}$ & $1.04\mathrm{e}{-6} \pm 1.66\mathrm{e}{-6}$ & $1.11\mathrm{e}{-6} \pm 1.59\mathrm{e}{-6}$ & $\bm{4.21\mathrm{e}{-8} \pm 1.51\mathrm{e}{-8}}$ \\
 & 10 & $3.12\mathrm{e}{-5} \pm 8.43\mathrm{e}{-8}$ & $9.71\mathrm{e}{-9} \pm 3.57\mathrm{e}{-9}$ & $2.20\mathrm{e}{-7} \pm 2.32\mathrm{e}{-7}$ & $2.35\mathrm{e}{-8} \pm 1.50\mathrm{e}{-8}$ & $1.59\mathrm{e}{-8} \pm 1.15\mathrm{e}{-8}$ & $\bm{3.25\mathrm{e}{-9} \pm 4.64\mathrm{e}{-10}}$ \\
 & 20 & $7.55\mathrm{e}{-8} \pm 6.82\mathrm{e}{-9}$ & $\bm{4.05\mathrm{e}{-11} \pm 2.96\mathrm{e}{-11}}$ & $1.22\mathrm{e}{-9} \pm 7.95\mathrm{e}{-10}$ & $1.20\mathrm{e}{-10} \pm 1.53\mathrm{e}{-10}$ & $7.44\mathrm{e}{-11} \pm 1.08\mathrm{e}{-10}$ & $1.73\mathrm{e}{-9} \pm 1.17\mathrm{e}{-11}$ \\
 & 30 & $4.84\mathrm{e}{-8} \pm 1.02\mathrm{e}{-9}$ & $4.48\mathrm{e}{-10} \pm 3.23\mathrm{e}{-10}$ & $9.95\mathrm{e}{-10} \pm 6.31\mathrm{e}{-10}$ & $6.25\mathrm{e}{-11} \pm 2.09\mathrm{e}{-11}$ & $3.14\mathrm{e}{-11} \pm 1.96\mathrm{e}{-11}$ & $\bm{6.76\mathrm{e}{-12} \pm 7.24\mathrm{e}{-12}}$ \\
\midrule
\multirow{3}{*}{Genz Discontinuous} & 5 & $4.29\mathrm{e}{0} \pm 3.19\mathrm{e}{-2}$ & $3.97\mathrm{e}{-1} \pm 1.09\mathrm{e}{-1}$ & $\bm{4.68\mathrm{e}{-3} \pm 4.32\mathrm{e}{-3}}$ & $5.32\mathrm{e}{-1} \pm 2.02\mathrm{e}{-1}$ & $1.46\mathrm{e}{1} \pm 4.71\mathrm{e}{-1}$ & $1.01\mathrm{e}{0} \pm 1.49\mathrm{e}{-2}$ \\
 & 10 & $1.08\mathrm{e}{2} \pm 1.26\mathrm{e}{0}$ & $7.34\mathrm{e}{0} \pm 9.01\mathrm{e}{0}$ & $5.05\mathrm{e}{0} \pm 2.07\mathrm{e}{0}$ & $\bm{3.99\mathrm{e}{0} \pm 4.77\mathrm{e}{0}}$ & $1.72\mathrm{e}{1} \pm 3.93\mathrm{e}{0}$ & $1.41\mathrm{e}{1} \pm 5.23\mathrm{e}{-2}$ \\
 & 20 & $1.97\mathrm{e}{1} \pm 2.34\mathrm{e}{-6}$ & $2.27\mathrm{e}{1} \pm 7.47\mathrm{e}{0}$ & $3.31\mathrm{e}{2} \pm 5.27\mathrm{e}{2}$ & $\bm{9.87\mathrm{e}{0} \pm 4.00\mathrm{e}{0}}$ & $3.34\mathrm{e}{1} \pm 4.00\mathrm{e}{0}$ & $4.02\mathrm{e}{1} \pm 1.24\mathrm{e}{0}$ \\
\midrule
\multirow{4}{*}{Genz Corner Peak} & 5 & $4.98\mathrm{e}{-5} \pm 8.43\mathrm{e}{-8}$ & $5.18\mathrm{e}{-6} \pm 2.19\mathrm{e}{-6}$ & $\bm{2.02\mathrm{e}{-8} \pm 1.85\mathrm{e}{-8}}$ & $4.10\mathrm{e}{-6} \pm 3.45\mathrm{e}{-6}$ & $2.00\mathrm{e}{-6} \pm 1.10\mathrm{e}{-6}$ & $7.77\mathrm{e}{-6} \pm 2.41\mathrm{e}{-7}$ \\
 & 10 & $2.89\mathrm{e}{-4} \pm 2.78\mathrm{e}{-7}$ & $2.66\mathrm{e}{-7} \pm 4.40\mathrm{e}{-8}$ & $\bm{1.91\mathrm{e}{-8} \pm 8.41\mathrm{e}{-9}}$ & $1.38\mathrm{e}{-7} \pm 6.37\mathrm{e}{-8}$ & $4.76\mathrm{e}{-8} \pm 6.61\mathrm{e}{-8}$ & $1.99\mathrm{e}{-7} \pm 1.38\mathrm{e}{-8}$ \\
 & 20 & $1.37\mathrm{e}{-6} \pm 2.87\mathrm{e}{-9}$ & $2.90\mathrm{e}{-9} \pm 2.36\mathrm{e}{-9}$ & $7.85\mathrm{e}{-10} \pm 6.42\mathrm{e}{-10}$ & $7.07\mathrm{e}{-10} \pm 3.33\mathrm{e}{-10}$ & $\bm{5.00\mathrm{e}{-10} \pm 1.18\mathrm{e}{-10}}$ & $1.76\mathrm{e}{-9} \pm 5.48\mathrm{e}{-11}$ \\
 & 30 & $1.50\mathrm{e}{-10} \pm 1.61\mathrm{e}{-13}$ & $1.03\mathrm{e}{-12} \pm 3.86\mathrm{e}{-13}$ & $6.20\mathrm{e}{-13} \pm 6.15\mathrm{e}{-13}$ & $1.27\mathrm{e}{-13} \pm 4.52\mathrm{e}{-14}$ & $\bm{1.10\mathrm{e}{-14} \pm 5.46\mathrm{e}{-15}}$ & $2.57\mathrm{e}{-13} \pm 2.60\mathrm{e}{-15}$ \\
\midrule
\multirow{4}{*}{Genz Product Peak} & 5 & $2.70\mathrm{e}{-4} \pm 4.66\mathrm{e}{-7}$ & $10.00\mathrm{e}{-7} \pm 3.09\mathrm{e}{-7}$ & $2.85\mathrm{e}{-7} \pm 4.24\mathrm{e}{-7}$ & $4.23\mathrm{e}{-6} \pm 8.18\mathrm{e}{-7}$ & $1.06\mathrm{e}{-6} \pm 5.65\mathrm{e}{-7}$ & $\bm{1.81\mathrm{e}{-7} \pm 4.19\mathrm{e}{-8}}$ \\
 & 10 & $5.99\mathrm{e}{-4} \pm 3.81\mathrm{e}{-7}$ & $4.10\mathrm{e}{-7} \pm 9.22\mathrm{e}{-8}$ & $\bm{1.92\mathrm{e}{-7} \pm 1.95\mathrm{e}{-7}}$ & $4.05\mathrm{e}{-7} \pm 1.49\mathrm{e}{-7}$ & $3.98\mathrm{e}{-7} \pm 1.16\mathrm{e}{-7}$ & $1.40\mathrm{e}{-6} \pm 1.64\mathrm{e}{-8}$ \\
 & 20 & $7.29\mathrm{e}{-6} \pm 4.77\mathrm{e}{-8}$ & $1.47\mathrm{e}{-8} \pm 5.30\mathrm{e}{-9}$ & $5.28\mathrm{e}{-9} \pm 5.98\mathrm{e}{-9}$ & $1.27\mathrm{e}{-9} \pm 7.77\mathrm{e}{-10}$ & $\bm{9.17\mathrm{e}{-10} \pm 7.95\mathrm{e}{-10}}$ & $4.50\mathrm{e}{-8} \pm 3.74\mathrm{e}{-10}$ \\
 & 30 & $1.97\mathrm{e}{-8} \pm 1.83\mathrm{e}{-10}$ & $2.76\mathrm{e}{-10} \pm 2.09\mathrm{e}{-10}$ & $2.33\mathrm{e}{-10} \pm 1.55\mathrm{e}{-10}$ & $1.00\mathrm{e}{-11} \pm 4.85\mathrm{e}{-12}$ & $\bm{8.56\mathrm{e}{-12} \pm 2.27\mathrm{e}{-12}}$ & $1.52\mathrm{e}{-10} \pm 1.51\mathrm{e}{-12}$ \\
\midrule
\multirow{3}{*}{GMM Cube} & 10 & $9.59\mathrm{e}{-1} \pm 1.81\mathrm{e}{-2}$ & $1.44\mathrm{e}{-3} \pm 1.29\mathrm{e}{-3}$ & $5.49\mathrm{e}{-2} \pm 7.75\mathrm{e}{-2}$ & $1.85\mathrm{e}{-3} \pm 1.43\mathrm{e}{-3}$ & $4.81\mathrm{e}{-3} \pm 1.47\mathrm{e}{-3}$ & $\bm{3.02\mathrm{e}{-4} \pm 1.68\mathrm{e}{-4}}$ \\
 & 20 & $9.44\mathrm{e}{-1} \pm 3.27\mathrm{e}{-3}$ & $\bm{2.24\mathrm{e}{-3} \pm 1.78\mathrm{e}{-3}}$ & $7.45\mathrm{e}{-1} \pm 2.32\mathrm{e}{-1}$ & $4.17\mathrm{e}{-2} \pm 4.77\mathrm{e}{-2}$ & $1.18\mathrm{e}{-2} \pm 5.76\mathrm{e}{-3}$ & $9.84\mathrm{e}{-1} \pm 0.00\mathrm{e}{0}$ \\
 & 50 & - & $1.98\mathrm{e}{-1} \pm 6.14\mathrm{e}{-3}$ & $8.65\mathrm{e}{-1} \pm 7.30\mathrm{e}{-8}$ & $3.00\mathrm{e}{-1} \pm 8.00\mathrm{e}{-2}$ & $\bm{1.68\mathrm{e}{-2} \pm 2.11\mathrm{e}{-2}}$ & $1.05\mathrm{e}{0} \pm 7.30\mathrm{e}{-8}$ \\
\midrule
\multirow{3}{*}{Stripe Cube} & 10 & $5.56\mathrm{e}{-2} \pm 1.67\mathrm{e}{-4}$ & $7.02\mathrm{e}{-5} \pm 1.86\mathrm{e}{-4}$ & $1.09\mathrm{e}{-5} \pm 8.50\mathrm{e}{-6}$ & $1.08\mathrm{e}{-5} \pm 1.99\mathrm{e}{-5}$ & $\bm{6.80\mathrm{e}{-6} \pm 6.73\mathrm{e}{-5}}$ & $1.01\mathrm{e}{-4} \pm 5.63\mathrm{e}{-5}$ \\
 & 20 & $7.14\mathrm{e}{-3} \pm 5.51\mathrm{e}{-5}$ & $2.09\mathrm{e}{-5} \pm 1.06\mathrm{e}{-3}$ & $1.28\mathrm{e}{-5} \pm 1.29\mathrm{e}{-5}$ & $\bm{1.44\mathrm{e}{-6} \pm 1.17\mathrm{e}{-4}}$ & $1.70\mathrm{e}{-6} \pm 7.08\mathrm{e}{-5}$ & $1.07\mathrm{e}{-5} \pm 6.51\mathrm{e}{-5}$ \\
 & 50 & - & $1.17\mathrm{e}{-4} \pm 4.98\mathrm{e}{-3}$ & $4.89\mathrm{e}{-5} \pm 3.66\mathrm{e}{-5}$ & $\bm{6.12\mathrm{e}{-7} \pm 5.64\mathrm{e}{-5}}$ & $8.10\mathrm{e}{-7} \pm 8.71\mathrm{e}{-5}$ & $1.53\mathrm{e}{-5} \pm 1.32\mathrm{e}{-5}$ \\
\midrule
\multirow{3}{*}{Funnel Cube} & 10 & $9.33\mathrm{e}{-1} \pm 4.96\mathrm{e}{-3}$ & $\bm{3.72\mathrm{e}{-3} \pm 2.28\mathrm{e}{-3}}$ & $1.27\mathrm{e}{-1} \pm 7.26\mathrm{e}{-2}$ & $4.76\mathrm{e}{-2} \pm 4.57\mathrm{e}{-3}$ & $1.16\mathrm{e}{-1} \pm 8.30\mathrm{e}{-2}$ & $8.55\mathrm{e}{-2} \pm 1.92\mathrm{e}{-4}$ \\
 & 20 & $8.15\mathrm{e}{-1} \pm 1.54\mathrm{e}{-1}$ & $1.79\mathrm{e}{-2} \pm 8.68\mathrm{e}{-3}$ & $6.13\mathrm{e}{-1} \pm 1.53\mathrm{e}{-1}$ & $1.99\mathrm{e}{-1} \pm 9.58\mathrm{e}{-2}$ & $8.78\mathrm{e}{-2} \pm 4.73\mathrm{e}{-2}$ & $\bm{9.32\mathrm{e}{-4} \pm 6.00\mathrm{e}{-4}}$ \\
 & 50 & - & $7.41\mathrm{e}{-1} \pm 5.28\mathrm{e}{-2}$ & $8.34\mathrm{e}{-1} \pm 2.37\mathrm{e}{-2}$ & $\bm{6.02\mathrm{e}{-1} \pm 8.33\mathrm{e}{-2}}$ & $6.50\mathrm{e}{-1} \pm 2.45\mathrm{e}{-2}$ & $6.90\mathrm{e}{-1} \pm 1.26\mathrm{e}{-2}$ \\
\midrule
\multirow{2}{*}{Asian Option} & 10 & $1.30\mathrm{e}{1} \pm 5.49\mathrm{e}{-3}$ & $1.25\mathrm{e}{-2} \pm 4.73\mathrm{e}{-3}$ & $1.27\mathrm{e}{-1} \pm 2.43\mathrm{e}{-5}$ & $\bm{1.28\mathrm{e}{-3} \pm 3.81\mathrm{e}{-4}}$ & $2.15\mathrm{e}{-2} \pm 3.60\mathrm{e}{-3}$ & $2.44\mathrm{e}{-2} \pm 1.29\mathrm{e}{-3}$ \\
 & 120 & - & $7.04\mathrm{e}{-1} \pm 5.81\mathrm{e}{-2}$ & $1.28\mathrm{e}{-1} \pm 1.93\mathrm{e}{-4}$ & $\bm{1.06\mathrm{e}{-3} \pm 1.47\mathrm{e}{-3}}$ & $2.15\mathrm{e}{-2} \pm 1.48\mathrm{e}{-3}$ & $5.66\mathrm{e}{-2} \pm 3.69\mathrm{e}{-4}$ \\
\midrule
\multirow{1}{*}{Elliptic Two-Loop} & 5 & $9.74\mathrm{e}{-3} \pm 5.71\mathrm{e}{-5}$ & $2.51\mathrm{e}{-3} \pm 2.58\mathrm{e}{-4}$ & $3.12\mathrm{e}{-3} \pm 1.35\mathrm{e}{-3}$ & $\bm{3.05\mathrm{e}{-4} \pm 2.12\mathrm{e}{-4}}$ & $7.32\mathrm{e}{-4} \pm 9.99\mathrm{e}{-5}$ & $1.02\mathrm{e}{-3} \pm 1.03\mathrm{e}{-5}$ \\
\bottomrule
\end{tabular}%
}
\end{table*}

\begin{figure*}
    \centering

    \begin{minipage}[t]{0.32\textwidth}
        \centering
        \includegraphics[width=\linewidth]{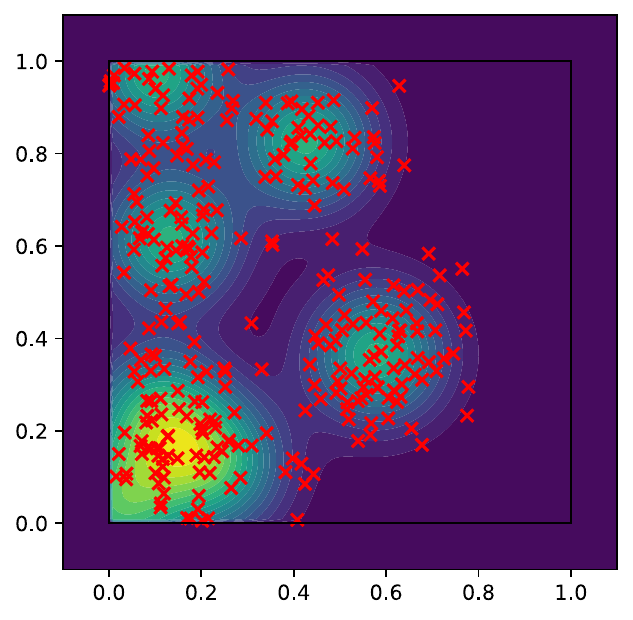}
        \vspace{0.2em}

        \small GMM Cube
    \end{minipage}
    \hfill
    \begin{minipage}[t]{0.32\textwidth}
        \centering
        \includegraphics[width=\linewidth]{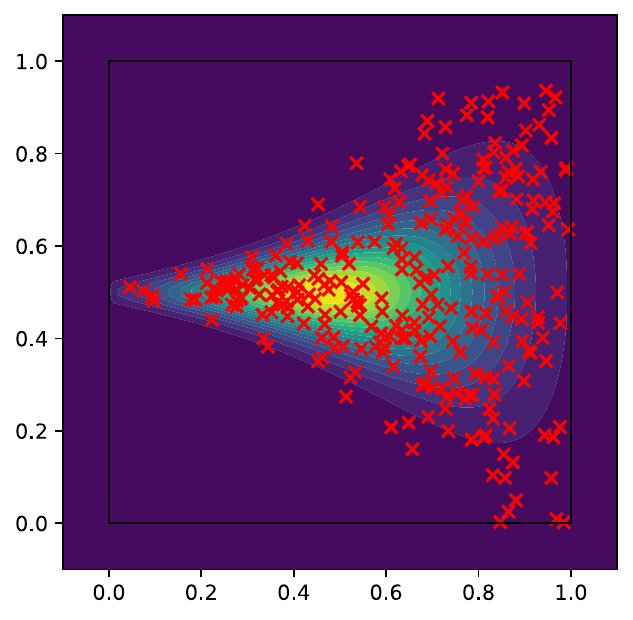}
        \vspace{0.2em}

        \small Funnel Cube
    \end{minipage}
    \hfill
    \begin{minipage}[t]{0.32\textwidth}
        \centering
        \includegraphics[width=\linewidth]{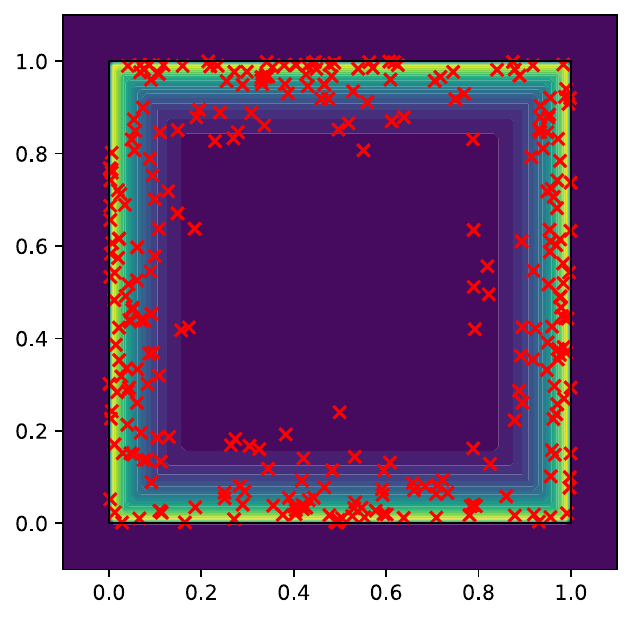}
        \vspace{0.2em}

        \small Stripe Cube
    \end{minipage}

    \caption{
        Samples generated by RDS-RP model.
    }
    \label{fig:synthetic_samples}
\end{figure*}

\end{document}